\documentclass{article}

\usepackage{blockgtq_arxiv,times}

\usepackage{amsmath,amsfonts,bm}

\def\eqref#1{equation~\ref{#1}}

\def\1{\bm{1}}

\DeclareMathAlphabet{\mathsfit}{\encodingdefault}{\sfdefault}{m}{sl}
\SetMathAlphabet{\mathsfit}{bold}{\encodingdefault}{\sfdefault}{bx}{n}

\newcommand{\E}{\mathbb{E}}

\newcommand{\softmax}{\mathrm{softmax}}

\usepackage[utf8]{inputenc}
\usepackage[T1]{fontenc}
\usepackage{hyperref}
\usepackage{url}
\usepackage{booktabs}
\usepackage{amssymb}
\usepackage{mathtools}
\usepackage{amsthm}
\usepackage{algorithm}
\usepackage{algpseudocode}
\usepackage{graphicx}
\usepackage[table]{xcolor}
\usepackage{subcaption}
\usepackage{multirow}
\usepackage{enumitem}
\usepackage{nicefrac}
\usepackage{microtype}
\usepackage{placeins}

\hypersetup{
  pdftitle={Block-GTQ: RoPE-Block Rate Allocation for KV-Cache Quantization},
  pdfauthor={Fengfeng Liang, Yuechen Zhang, Jiaya Jia},
  colorlinks=true,
  linkcolor=black,
  citecolor=black,
  urlcolor=blue
}

\newtheorem{theorem}{Theorem}
\newtheorem{lemma}[theorem]{Lemma}

\newcommand{\bq}{\mathbf{q}}
\newcommand{\bk}{\mathbf{k}}

\newcommand{\bx}{\mathbf{x}}
\newcommand{\be}{\mathbf{e}}
\newcommand{\hbk}{\hat{\mathbf{k}}}

\title{Block-GTQ: RoPE-Block Rate Allocation for KV-Cache Quantization}

\author{%
  \large Fengfeng Liang\textsuperscript{1} \qquad
  Yuechen Zhang\textsuperscript{2,3} \qquad
  Jiaya Jia\textsuperscript{1,}\thanks{Corresponding author.}
  \\[7pt]
  \small \textsuperscript{1}Hong Kong University of Science and Technology \\[1pt]
  \small \textsuperscript{2}The Chinese University of Hong Kong \\[1pt]
  \small \textsuperscript{3}MiMo, Xiaomi Corporation
}

\date{}

\begin{document}
\raggedbottom

\maketitle

\begin{abstract}
Larger models and longer contexts make the KV cache a memory and bandwidth
bottleneck. We introduce \emph{Block-GTQ}, a framework that uses the
two-dimensional frequency blocks of rotary position embeddings (RoPE) to
allocate cached-key precision. A blockwise logit-error bound motivates a
label-free query--key energy score, which we combine with codec-dependent
rate penalties. Under equal-cost budgets and diminishing marginal gains,
largest-marginal-gain allocation exactly minimizes the separable integer
proxy. We instantiate Block-GTQ with TurboQuant-MSE (TQ-MSE), grouping equal-rate
blocks to support attention over packed caches without materializing a decoded
fp16 KV cache. Across ten model configurations, K-only diagnostics
reduce layer-averaged logit mean absolute error by about $32$--$80\%$
relative to uniform TQ-MSE, improving all $367$ layers at both 2- and 3-bit
average logical key rates. At K2V2 (2 average logical bits per K/V
coordinate) on Llama-3.1-8B-Instruct, Block-GTQ raises the six-task
needle-in-a-haystack (NIAH) average from $70.6$ to $97.4$ and LongBench-EN
from $33.7$ to $44.9$; K2V2 gains also hold on the tested 32B and
70B models. On DeepSeek-R1-Distill-Qwen-7B at K3V2, without fp16 sink or
recent-token buffers, it achieves $49.6/38.8\%$ avg@8 on AIME 2024/2025 versus fp16’s $54.2/37.9\%$,
where uniform TQ-MSE scores zero on both. On Qwen2.5-3B-Instruct, packed
K3V3 serving on one H800 reduces resident KV footprint by a factor of
$3.24$ and retains perplexity close to fp16 at 128K, while uniform TQ-MSE
degrades severely.
\end{abstract}

\vspace{-3mm}
\section{Introduction}
\label{sec:introduction}
\vspace{-2mm}
Longer contexts and larger models make the KV cache an increasingly important
bottleneck in autoregressive inference. The cache stores one key and one value
vector for every past token at every transformer layer, and decoding repeatedly
accesses this growing state to attend over the context. For example, in a
70B-class grouped-query attention (GQA) model with 80 layers, 8 KV heads, and 128-dimensional
heads, an fp16 KV cache occupies about 320 KiB per token, or roughly 40 GiB at
a 128K-token context. Scaling the layer count or KV width increases the
per-token footprint, while longer contexts increase both the resident capacity
and the K/V traffic of each decode step, creating coupled capacity and
bandwidth bottlenecks~\citep{pope2022efficiently, 
dao2023flashattention2,kwon2023efficient,liu2024cachegen}.

KV-cache quantization mitigates this pressure by storing cached keys and values
at lower precision~\citep{liu2024kivi,hooper2024kvquant,he2024zipcache,
zhang2024coupled,zandieh2025turboquant}. Many methods define quantization
granularity over tokens, channels, groups, heads, or whole vectors, including
methods that account for query sensitivity. A cached key is repeatedly
paired with future queries to form query--key logits, whereas a cached value is
used only after the attention weights have been formed. We therefore look
to the query--key computation, rather than the cache layout alone, to define
the unit of precision allocation. 

Rotary position embeddings (RoPE) give this query--key computation an explicit
block structure~\citep{su2024roformer}. For a query $\bq$, cached key $\bk$,
and relative offset $\Delta$, the unscaled logit is
$\bq^\top R_\Delta\bk$, where $R_\Delta$ is block diagonal with
two-dimensional rotations. It therefore decomposes as
$\sum_i \bq^{(i)\top}R(\Delta\theta_i)\bk^{(i)}$. A perturbation of the cached
key changes the logit through a signed sum over these blocks, while an
additive upper bound provides tractable blockwise error accounting. The
two-dimensional RoPE frequency blocks thus form a computation-aligned,
encoder-independent unit for key-cache rate allocation. Empirically, their
Q/K energy scores can span orders of magnitude (Figure~\ref{fig:overview}),
suggesting that a uniform within-head rate can misallocate the available bit budget.

We propose \emph{Block-GTQ}, a RoPE-block rate-allocation framework that
separates calibration-based block importance from codec-dependent rate penalties.
It estimates a bound-motivated Q/K energy score from marginal second moments
of unlabeled calibration tokens. Under an equal-cost integer budget, the
resulting allocation is solved exactly by largest marginal gain whenever the
per-block penalties are separable and exhibit diminishing returns. We
instantiate the framework with TurboQuant-MSE (TQ-MSE), using the $4^{-b}$
rate dependence in its reference MSE analysis as the allocation penalty.
Equal-rate groups use the same TQ-MSE implementation as the uniform baseline,
values remain uniformly quantized, and a fused packed-cache path serves
compressed K/V codes without materializing a decoded fp16 KV cache in HBM.

\vspace{-2mm}

\begin{figure}[t]
\vspace{-3mm}
  \centering
  \includegraphics[width=\linewidth,trim=0 55pt 0 0,clip]{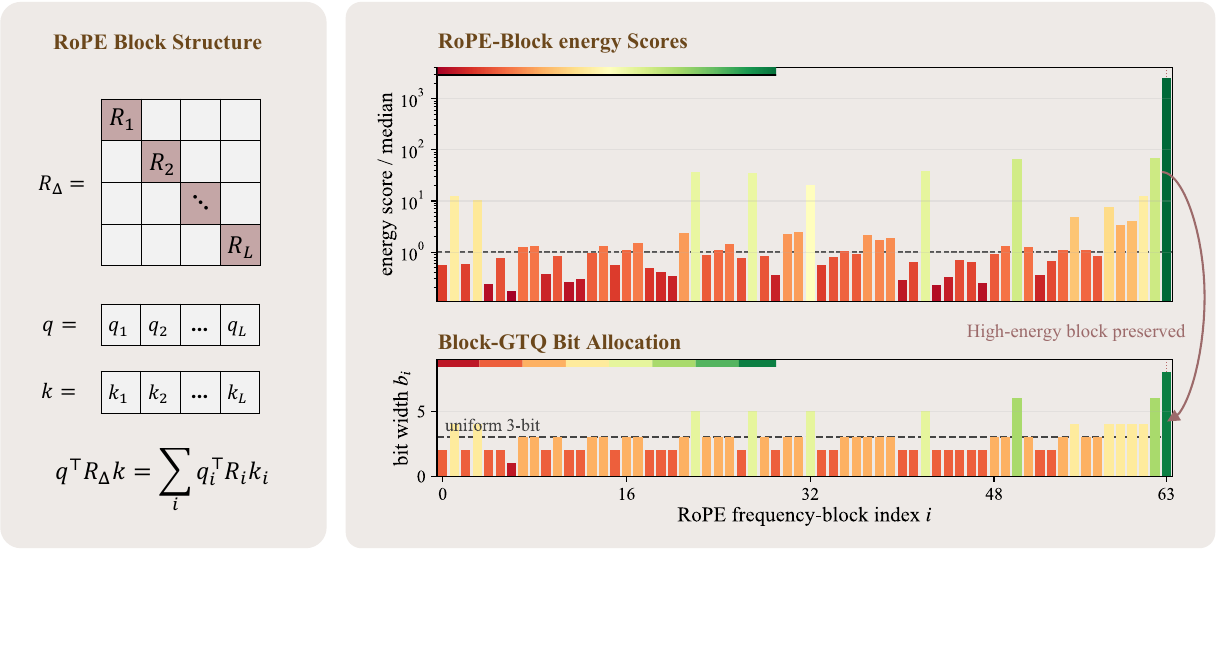}
 \caption{\textbf{RoPE-block precision allocation.}
\textbf{(a)} A RoPE query--key logit decomposes over two-dimensional frequency
blocks. \textbf{(b)} Block energy scores for Qwen3-8B layer 10, KV head 4,
median-normalized for display, span orders of magnitude.
\textbf{(c)} At the same average 3-bit logical rate, Block-GTQ reallocates
precision from low-score to high-score blocks instead of assigning a uniform
width.}
\vspace{-5mm}
  \label{fig:overview}
\end{figure}

Our contributions are:
\begin{enumerate}[leftmargin=*,nosep]
    \item We identify two-dimensional RoPE frequency blocks as an
    encoder-independent unit for key-cache rate allocation by decomposing
    query--key logits and deriving an additive blockwise upper bound on
    key-induced logit error.

    \item We construct a label-free Q/K energy score from marginal second
    moments and formulate codec-conditioned integer rate allocation. For
    separable per-block penalties with diminishing returns, we prove
    largest-marginal-gain allocation globally optimal under a fixed
    equal-cost integer budget.

    \item We instantiate the framework with TQ-MSE and translate its
    fine-grained schedule into a deployable codec and serving path. Equal-rate
    blocks share the uniform baseline's TQ-MSE implementation, while fused
    packed attention decodes transient K/V tiles and avoids materializing a
    decoded fp16 KV cache in HBM.

    \item We validate Block-GTQ from query--key logit fidelity to long-context
    retrieval, understanding, reasoning, and packed serving.
    Across ten model/path configurations, K-only diagnostics show about
    $32$--$80\%$ lower layer-averaged logit MAE than uniform TQ-MSE, with all
    $367$ layers improved at both 2- and 3-bit average key rates. On
    Qwen2.5-3B-Instruct, packed K3V3 serving on one H800 reduces resident KV
    footprint by a factor of $3.24$ and retains perplexity close to
    fp16 at 128K, while uniform TQ-MSE degrades severely. At a 1M-token
    context, Block-GTQ achieves a median decode latency of 17.94 ms/step
    versus 19.50 ms/step for FP16 FlashAttention-2, an 8.0\% reduction.
\end{enumerate}

\section{Related Work}
\label{sec:related}
\vspace{-2mm}

\paragraph{KV-cache codecs and error protection.}
Low-bit KV-cache methods first differ in how they encode each cached
representation. KIVI uses asymmetric granularities---per-channel keys and
per-token values---together with a high-precision residual region, while
KVQuant combines pre-RoPE per-channel key quantization, non-uniform codebooks,
and sparse fp16 outliers~\citep{liu2024kivi,hooper2024kvquant}.
TurboQuant provides rotation-based online vector quantizers with controlled
distortion, and its MSE-oriented variant is the local codec used in our
implementation~\citep{zandieh2025turboquant}. CommVQ instead learns additive
codebooks designed to commute with RoPE, whereas AQUA-KV predicts K/V states
from other layers and quantizes only the residual information
~\citep{li2025commvq,malinovskii2025aqua}. KVTC follows a transform-coding
route, decorrelating KV caches before adaptive quantization and entropy
coding~\citep{staniszewski2026kvtc}. KVSink dynamically identifies sink
tokens to protect during quantization~\citep{su2025kvsink}. These methods
primarily redesign the representation or protect selected errors and tokens.
Block-GTQ separates the local codec from the question of how logical key rate
should be distributed across RoPE blocks.
\vspace{-2mm}
\paragraph{Mixed precision and rate allocation.}
A second line varies precision across selected cache components.
\citet{hariri2026quantize} study the global division of precision between
keys and values and favor key-heavy allocations. Kitty boosts sensitive key
channels from a low-bit base and co-designs the mixed-precision layout with its
serving kernels~\citep{xia2025kitty}. MixKVQ selects mixed key-channel
precision using query magnitude and estimated quantization sensitivity, while
applying uniform per-token value quantization
~\citep{zhang2025mixkvq}. PM-KVQ progressively
reduces precision during long-CoT generation and allocates memory across
sensitive transformer blocks, together with position-interpolated calibration
~\citep{liu2025pmkvq}. AATC whitens K/V features and applies attention-aware
reverse water-filling across transformed channels and layers; its
\emph{variance-only} ablation removes the attention-dependent weights and allocates
from reconstruction variance alone~\citep{laus2026aatc}. These approaches
allocate between K and V or across channels, transformed coefficients,
layers, or temporal stages. Block-GTQ instead distributes a fixed key budget
across two-dimensional RoPE blocks within each head, scores them from label-free Q/K
statistics, and solves the resulting integer allocation for separable
penalties with diminishing marginal gains.
\vspace{-2mm}
\paragraph{RoPE-aware cache compression.}
Several methods account explicitly for positional rotation. KVQuant quantizes
keys before RoPE, RotateKV applies adaptive rotations to suppress outliers, and
CommVQ learns codebooks whose decoding commutes with RoPE
~\citep{hooper2024kvquant,su2025rotatekv,li2025commvq}. EliteKV combines
RoPE frequency selection with joint low-rank projection, RAP prunes whole
RoPE-aligned pairs, and TriAttention uses pre-RoPE geometry to score cached
tokens~\citep{zhou2025elitekv,xin2026rap,mao2026triattention}. Block-GTQ uses a
different role for RoPE: its exact query--key decomposition defines the
precision-allocation unit. Every frequency block is retained; only its logical
rate changes, and blocks assigned the same rate are grouped for codec
realization.

\vspace{-2mm}
\paragraph{Other cache-reduction mechanisms.}
H$_2$O and SnapKV retain selected tokens, while PyramidKV varies the retained
token budget by layer~\citep{zhang2023h2o,li2024snapkv,cai2024pyramidkv}.
SubGen combines online key clustering with value sampling~\citep{zandieh2024subgen}.
MiniCache merges cache states across layers, Palu caches low-rank projected
states, and GEAR combines low-bit quantization with low-rank and sparse error
correction~\citep{liu2024minicache,chang2025palu,kang2024gear}.
These methods change which states are kept, how attention accesses them, or how
quantization errors are corrected; Block-GTQ allocates precision across RoPE
frequency blocks.

\vspace{-2mm}
\section{RoPE-Block Rate Allocation}
\label{sec:method}
\vspace{-2mm}

Block-GTQ uses RoPE frequency blocks as the unit of key-rate allocation,
estimates their importance from unlabeled Q/K statistics, and combines this
importance with codec-dependent rate penalties. We first establish the
blockwise error accounting, then formulate and solve the integer allocation,
and finally realize the resulting rate assignment with TQ-MSE.
\vspace{-2mm}
\subsection{RoPE blocks as the allocation unit}
\label{sec:problem}
\vspace{-2mm}
For a query head and its KV head in a fixed layer, let
  $\bq,\bk\in\mathbb{R}^{d_h}$ be pre-RoPE vectors at query position $n$ and
  key position $m$, with even head dimension $d_h$. Partition them into
  $L=d_h/2$ blocks
  $\bq^{(i)},\bk^{(i)}\in\mathbb{R}^{2}$. At relative offset $\Delta=m-n$, the
  unscaled query--key logit is
  \begin{equation}
    \mathcal{K}_\Delta(\bq,\bk)
    = \bq^\top R_\Delta \bk
    = \sum_{i=1}^{L}
      \bq^{(i)\top}R(\Delta\theta_i)\bk^{(i)},
    \label{eq:rope-decomposition}
  \end{equation}

  where $R(\alpha)$ denotes the two-dimensional rotation by angle $\alpha$,
  $R_\Delta=\operatorname{diag}(R(\Delta\theta_1),\ldots,
  R(\Delta\theta_L))$ and $\theta_i$ is the RoPE frequency of block $i$.
  We omit the attention scaling factor $1/\sqrt{d_h}$.

  Let $\hbk$ be a reconstructed key in the same coordinates and define the
  blockwise reconstruction error as
  $\be_{\bk}^{(i)}=\bk^{(i)}-\hbk^{(i)}$. Holding the query fixed, the
  key-induced logit error is
  \begin{equation}
    \delta\mathcal{K}_\Delta
    = \sum_{i=1}^{L}
      \bq^{(i)\top}R(\Delta\theta_i)\be_{\bk}^{(i)}.
    \label{eq:exact-block-error}
  \end{equation}
  This error admits the following additive bound.

  \begin{lemma}[Blockwise logit-error bound]
  \label{lem:block-error}
  For any query, key, reconstruction, and relative offset,
  \begin{equation}
    \left|\delta\mathcal{K}_\Delta\right|
    \leq
    \sum_{i=1}^{L}
    \|\bq^{(i)}\|_2\,\|\be_{\bk}^{(i)}\|_2.
    \label{eq:block-error-bound}
  \end{equation}
  \end{lemma}

  The bound follows from the triangle inequality, Cauchy--Schwarz, and the
  orthogonality of each RoPE rotation. It upper-bounds the absolute logit
  perturbation by a sum of blockwise terms, making each two-dimensional RoPE
  block an accounting unit without assuming a particular codec. For pure
  rotations, the same analysis applies to post-RoPE caches through a
  norm-preserving change of coordinates. Appendix~\ref{app:coordinate-change}
  gives this mapping and the extension to scalar-scaled RoPE.
  Block-GTQ allocates a fixed average
  key-rate budget across these blocks while quantizing values uniformly.
  Appendix~\ref{app:kv-rate-sensitivity} reports local K/V rate-sensitivity
  contrasts that provide empirical motivation for this focus on K.

\subsection{Block score and codec-conditioned allocation}
\label{sec:score}
\label{sec:allocation}

Motivated by Lemma~\ref{lem:block-error}, Block-GTQ assigns block $i$ in layer
$\ell$ and KV head $h$ the score
\begin{equation}
  s_{\ell,h,i}
  = \frac{1}{2}\left(
      \E_{\substack{t\in\mathcal C\\g\in G(h)}}
      \!\left[\|\bq_{\ell,g,t}^{(i)}\|_2^2\right]
      +
      \E_{u\in\mathcal C}
      \!\left[\|\bk_{\ell,h,u}^{(i)}\|_2^2\right]
    \right),
  \label{eq:block-score}
\end{equation}
where $\mathcal C$ is the calibration-position set and $G(h)$ the query heads
served by KV head $h$. Expectations are uniform empirical averages over the
indicated positions and heads; $t$ and $u$ need not form a query--key pair.
These label-free mean Q/K energy scores determine a static allocation per
layer and KV head. By AM--GM, they upper-bound the paired norm-product moment
for any joint distribution with these calibration marginals
(Appendix~\ref{app:proxy-derivation}).

For fixed $\ell,h$, let $b_i=b_{\ell,h,i}$ be block $i$'s logical bit width per
coordinate and $\mathbf b=(b_1,\ldots,b_L)$. With minimum and maximum integer
widths $b_{\min},b_{\max}$, define the average key rate
$\bar b=L^{-1}\sum_i b_i$ and normalized budget $B=L\bar b$:
\begin{equation}
  \sum_{i=1}^{L} b_i=B=L\bar b,
  \qquad
  b_i\in\{b_{\min},\ldots,b_{\max}\}.
  \label{eq:budget-constraint}
\end{equation}
The corresponding key-code payload is $2B=d_h\bar b$ bits per vector, and
increasing any $b_i$ by one adds the same two code bits. For a codec
$\mathcal E$, let $F_{\ell,h,i}^{\mathcal E}$ be a finite real-valued penalty
curve on the feasible rates, held fixed during allocation. We consider
\begin{equation}
  \min_{\mathbf b}\quad
  J_{\ell,h}^{\mathcal E}(\mathbf b)
  =\sum_{i=1}^{L}F_{\ell,h,i}^{\mathcal E}(b_i)
  \quad\text{s.t. Equation~\ref{eq:budget-constraint}.}
  \label{eq:general-allocation-objective}
\end{equation}
Block-GTQ uses the factorized proxy
\begin{equation}
  F_{\ell,h,i}^{\mathcal E}(b)
  =s_{\ell,h,i}\,\phi_{\ell,h,i}^{\mathcal E}(b),
  \label{eq:factorized-penalty}
\end{equation}
where $\phi_{\ell,h,i}^{\mathcal E}(b)$ is a prescribed codec-dependent rate
penalty on the feasible widths, fixed during allocation. This separates the
calibration score from the rate penalty; the optimizer only requires separable
per-block curves. Suppressing the fixed $\ell,h,\mathcal E$ indices, define
the one-bit gain for $b\in\{b_{\min},\ldots,b_{\max}-1\}$ as
\begin{equation}
  \Delta_i(b)=F_i(b)-F_i(b+1).
  \label{eq:general-marginal-gain}
\end{equation}
Starting from $b_{\min}$, the marginal rule repeatedly assigns the next rate
unit to the block with the largest available $\Delta_i(b_i)$.

\begin{theorem}[Optimal marginal allocation]
\label{thm:greedy}
For the fixed penalty curves above and the equal-cost budget in
Equation~\ref{eq:budget-constraint}, suppose
$b_{\min},b_{\max},B\in\mathbb Z$, $Lb_{\min}\leq B\leq Lb_{\max}$, every rate in
$\{b_{\min},\ldots,b_{\max}\}$ is available, and each marginal-gain sequence is
nonnegative and non-increasing:
\begin{equation}
  \Delta_i(b_{\min})
  \geq \cdots
  \geq \Delta_i(b_{\max}-1)
  \geq 0.
  \label{eq:discrete-convexity}
\end{equation}
The gain condition is vacuous when $b_{\min}=b_{\max}$.
Starting from $b_i=b_{\min}$ and repeatedly incrementing any eligible block
with the largest current gain $\Delta_i(b_i)$ until the budget is reached
returns a global minimizer of Equation~\ref{eq:general-allocation-objective}.
\end{theorem}

Under these equal-cost and diminishing-gain assumptions, this is an instance
of classical marginal analysis for separable discrete resource
allocation~\citep{fox1966discrete}. Appendices~\ref{app:allocation-algorithm}
and~\ref{app:greedy-proof} give the pseudocode and a self-contained proof,
respectively, while
Appendix~\ref{app:conditional-allocation} states the assumptions and
limitations. The guarantee is exact only for the separable objective in
Equation~\ref{eq:general-allocation-objective}.

\vspace{-2mm}
\subsection{TQ-MSE instantiation via equal-rate grouping}
\label{sec:grouping}

\vspace{-2mm}
Block-GTQ and the uniform baseline use the same TQ-MSE implementation. It
normalizes each nonzero input, applies a random orthogonal rotation, and
scalar-quantizes the rotated coordinates with Gaussian Lloyd--Max codebooks.
Decoding renormalizes the reconstructed direction before restoring the stored
input norm (Appendix~\ref{app:rate-accounting}). Motivated by the $4^{-b}$ MSE
rate dependence in the reference analysis of \citet{zandieh2025turboquant}, we
set $\phi_{\ell,h,i}^{\mathrm{TQ}}(b)=\phi^{\mathrm{TQ}}(b)=4^{-b}$ for every
layer, KV head, and block.
Appendix~\ref{app:proxy-derivation} distinguishes that reference guarantee from
our implementation and allocation proxy. The resulting objective and marginal
gain are
\begin{equation}
  J_{\ell,h}^{\mathrm{TQ}}(\mathbf b)
  =\sum_{i=1}^{L}s_{\ell,h,i}4^{-b_i},
  \qquad
  \Delta_{\ell,h,i}^{\mathrm{TQ}}(b_i)
  =\frac{3}{4}s_{\ell,h,i}4^{-b_i}.
  \label{eq:allocation-objective}
\end{equation}
The allocator initially favors blocks with larger scores, but repeated
increments have diminishing returns: after one increment, the gain from
incrementing the same block again is one quarter as large.

For each occupied rate $b$, we concatenate the post-RoPE key blocks in
$\mathcal G_b^{(\ell,h)}=\{i:b_{\ell,h,i}=b\}$ and encode the resulting vector
with TQ-MSE at $b$ bits per coordinate. This realizes the schedule with one
grouped codec per occupied rate; values use a single uniform TQ-MSE rate.
Section~\ref{sec:packed} describes packed serving, and
Appendix~\ref{app:grouping-scope} details the grouped realization.

  \vspace{-2mm}
\section{Packed-Cache Serving}
\label{sec:packed}
  \vspace{-2mm}
\begin{figure}[!ht]
  \vspace{-2mm}
  \centering
  \includegraphics[width=\linewidth]{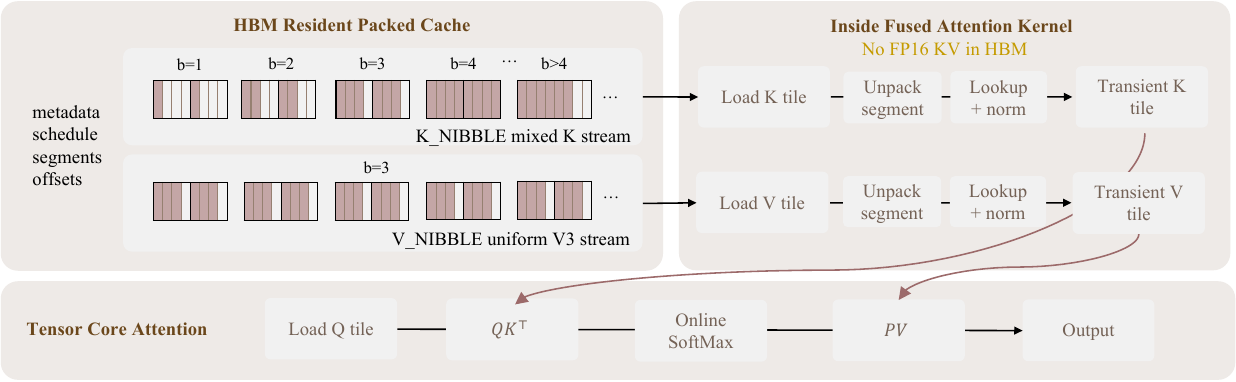}
  \caption{\textbf{Packed-cache serving path.} HBM retains packed K/V codes
  and norms. Fused attention reconstructs one tile in kernel-local storage and
  immediately consumes it in $QK^\top$ and $PV$, without materializing a
  decoded fp16 KV cache in HBM.}
  \label{fig:packed-path}
  \vspace{-2mm}
\end{figure}
  \vspace{-2mm}
Each layer and KV head stores its equal-rate key groups as contiguous
segments in a mixed-rate K stream, alongside a uniform-rate V stream.
During decoding, post-RoPE queries are transformed into the matching codec
coordinates. Fused attention reconstructs K/V tiles in kernel-local storage
and immediately evaluates $QK^\top$, online softmax, and $PV$ with fp32
accumulation. Values are accumulated in codec coordinates; split-key partial
outputs are merged using their log-sum-exp states before a single inverse
value transform. Decoded K/V tiles are never written to HBM.
Prefill writes prompt K/V directly into the same packed layout and attends
over its tiles without materializing a decoded cache or dense attention
matrix. The optimized decode path fuses cache updates and uses a whole-step CUDA
graph. Appendix~\ref{app:packed-details1} provides implementation details;
Appendix~\ref{app:deployment-details} reports deployment results and memory
accounting.

\vspace{-3mm}
\section{Experiments}
\label{sec:experiments}
\vspace{-2mm}
\subsection{Evaluation setup}
\label{sec:protocol}
\vspace{-2mm}
Uniform TQ-MSE is our primary controlled baseline: it uses the same codec and
logical K/V budget as Block-GTQ but assigns a uniform key rate. We write K3V2
for average logical rates of 3 bits per key coordinate and 2 per value
coordinate, and analogously for other budgets. The controlled NIAH and
LongBench experiments use no fp16 sink or recent-token window; AIME reports
both this \emph{Stress} setting and a \emph{Protected} setting with four initial
fp16 sink tokens and 128 recent fp16 tokens. Preserving initial tokens follows
the attention-sink observation~\citep{xiao2024efficient}; the particular
four-token count is our evaluation choice. Calibration data and protocols are
specified per experiment; Appendix~\ref{app:rate-accounting} defines rate
accounting.
\vspace{-2mm}
\subsection{Allocation and attention diagnostics}
\label{sec:mechanism}
\vspace{-2mm}
We examine Block-GTQ's mechanism at three levels: (i) the bit allocation it
learns, (ii) the resulting RoPE-logit error, and (iii) the resulting attention
distribution. The ten-configuration panel covers Qwen, Llama, DeepSeek,
Mistral, and GLM, spanning full-RoPE GQA, partial-RoPE, and multi-head latent
attention (MLA) key paths. All
diagnostics quantize only K, leave V in fp16, and use no recent-token buffer,
isolating key-rate allocation. Appendix~\ref{app:attention-details} gives the
model/path scope and protocol.

We first inspect the learned allocation at 3 logical K bits.
Figure~\ref{fig:allocation-fingerprint} shows heterogeneous mixtures of integer
widths that vary across layers and architectures.

\begin{figure}[H]
  \centering
  \includegraphics[width=\linewidth]{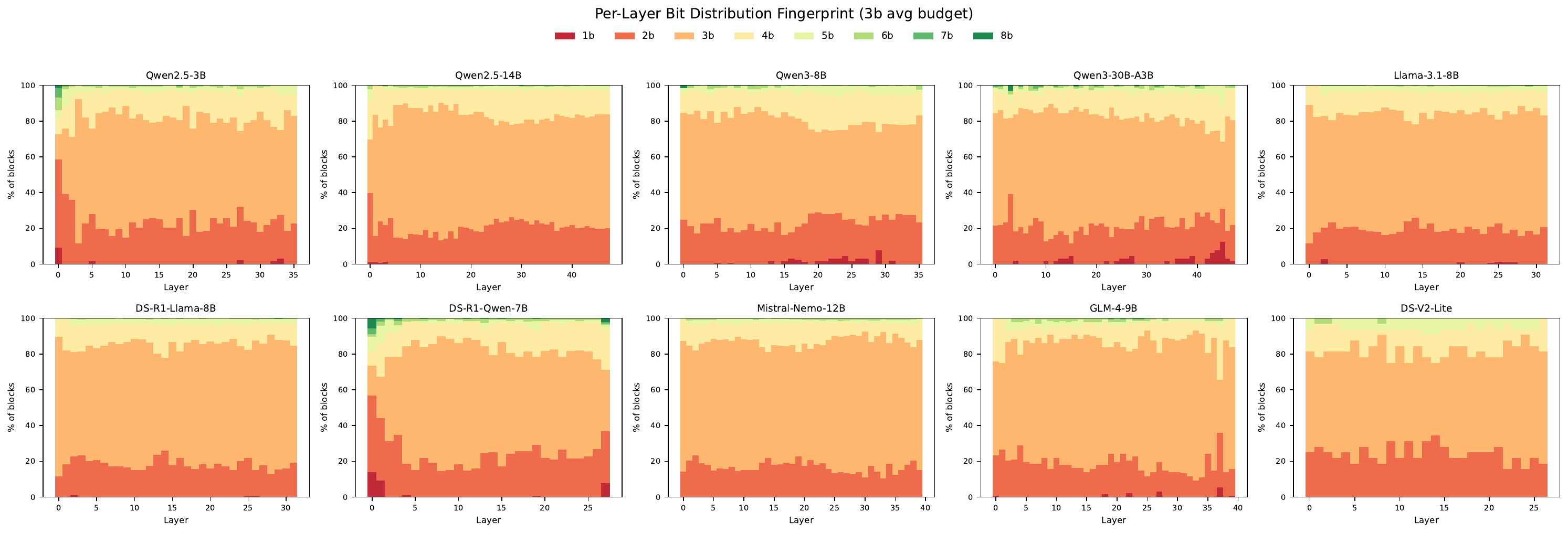}
  \caption{\textbf{Block-GTQ allocation profiles at 3 logical K bits.}
  Each panel shows the per-layer distribution of bit widths across
  two-dimensional RoPE blocks.}
  \label{fig:allocation-fingerprint}
  \vspace{-5mm}
\end{figure}

We then measure layerwise RoPE-logit error. Relative to uniform TQ-MSE,
Block-GTQ reduces
layer-averaged MAE by 32--79\% across all ten configurations and wins every
layer comparison at 3 bits (367/367; Table~\ref{tab:logit-mae-3b}). The same
367/367 pattern holds at 2 bits (Table~\ref{tab:logit-mae-2b}).
\begin{table}[H]
  \centering
  \caption{\textbf{RoPE-logit fidelity at 3 logical K bits.} Layer-averaged
  MAE under K-only quantization; lower is better. $\Delta$ is the relative
  reduction over uniform TQ-MSE, and Wins counts layers improved by Block-GTQ.
  GLM-4-9B uses partial-RoPE logits; DS-V2-Lite uses its RoPE branch.}
  \label{tab:logit-mae-3b}
  \scriptsize
  \setlength{\tabcolsep}{3.6pt}
  \begin{tabular*}{\linewidth}{@{\extracolsep{\fill}}lrrrc@{\hspace{1.5em}}lrrrc@{}}
    \toprule
    Model & TQ & Block & $\Delta$ & Wins & Model & TQ & Block & $\Delta$ & Wins \\
    \midrule
    Qwen2.5-3B & 6.43 & \textbf{3.23} & 49.9\% & 36/36 &
    DS-R1-Llama-8B & 3.44 & \textbf{2.33} & 32.2\% & 32/32 \\
    Qwen2.5-14B & 4.14 & \textbf{2.61} & 37.1\% & 48/48 &
    DS-R1-Qwen-7B & 11.44 & \textbf{2.40} & 79.1\% & 28/28 \\
    Qwen3-8B & 5.81 & \textbf{2.96} & 49.0\% & 36/36 &
    Mistral-Nemo-12B & 3.46 & \textbf{2.28} & 34.2\% & 40/40 \\
    Qwen3-30B-A3B & 6.76 & \textbf{3.00} & 55.6\% & 48/48 &
    GLM-4-9B & 7.30 & \textbf{4.51} & 38.2\% & 40/40 \\
    Llama-3.1-8B & 3.80 & \textbf{2.55} & 32.7\% & 32/32 &
    DS-V2-Lite & 6.01 & \textbf{3.87} & 35.5\% & 27/27 \\
    \bottomrule
  \end{tabular*}
  \vspace{-5mm}
\end{table}

Finally, we test whether the logit improvements propagate through softmax. At
2, 3, and 4 bits, Block-GTQ achieves the lowest mean KL to fp16 and the highest
mean top-10 overlap across configurations
(Figure~\ref{fig:attention-diagnostic}). Appendix~\ref{app:attention-details}
reports the 2-bit logit results and all per-configuration attention metrics.

\begin{figure}[H]
  \centering
  \includegraphics[width=\linewidth]{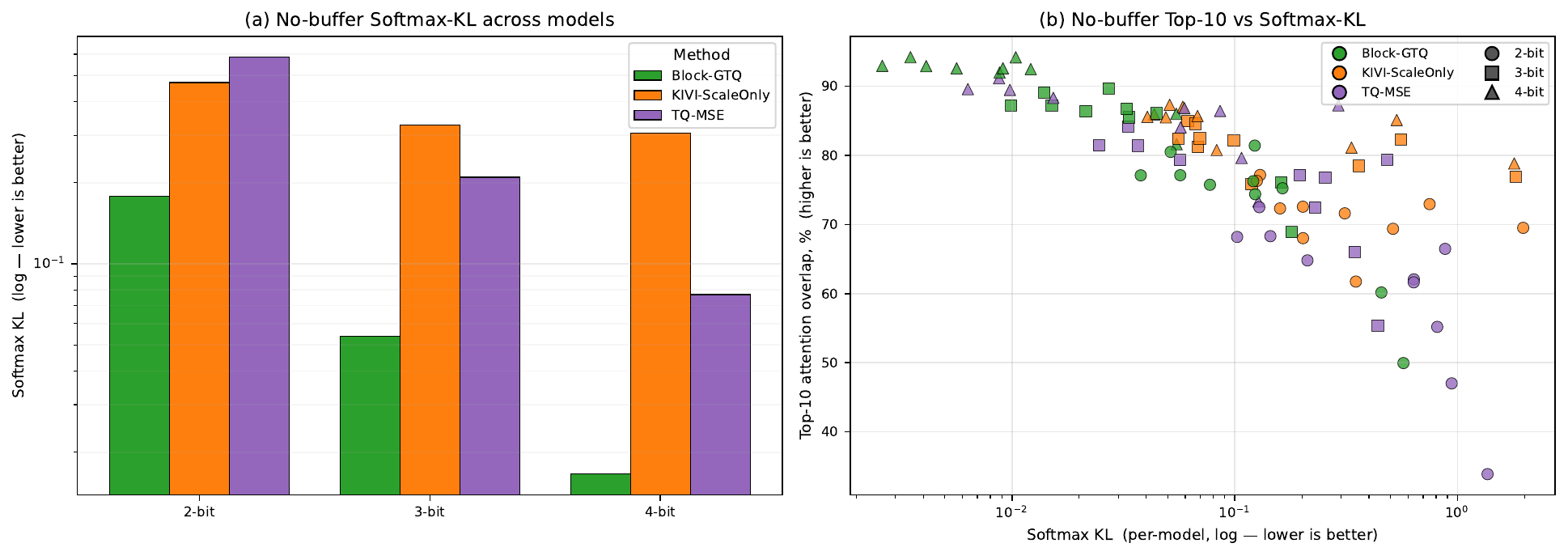}
  \caption{\textbf{Attention fidelity under K-only quantization without a
  recent-token buffer.}
  \textbf{(a)} Mean softmax KL to fp16 across ten configurations.
  \textbf{(b)} Top-10 overlap versus KL; upper left is better.}
  \label{fig:attention-diagnostic}
  \vspace{-5mm}
\end{figure}
  \vspace{-2mm}
\subsection{Calibration robustness}
\label{sec:calibration-robustness}
  \vspace{-2mm}

We vary calibration length using WikiText-2~\citep{merity2017pointer}
prefixes of 64--2048 tokens. At a fixed 2048-token budget, we compare
WikiText-2, PG19~\citep{rae2020compressive}, C4~\citep{raffel2020exploring},
and code as
calibration sources across six NIAH subtasks; on AIME, we compare general-text
calibration using WikiText-2 with math-domain calibration using the MATH
training split~\citep{hendrycks2021math}.

\begin{table}[H]
\vspace{-2mm}
  \centering
  \caption{Calibration robustness. Panels (a) and (b) report the six-task
  NIAH average (\%) on Llama-3.1-8B-Instruct; panel (c) reports Block-GTQ K3V2
  AIME avg@8 (\%) under the \emph{Stress} setting.}
  \label{tab:calibration-main}
  \begin{subtable}[t]{0.49\linewidth}
    \centering\scriptsize
    \setlength{\tabcolsep}{3pt}
    \caption{Calibration length: $\Delta$ vs. 2048-token WikiText-2.}
    \label{tab:calibration-length}
    \begin{tabular*}{\linewidth}{@{\extracolsep{\fill}}lcccc}
      \toprule
      & \multicolumn{2}{c}{K2V2} & \multicolumn{2}{c}{K3V3} \\
      \cmidrule(lr){2-3}\cmidrule(lr){4-5}
      $N_{\mathrm{cal}}$ & Overall & $\Delta$ & Overall & $\Delta$ \\
      \midrule
      64   & 95.68 & $-1.68$ & 98.46 & $+0.06$ \\
      128  & 94.28 & $-3.08$ & 98.63 & $+0.23$ \\
      256  & 94.08 & $-3.28$ & 97.33 & $-1.07$ \\
      512  & 95.79 & $-1.57$ & 98.01 & $-0.39$ \\
      1024 & 93.27 & $-4.09$ & 98.29 & $-0.11$ \\
      \textbf{2048} & \textbf{97.36} & (base) & \textbf{98.40} & (base) \\
      \bottomrule
    \end{tabular*}
  \end{subtable}%
  \hfill
  \begin{subtable}[t]{0.49\linewidth}
    \centering\scriptsize
    \setlength{\tabcolsep}{3pt}
    \renewcommand{\arraystretch}{1.1}
    \caption{Calibration domain (NIAH): $\Delta$ vs. WikiText-2; 2048 tokens
    per source.}
    \label{tab:calibration-corpus}
    \begin{tabular*}{\linewidth}{@{\extracolsep{\fill}}lcccc}
      \toprule
      & \multicolumn{2}{c}{K2V2} & \multicolumn{2}{c}{K3V3} \\
      \cmidrule(lr){2-3}\cmidrule(lr){4-5}
      Corpus & Overall & $\Delta$ & Overall & $\Delta$ \\
      \midrule
      WikiText-2 & 97.36 & (base) & 98.40 & (base) \\
      PG19 & 97.08 & $-0.28$ & 98.12 & $-0.28$ \\
      C4 & 94.58 & $-2.78$ & 98.23 & $-0.17$ \\
      code & 93.66 & $-3.70$ & 98.06 & $-0.34$ \\
      \bottomrule
      
    \end{tabular*}
  \end{subtable}
  \par\smallskip
  \begin{subtable}[t]{0.98\linewidth}
    \centering\scriptsize
    \setlength{\tabcolsep}{4pt}
    \caption{Calibration domain (AIME):
    $\Delta=\text{MATH}-\text{WikiText-2}$ before rounding; 2048 tokens per source.}
    \label{tab:aime-calibration}
    \begin{tabular*}{\linewidth}{@{\extracolsep{\fill}}lrrrrrr@{}}
      \toprule
      & \multicolumn{3}{c}{AIME 2024} & \multicolumn{3}{c}{AIME 2025} \\
      \cmidrule(lr){2-4}\cmidrule(lr){5-7}
      Model & WikiText-2 & MATH & $\Delta$ & WikiText-2 & MATH & $\Delta$ \\
      \midrule
      DeepSeek-R1-Distill-Llama-8B & 27.50 & 27.50 & $0.00$ & 25.42 & 25.83 & $+0.42$ \\
      DeepSeek-R1-Distill-Qwen-7B  & 49.58 & 51.67 & $+2.08$ & 38.75 & 37.50 & $-1.25$ \\
      \bottomrule
    \end{tabular*}
  \end{subtable}
  \vspace{-2mm}
\end{table}

On this NIAH panel, K2V2 is more sensitive to calibration length and source
than K3V3, with no monotone length trend.
Appendices~\ref{app:calibration-data} and~\ref{app:aime-calibration-corpus}
give the detailed results.

\vspace{-2mm}
\subsection{Block-score ablation}
\label{sec:score-ablation}
\vspace{-2mm}

We compare eight block scores on Llama-3.1-8B-Instruct at K2V2, holding
the codec, allocator, and evaluation instances fixed.
Table~\ref{tab:score-ablation} shows that the default mean Q/K energy score
has the highest observed average ($97.36$); the direct RoPE-logit score trails
by $0.14$ points but requires relative-offset sampling. Full definitions and
per-task results appear in Appendix~\ref{app:energy-score-ablation}.

\begin{table}[H]
\vspace{-2mm}
  \centering
  \caption{Block-score ablation on Llama-3.1-8B-Instruct at K2V2
  (six-task NIAH average, \%).}
  \vspace{-2mm}
  \label{tab:score-ablation}
  \scriptsize
  \setlength{\tabcolsep}{1.5pt}
  \begin{minipage}[t]{0.485\linewidth}
    \centering
    \begin{tabular*}{\linewidth}{@{\extracolsep{\fill}}llr@{}}
      \toprule
      Score & Definition & Avg. \\
      \midrule
      Mean Q/K energy (default) & $\frac12(\E\|\bq^{(i)}\|^2+\E\|\bk^{(i)}\|^2)$ & \textbf{97.36} \\
      Direct RoPE logit & $\E|\bq^{(i)\top}R(\Delta\theta_i)\bk^{(i)}|$ & 97.22 \\
      Paired norm product & $\E[\|\bq^{(i)}\|_2\|\bk^{(i)}\|_2]$ & 89.84 \\
      Weight only & $\E_{g\in G(h)}\|W_{q,g}^{(i)\top}W_{k,h}^{(i)}\|_{\rm F}$ & 85.10 \\
      \bottomrule
    \end{tabular*}
  \end{minipage}%
  \hspace{0.02\linewidth}%
  \begin{minipage}[t]{0.485\linewidth}
    \centering
    \begin{tabular*}{\linewidth}{@{\extracolsep{\fill}}llr@{}}
      \toprule
      Score & Definition & Avg. \\
      \midrule
      Maximum & $\max\{\E\|\bq^{(i)}\|^2,\E\|\bk^{(i)}\|^2\}$ & 96.77 \\
      Key only & $\E\|\bk^{(i)}\|^2$ & 95.29 \\
      Geometric mean & $\sqrt{\E\|\bq^{(i)}\|^2\E\|\bk^{(i)}\|^2}$ & 95.12 \\
      Query only & $\E\|\bq^{(i)}\|^2$ & 94.05 \\
      \bottomrule
    \end{tabular*}
  \end{minipage}
  \vspace{-4mm}
\end{table}

\FloatBarrier
 \vspace{-4mm}
\subsection{Long-context evaluation}
 \vspace{-4mm}
\label{sec:long-context}
\begin{figure}[H]
  \centering
  \includegraphics[width=\linewidth]{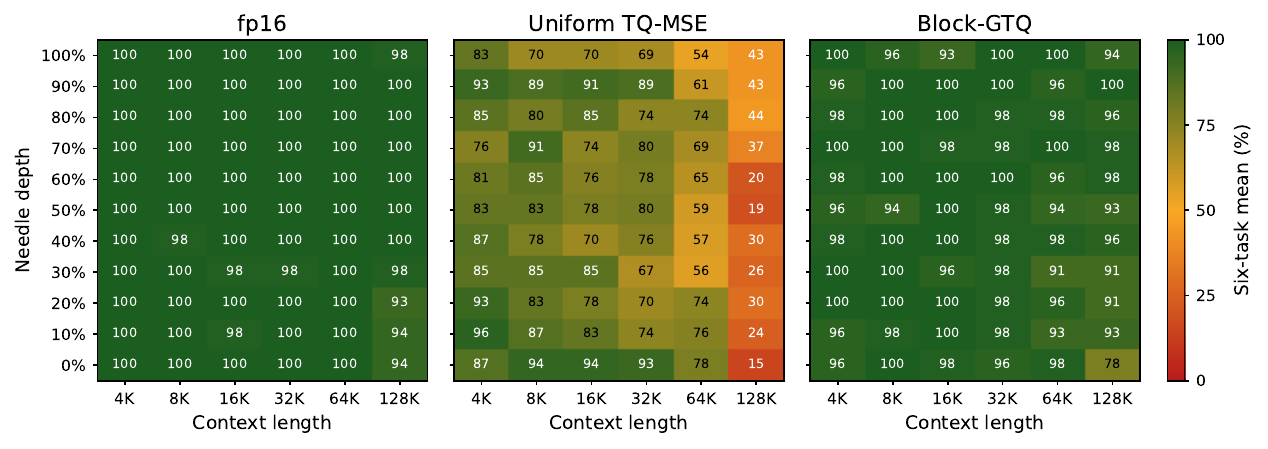}
    \vspace{-4mm}
  \caption{\textbf{Six-task NIAH mean at K2V2 on
  Llama-3.1-8B-Instruct.}}
  \label{fig:niah-heatmap-main}
  \vspace{-4mm}
\end{figure}
\vspace{-2mm}
At K2V2, Block-GTQ tracks fp16 more closely than uniform TQ-MSE across lengths
and depths, especially at longer contexts
(Figure~\ref{fig:niah-heatmap-main}).
Table~\ref{tab:long-context-summary} reports the aggregate results.

\begin{table}[H]
 \vspace{-2mm}
  \centering
  \caption{Results on NIAH~\citep{hsieh2024ruler} and
  LongBench-EN~\citep{bai2024longbench}. NIAH averages six subtasks--single-needle, distractor, multi-needle, multi-key,
  multi-value, and multi-query---over
  4K--128K and 11 depths (\%; three trials per cell). LongBench-EN averages 16 subtasks.}
  \label{tab:long-context-summary}
  \vspace{-2mm}
  \begin{subtable}[t]{0.485\linewidth}
    \centering
    \caption{NIAH.}
    \label{tab:niah-summary-main}
    \scriptsize
    \renewcommand{\arraystretch}{0.92}
    \setlength{\tabcolsep}{2.5pt}
    \begin{tabular*}{\linewidth}{@{\extracolsep{\fill}}llrrr@{}}
      \toprule
      Backbone & Rate & fp16 & TQ-MSE & Block-GTQ \\
      \midrule
      \multirow{2}{*}{Llama-3.1-8B-Instruct}
        & K3V3 & 99.55 & 97.73 & \textbf{98.40} \\
        & K2V2 & 99.55 & 70.62 & \textbf{97.36} \\
      \midrule
      \multirow{2}{*}{Qwen2.5-7B-Instruct}
        & K3V3 & 72.22 & 0.00 & \textbf{72.25} \\
        & K2V2 & 72.22 & 0.00 & \textbf{61.70} \\
      \bottomrule
    \end{tabular*}
  \end{subtable}
  \hfill
  \begin{subtable}[t]{0.485\linewidth}
    \centering
    \caption{LongBench-EN.}
    \label{tab:longbench-summary-main}
    \scriptsize
    \renewcommand{\arraystretch}{0.92}
    \setlength{\tabcolsep}{2.5pt}
    \begin{tabular*}{\linewidth}{@{\extracolsep{\fill}}llrrr@{}}
      \toprule
      Backbone & Rate & fp16 & TQ-MSE & Block-GTQ \\
      \midrule
      \multirow{2}{*}{Llama-3.1-8B-Instruct}
        & K3V3 & 49.73 & 47.37 & \textbf{48.56} \\
        & K2V2 & 49.73 & 33.67 & \textbf{44.89} \\
      \midrule
      \multirow{2}{*}{Llama-3.1-70B-Instruct}
        & K3V3 & 52.66 & 51.77 & \textbf{52.39} \\
        & K2V2 & 52.66 & 47.20 & \textbf{50.14} \\
      \bottomrule
    \end{tabular*}
  \end{subtable}
  \vspace{-4mm}
\end{table}

At K2V2, Block-GTQ improves over uniform TQ-MSE by $26.74$ points on
Llama-3.1-8B NIAH and $11.22$/$2.94$ on 8B/70B LongBench-EN. It avoids
the zero NIAH scores for uniform TQ-MSE on Qwen2.5-7B-Instruct.
Appendix~\ref{app:long-context-details} reports results;
Appendix~\ref{app:qwen-tqmse-collapse} presents a 4K NIAH failure
case study.
After merging the LoRA adapter into Llama-3.1-8B-Instruct and
recalibrating, Block-GTQ retains a $26.29$-point advantage over uniform
TQ-MSE on six-task NIAH at K2V2
(Appendix~\ref{app:lora-recalibration}).

\vspace{-4mm}
\subsection{Long-chain reasoning}
\label{sec:aime}
\vspace{-2mm}
We evaluate two DeepSeek-R1 distilled models~\citep{deepseekai2025r1} on AIME
2024/2025.
\begin{table}[H]
  \centering
  \caption{\textbf{AIME 2024/2025 at nominal K3V2} (avg@8, \%).
  Bold marks the highest reported quantized score in each column. Block-GTQ
  and uniform TQ-MSE form the controlled allocation comparison. Cross-family
  methods are adapted to the Stress and Protected regimes where supported.}
  \vspace{-3mm}
  \label{tab:aime-k3v2-main}
  \scriptsize
  \renewcommand{\arraystretch}{0.92}
  \setlength{\tabcolsep}{2.4pt}

  \begin{minipage}[t]{0.49\textwidth}
    \centering
    \textit{(a) Stress }\par\smallskip
    \begin{tabular*}{\linewidth}{@{\extracolsep{\fill}}lrrrr@{}}
      \toprule
      & \multicolumn{2}{c}{DS-R1-Llama-8B}
      & \multicolumn{2}{c}{DS-R1-Qwen-7B} \\
      \cmidrule(lr){2-3}\cmidrule(lr){4-5}
      Method & '24 & '25 & '24 & '25 \\
      \midrule
      \multicolumn{5}{@{}l}{\textit{Controlled allocation comparison}} \\
      fp16                    & 43.33 & 28.75 & 54.17 & 37.92 \\
      Uniform TQ-MSE          & 26.25 & 20.00 &  0.00 &  0.00 \\
      Block-GTQ               & 27.50 & 25.42
                              & \textbf{49.58} & \textbf{38.75} \\
      \midrule
      \multicolumn{5}{@{}l}{\textit{Cross-family context}} \\
      KIVI-ScaleOnly          &  7.50 & 10.42 & 28.75 & 17.50 \\
      AATC                    &  4.58 &  3.75 & 25.00 & 23.75 \\
      AATC (variance-only)    &  1.25 &  0.83 & 25.00 & 24.17 \\
      AQUA-KV$^{*}$           & \textbf{37.92} & \textbf{29.17} & 0.00 & 0.00 \\
      Kitty$^{\dagger}$       & -- & -- & -- & -- \\
      CommVQ                  &  7.50 &  6.67 & 14.58 & 17.92 \\
      KVQuant-0.01\%          &  1.25 &  0.00 & 25.42 & 20.83 \\
      PM-KVQ (matched)$^{\ddagger}$
                              & 15.42 &  3.33 & 18.33 &  4.17 \\
      \bottomrule
    \end{tabular*}
  \end{minipage}
  \hfill
  \begin{minipage}[t]{0.49\textwidth}
    \centering
    \textit{(b) Protected }\par\smallskip
    \begin{tabular*}{\linewidth}{@{\extracolsep{\fill}}lrrrr@{}}
      \toprule
      & \multicolumn{2}{c}{DS-R1-Llama-8B}
      & \multicolumn{2}{c}{DS-R1-Qwen-7B} \\
      \cmidrule(lr){2-3}\cmidrule(lr){4-5}
      Method & '24 & '25 & '24 & '25 \\
      \midrule
      \multicolumn{5}{@{}l}{\textit{Controlled allocation comparison}} \\
      fp16                    & 43.33 & 28.75 & 54.17 & 37.92 \\
      Uniform TQ-MSE          & 37.92 & 26.67 &  7.50 &  8.75 \\
      Block-GTQ               & \textbf{43.75} & \textbf{30.00}
                              & \textbf{54.17} & \textbf{37.50} \\
      \midrule
      \multicolumn{5}{@{}l}{\textit{Cross-family context}} \\
      KIVI                    & 43.33 & 27.08 & 52.92 & 35.83 \\
      AATC                    & 35.00 & 29.58 & 36.25 & 28.33 \\
      AATC (variance-only)    & 40.42 & 27.50 & 43.33 & 26.25 \\
      AQUA-KV$^{*}$          & 43.33 & \textbf{30.00} & 0.00 & 0.00 \\
      Kitty                   & 40.83 & 28.33 & 50.42 & \textbf{37.50} \\
      CommVQ                  & 33.33 & 25.42 & 37.92 & 29.17 \\
      KVQuant-0.5\%           & 16.67 & 20.42 & 39.58 & 27.92 \\
      PM-KVQ (matched)$^{\ddagger}$
                              & 19.58 &  6.25 & 18.33 &  6.67 \\
      \bottomrule
    \end{tabular*}
  \end{minipage}

  \smallskip
  \parbox{0.98\textwidth}{\scriptsize\raggedright
  $^{*}$AQUA-KV exhibits generation collapse on DS-R1-Qwen-7B.
  $^{\dagger}$Kitty requires an fp16 buffer, and smaller groups
  increase metadata cost; no rate-matched Stress setting is available.
  $^{\ddagger}$PM-KVQ uses offline reference generation lengths to set
  per-problem memory budgets. Method-specific settings and rate accounting are detailed in
  Appendix~\ref{app:aime-methods} and Table~\ref{tab:aime-rate-ledger}.}
  \vspace{-2mm}
\end{table}

In the controlled comparison, Block-GTQ improves on uniform TQ-MSE in every
cell. The largest gap occurs on Qwen under Stress: uniform TQ-MSE scores
$0.00/0.00$, whereas Block-GTQ reaches $49.58/38.75$. Under Protected,
Block-GTQ's four reported point estimates stay within $1.25$ points of
fp16 and are highest or tied-highest among quantized methods.
Appendix~\ref{app:aime-extended} gives protocol and cross-family details;
Appendix~\ref{app:qwen-tqmse-collapse} examines a separate 4K NIAH failure
case on Qwen2.5-7B-Instruct.

In additional MATH-500 comparisons, Block-GTQ improves avg@8 over uniform
TQ-MSE by 13.2 points on GLM-4-9B (partial-RoPE) and 6.0 points on
DeepSeek-V2-Lite-Chat (MLA), using fixed $(2,4)$ nonrotary/rotary branch
rates under Stress (Appendix~\ref{app:architecture-adapters}).

\FloatBarrier
\vspace{-2mm}
\subsection{Packed deployment}
\label{sec:deployment}
\vspace{-2mm}
We benchmark Qwen2.5-3B-Instruct on one H800 at batch size 1, comparing
native-GQA FP16 FlashAttention-2 with packed uniform TQ-MSE and Block-GTQ
at K3V3. The packed methods share the serving framework but differ in
key-rate allocation; all three methods use whole-step CUDA graphs.
At 128K, Block-GTQ reduces resident KV footprint from 4.83 to 1.49 GB
($3.24\times$) while maintaining WikiText-2 perplexity close to FP16
(9.80 versus 9.34), whereas uniform TQ-MSE's perplexity reaches 9319.53.
Block-GTQ's median decode latency is 9.50 versus FP16's 8.99 ms/step at
128K (5.7\% overhead), but falls below FP16 at 512K and 1M, with latency
reductions of 2.1\% and 8.0\%, respectively. At 1M-token context, it takes 17.94 versus
19.50 ms/step.
Appendix~\ref{app:deployment-details} gives the full protocol and
measurements.

\begin{figure}[h]
  \centering
  \includegraphics[width=\linewidth]{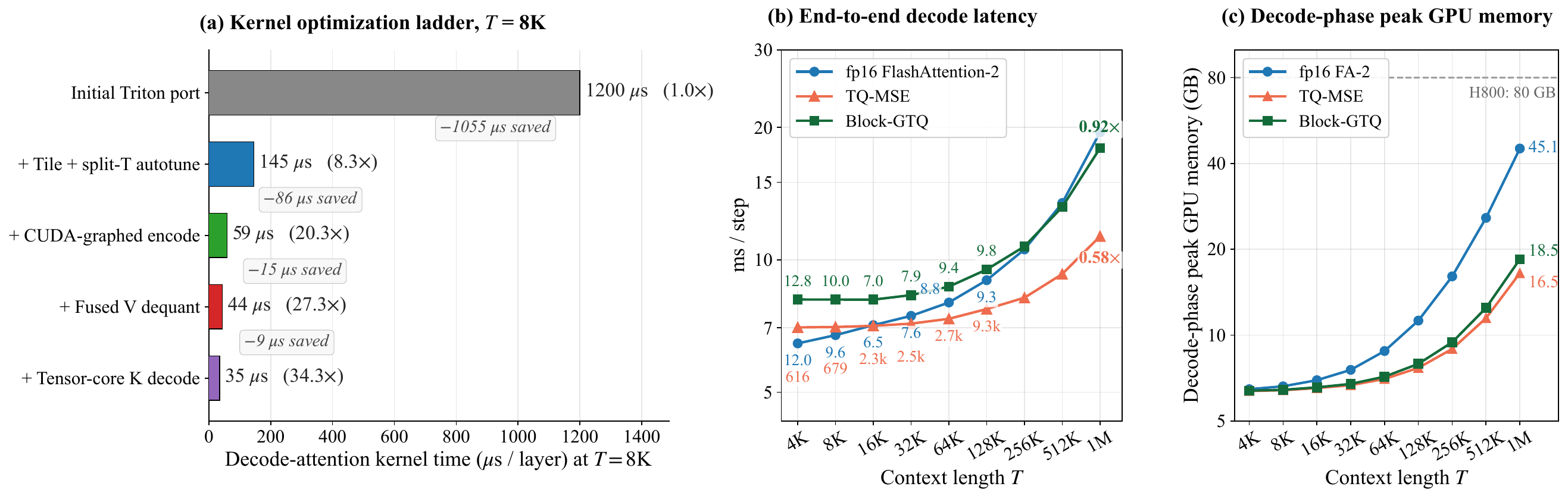}
  \caption{\textbf{Qwen2.5-3B-Instruct K3V3 deployment on one H800.}
  (a) Kernel optimization at 8K; further packed-serving optimizations are
  described in Appendix~\ref{app:deployment-optimization}.
  (b) End-to-end decode latency; annotations give perplexity, evaluated through 128K.
  (c) Decode-phase peak GPU memory.}
  \label{fig:deployment-main}
  \vspace{-4mm}
\end{figure}
\FloatBarrier

\section{Conclusion}
\label{sec:conclusion}

Block-GTQ uses two-dimensional RoPE frequency blocks as the unit for analyzing
query--key logit error and allocating key precision. Its TQ-MSE instantiation
combines a label-free Q/K score, exact optimization of the separable allocation
proxy under the stated assumptions, and equal-rate grouping for packed serving.
Against uniform TQ-MSE at matched logical rates, it better preserves
query--key logits and attention distributions, with gains on NIAH,
LongBench-EN, and AIME in controlled low-rate comparisons. Compared with
FP16 FlashAttention-2, our packed serving path reduces resident KV memory
and, at the longest tested contexts, decode latency, without materializing
a decoded fp16 KV cache in HBM.

\paragraph{Limitations.}
Our main K/V-cache formulation takes separate average key and value rates as
inputs, allocates precision across key blocks, and quantizes values uniformly.
Joint optimization of the K/V budget split and value-side rate allocation
remains unexplored. We instantiate the allocator with TQ-MSE, and evaluate
end-to-end packed serving only on Qwen2.5-3B-Instruct at batch size 1
using one H800. Other
codecs and batched or distributed serving remain future work.
We evaluate accuracy and efficiency, but do not assess whether quantization
preserves safety-related behavior.

\section*{Declaration of generative AI use}
Generative AI tools assisted with experimental design and protocol feedback,
experiment suggestions, manuscript preparation, and consistency checks of
results, LaTeX sources, and mathematical arguments. The method and mathematical
claims were developed by the authors, who independently verified the arguments
and results, reviewed the cited literature, and take responsibility for all
content.

\bibliography{references}
\bibliographystyle{blockgtq_arxiv}

\appendix
\section*{Appendix Roadmap}
Appendix~\ref{app:theory} gives the coordinate change, blockwise bound,
allocation-proxy derivation, and greedy proof.
Appendix~\ref{app:protocols} records rate accounting and the calibration,
NIAH, LongBench, and deployment protocols.
Appendix~\ref{app:attention-details} reports the diagnostic scope and per-model
logit and attention metrics, including a 4K NIAH failure case study on
Qwen2.5-7B-Instruct.
Appendix~\ref{app:calibration} gives calibration and score ablations plus a
LoRA-adaptation comparison, and Appendix~\ref{app:long-context-details} provides the complete
NIAH and LongBench results plus key-versus-value rate sensitivity.
Appendix~\ref{app:aime-extended} provides the AIME calibration swap, K3V2 rate
ledger, and extended results.
Appendix~\ref{app:architecture-adapters} describes partial-RoPE and MLA
interfaces and reports their matched branch-codec comparisons on MATH-500.
Appendix~\ref{app:packed-details1} describes the packed-cache implementation,
and Appendix~\ref{app:deployment-details} gives complete packed-system tables.

\section{Theory and Allocation Details}
\label{app:theory}

\subsection{Post-RoPE cache and pre-RoPE coordinates}
\label{app:coordinate-change}

Let $R_t$ denote only the orthogonal RoPE rotation at position $t$, with any
scalar amplitude treated separately below. For unit-amplitude RoPE, the
post-RoPE key is $\bk_m^{\mathrm R}=R_m\bk_m$. If the compressed cache decodes
it as $\hbk_m^{\mathrm R}$, define the equivalent pre-RoPE reconstruction and
the error in the two coordinate systems by
\begin{equation}
  \hbk_m=R_m^\top\hbk_m^{\mathrm R},
  \qquad
  \be_m^{\mathrm R}=\bk_m^{\mathrm R}-\hbk_m^{\mathrm R},
  \qquad
  \be_m=\bk_m-\hbk_m=R_m^\top\be_m^{\mathrm R}.
  \label{eq:post-rope-error-map}
\end{equation}
Consequently, orthogonality within every two-dimensional block gives
\begin{equation}
  \|\be_m^{(i)}\|_2
  =\|\be_m^{\mathrm R,(i)}\|_2.
  \label{eq:post-rope-error-norm}
\end{equation}
For the fixed-frequency RoPE parameterization considered here,
$R_n^\top R_m=R_{m-n}$. Hence, for a query at position $n$,
\begin{align}
  (R_n\bq_n)^\top\hbk_m^{\mathrm R}
  &=\bq_n^\top R_n^\top R_m\hbk_m \notag\\
  &=\bq_n^\top R_{m-n}\hbk_m.
  \label{eq:post-to-pre-rope}
\end{align}
Thus a deployed post-RoPE reconstruction can be analyzed as a pre-RoPE key at
relative offset $m-n$, with exactly the same blockwise reconstruction-error
norms. Norm preservation and Lemma~\ref{lem:block-error} require only blockwise
orthogonality; the notation $R_{m-n}$ additionally uses the relative-position
identity above.

\paragraph{Scalar-scaled RoPE.}
For a shared fixed amplitude $a>0$ on Q and K, write
$\bq_n^{\mathrm R}=aR_n\bq_n$ and
$\bk_m^{\mathrm R}=aR_m\bk_m$. This covers the static-YaRN setting in
Appendix~\ref{app:qwen-yarn}. The corresponding raw pre-RoPE reconstruction
and error satisfy
\begin{equation}
  \hbk_m=a^{-1}R_m^\top\hbk_m^{\mathrm R},
  \qquad
  \be_m=a^{-1}R_m^\top\be_m^{\mathrm R},
  \qquad
  \|\be_m^{\mathrm R,(i)}\|_2=a\|\be_m^{(i)}\|_2.
  \label{eq:scaled-rope-error-map}
\end{equation}
Consequently, the unscaled dot-product perturbation is
\begin{align}
  \delta\mathcal K_\Delta^{\mathrm R}
  &=(\bq_n^{\mathrm R})^\top\be_m^{\mathrm R}
    =a^2\sum_i\bq_n^{(i)\top}R(\Delta\theta_i)\be_m^{(i)},\notag\\
  |\delta\mathcal K_\Delta^{\mathrm R}|
  &\leq a^2\sum_i\|\bq_n^{(i)}\|_2\|\be_m^{(i)}\|_2,
  \qquad \Delta=m-n.
  \label{eq:scaled-rope-bound}
\end{align}
Equivalently, Lemma~\ref{lem:block-error} applies unchanged to the effective
pre-rotation vectors $a\bq_n,a\bk_m$ and reconstruction
$R_m^\top\hbk_m^{\mathrm R}$. Norm preservation then refers to these effective
coordinates, not the raw pre-RoPE activations. Our energy scores use the raw
pre-RoPE activations. Including the common amplitude would multiply every
score in a layer and KV head by $a^2$, leaving its fixed-budget allocation
unchanged.

\subsection{Proof of the blockwise logit-error bound}
\label{app:block-proof}

\begin{proof}[Proof of Lemma~\ref{lem:block-error}]
Using $\be_{\bk}^{(i)}=\bk^{(i)}-\hbk^{(i)}$ and
Equation~\ref{eq:exact-block-error},
\begin{align}
  \left|\delta\mathcal K_\Delta\right|
  &=\left|\sum_i
    \bq^{(i)\top}R(\Delta\theta_i)\be_{\bk}^{(i)}\right|\\
  &\leq\sum_i
    \left|\bq^{(i)\top}R(\Delta\theta_i)\be_{\bk}^{(i)}\right|\\
  &\leq\sum_i
    \|\bq^{(i)}\|_2
    \|R(\Delta\theta_i)\be_{\bk}^{(i)}\|_2\\
  &=\sum_i
    \|\bq^{(i)}\|_2\|\be_{\bk}^{(i)}\|_2.
\end{align}
The first inequality is the triangle inequality, the second is
Cauchy--Schwarz, and the final equality follows from rotation orthogonality.
\end{proof}

\subsection{From the block bound to the allocation proxy}
\label{app:proxy-derivation}

\paragraph{Marginal score bound.}
Fix a layer and KV head and suppress their indices. For any joint distribution
$P$ of query--key pairs with finite second moments, define
\begin{align}
  s_i^\star(P)
  &=\E_P\!\left[\|\bq^{(i)}\|_2\|\bk^{(i)}\|_2\right],
  \label{eq:paired-product-weight}\\
  s_i(P)
  &=\frac12\left(
    \E_P\|\bq^{(i)}\|_2^2+\E_P\|\bk^{(i)}\|_2^2
  \right).
  \label{eq:distribution-score}
\end{align}
Pointwise AM--GM gives
\begin{equation}
  s_i^\star(P)\leq s_i(P),
  \label{eq:marginal-score-bound}
\end{equation}
without requiring query--key independence. If $P$ has the respective empirical
query and key marginals used in Equation~\ref{eq:block-score}, then $s_i(P)=s_i$,
where $s_i$ denotes the calibration score. Thus $s_i$ upper-bounds
$s_i^\star(P)$ for every coupling of those marginals, without collecting
explicit pairs or labels. The calibration query marginal averages over all
query heads $g\in G(h)$ served by KV head $h$, and no monotone prior is imposed
across frequency indices. For a deployment distribution with different
marginals, AM--GM still bounds the paired moment by that distribution's own
second moments; the fixed calibration score is a proxy, not a guaranteed upper
bound on that paired moment.

\paragraph{Conditional error interpretation.}
The connection to key-reconstruction error can be made explicit under an
\emph{idealized block-local error model}. For a fixed rate assignment, suppose
finite deterministic factors $\alpha_i^{\mathcal E}(b_i)\geq0$ satisfy
$\|\be_{\bk}^{(i)}\|_2\leq
\alpha_i^{\mathcal E}(b_i)\|\bk^{(i)}\|_2$ for every key and every realization
of the codec randomness. Taking expectations of Lemma~\ref{lem:block-error}
at any fixed relative offset $\Delta$ gives
\begin{equation}
  \E_{P,\omega}\!\left[|\delta\mathcal K_\Delta|\right]
  \leq\sum_i\alpha_i^{\mathcal E}(b_i)s_i^\star(P)
  \leq\sum_i\alpha_i^{\mathcal E}(b_i)s_i(P),
  \label{eq:expected-absolute-bound}
\end{equation}
where $\omega$ denotes codec randomness, when present. This conditional model
motivates the product moment; it is not assumed for our grouped implementation.
An equal-rate group's rotation mixes RoPE blocks, so a block with zero input
norm need not have zero reconstruction error. Neither Lemma~\ref{lem:block-error}
nor Theorem~\ref{thm:greedy} requires this relative-error model.

\paragraph{Codec-informed allocation proxy.}
Let $\bx$ be a fixed input vector to the reference TQ-MSE codec of
\citet{zandieh2025turboquant}, and let $\hat\bx_b^{\mathrm{ref}}$ denote its
reconstruction at $b$ bits per coordinate. With their dimension-matched scalar
quantizers, the reported MSE bound is
\begin{equation}
  \E_\omega\|\bx-\hat\bx_b^{\mathrm{ref}}\|_2^2
  \leq C_{\mathrm{TQ}}\|\bx\|_2^2 4^{-b},
  \qquad C_{\mathrm{TQ}}=\frac{\sqrt{3}\pi}{2},
  \label{eq:tq-rate-law}
\end{equation}
for $b\geq1$, with expectation over the sampled random rotation and exact norm
restoration for nonunit inputs. Our implementation uses Gaussian Lloyd--Max
codebooks and decoded-direction renormalization
(Appendix~\ref{app:rate-accounting}); we do not assert the same bound for this
variant or a fixed rotation seed. We adopt only the reference rate dependence,
setting $\phi^{\mathrm{TQ}}(b)=4^{-b}$ in the proxy $\sum_i s_i\phi^{\mathrm{TQ}}(b_i)$.

The score and rate term have distinct roles: the block bound motivates
importance, while the reference MSE analysis motivates a rate penalty.
Their product is a separable allocation proxy, not an identification of
$\alpha_i^{\mathcal E}(b_i)$ with $4^{-b}$ or a bound on the grouped codec's
expected absolute logit error. Theorem~\ref{thm:greedy} exactly optimizes the
prescribed curves without requiring that they upper-bound deployed distortion.

A more codec-specific variant may estimate the complete curves
$F_i^{\mathcal E}(b)$ directly from calibration data rather than factorizing
them. The same optimizer and theorem apply when the measured curves remain
separable and satisfy Equation~\ref{eq:discrete-convexity}.

\subsection{Marginal-allocation algorithm}
\label{app:allocation-algorithm}

\begin{algorithm}[H]
\caption{Marginal RoPE-block allocation for one layer and KV head}
\label{alg:blockgtq}
\begin{algorithmic}[1]
  \Require fixed curves $F_1,\ldots,F_L$ and integer bounds
  $b_{\min},b_{\max}$ satisfying Theorem~\ref{thm:greedy};
  integer budget $B$ with $Lb_{\min}\leq B\leq Lb_{\max}$
  \Ensure integer rates $b_1,\ldots,b_L$
  \State $b_i\gets b_{\min}$ for every $i$
  \If{$B=Lb_{\min}$}
    \State \Return $b_1,\ldots,b_L$
  \EndIf
  \State insert each block with $b_i<b_{\max}$ into a max queue with
  $\Delta_i\gets F_i(b_i)-F_i(b_i+1)$
  \For{$r=1,\ldots,B-Lb_{\min}$}
    \State pop $i^\star$ with the largest $\Delta_i$
    \State $b_{i^\star}\gets b_{i^\star}+1$
    \If{$b_{i^\star}<b_{\max}$}
      \State $\Delta_{i^\star}\gets
      F_{i^\star}(b_{i^\star})-F_{i^\star}(b_{i^\star}+1)$
      \State reinsert $i^\star$ into the max queue
    \EndIf
  \EndFor
  \State \Return $b_1,\ldots,b_L$
\end{algorithmic}
\end{algorithm}

With a linear-time heap construction, the allocation costs
$O(L+(B-Lb_{\min})\log(1+L))$ time and $O(L)$ memory after score and penalty
collection.

\subsection{Exchange proof of optimal marginal allocation}
\label{app:greedy-proof}

\begin{proof}[Proof of Theorem~\ref{thm:greedy}]
Suppress the fixed layer, head, and codec indices. Start every block at
$b_{\min}$ and let $K=B-Lb_{\min}$ be the number of additional rate units.
If $K=0$, the initial allocation is the only feasible one, so the claim is
immediate. Otherwise, for block $i$, define the gain of its $r$th additional
rate unit as
\begin{equation}
  g_{i,r}
  =F_i(b_{\min}+r-1)-F_i(b_{\min}+r),
  \qquad
  r=1,\ldots,b_{\max}-b_{\min}.
  \label{eq:general-gain-chain}
\end{equation}
Label each increment by $e=(i,r)$ and write $g(e)=g_{i,r}$ for its gain.
Distinct labels remain distinct even when their gains tie.
Discrete convexity implies
$g_{i,1}\geq g_{i,2}\geq\cdots\geq0$.

For any feasible allocation, let
$k_i=b_i-b_{\min}\in\{0,\ldots,b_{\max}-b_{\min}\}$, so $\sum_i k_i=K$.
This selects the increments $(i,1),\ldots,(i,k_i)$ from block $i$, with an
empty prefix when $k_i=0$. By telescoping,
\begin{equation}
  J(\mathbf b)
  =\sum_{i=1}^{L}F_i(b_{\min})
  -\sum_{i=1}^{L}\sum_{r=1}^{k_i}g_{i,r}.
  \label{eq:gain-telescoping}
\end{equation}
Thus minimizing $J$ is equivalent to selecting exactly $K$ increments of
maximum total gain, subject to selecting a prefix from each block's chain.

Let $P_t$ be the set of the first $t$ increments chosen by
Algorithm~\ref{alg:blockgtq}. We prove by induction that, for each
$t\in\{0,\ldots,K\}$, an optimal feasible selection $O_t\supseteq P_t$
exists. The base case holds because $P_0=\varnothing$. For $t<K$, let $e$ be
the next increment selected by the algorithm, from block $a$. If $e\in O_t$,
take $O_{t+1}=O_t$. Otherwise $O_t\cup\{e\}$ preserves every prefix because
all predecessors of $e$ lie in $P_t$, but contains $K+1$ increments.

Since $|P_t|=t<K=|O_t|$, some block $j$ contributes an increment in
$O_t\setminus P_t$. Necessarily $j\neq a$: a feasible prefix omitting $e$
cannot contain any later increment of block $a$. Let $f$ be the first such
increment from block $j$, and let $f'$ be block $j$'s last increment in $O_t$.
Increment $f$ was available when $e$ was chosen, so $g(e)\geq g(f)$.
Diminishing gains give $g(f')\leq g(f)$. Removing $f'$ preserves the prefix
constraint and removes no element of $P_t$. Hence
$O_{t+1}=(O_t\setminus\{f'\})\cup\{e\}$ is feasible, has total gain at least
that of $O_t$, and contains $P_{t+1}=P_t\cup\{e\}$. It is therefore optimal.
At $t=K$, both sets contain $K$ increments, so $P_K=O_K$ is optimal.
\end{proof}

For the TQ-MSE proxy in Equation~\ref{eq:allocation-objective},
\begin{equation}
  g_{i,r}
  =\frac34 s_i4^{-(b_{\min}+r-1)},
\end{equation}
so each block exposes a geometrically decreasing gain chain.

\subsection{Scope of the allocation theorem}
\label{app:conditional-allocation}

Theorem~\ref{thm:greedy} requires separability, equal-cost increments,
contiguous integer rates, and nonnegative diminishing marginal gains. These are
sufficient mathematical conditions; they do not assert that every quantizer
satisfies them or require the prescribed penalties to upper-bound actual
distortion. If a codec yields separable but non-convex empirical curves, an
exact dynamic program over blocks and total rate solves the small allocation
problem. When optimizing actual distortion or physical storage, codebooks,
packing overhead, outlier handling, or joint transforms may couple multiple
blocks. Those quantities require an appropriate group-level formulation or a
coupled optimizer; this does not affect optimality for the prescribed separable
proxy.

\subsection{Equal-rate grouping}
\label{app:grouping-scope}

Equal-rate grouping exactly preserves the assigned logical widths. For each
nonempty group $\mathcal G_b^{(\ell,h)}$, all constituent coordinates are
encoded at $b$ bits using one shared TQ-MSE encoder. This retains RoPE-block
granularity in the allocation while amortizing normalization, rotation, and
metadata across blocks. Because each block contains two coordinates, the
resulting key-code payload is $2\sum_i b_i=2B=d_h\bar b$ bits per vector;
physical storage additionally includes group norms and container metadata. All
reported Block-GTQ results use this grouped realization.

\section{Experimental Protocols}
\label{app:protocols}

In the long-context tables, Llama-8B and Llama-70B denote
Llama-3.1-8B-Instruct and Llama-3.1-70B-Instruct, respectively; Qwen-7B denotes
Qwen2.5-7B-Instruct.

\subsection{Rate accounting}
\label{app:rate-accounting}

We distinguish three rates. \emph{Logical payload rate} is the K/V code bits
per dimension (b/d) and is summarized by K$k$V$v$. \emph{Total effective rate} adds
norms, scales, metadata, and fp16 sink, recent-token, or outlier regions,
divided by the realized cached dimensions. \emph{Allocated physical footprint}
is the tensor storage reserved by the implementation. Matched TQ-MSE comparisons
hold logical payload fixed; cross-family comparisons report total-rate
components; deployment compression uses allocated bytes.

For Block-GTQ, the logical key rate is
$\bar b=L^{-1}\sum_i b_i$. The normalized optimization budget
$B=L\bar b$ is therefore not a physical bit count: a key-head vector contains
$d_h\bar b=2B$ code bits before norms and metadata.

Unless stated otherwise, Block-GTQ allows integer key widths
$b_i\in\{1,\ldots,8\}$, applies the TQ-MSE rotation to every nonempty
equal-rate key group, and encodes each value head with TQ-MSE at one uniform
rate. Each group rotation is generated once from a fixed codec seed and reused
across tokens. Scalar codebooks use Lloyd--Max centroids under the Gaussian
coordinate model, and each decoded group direction is renormalized before
restoring its stored input norm. Block-GTQ and the uniform TQ-MSE baseline share
these implementation choices; Appendix~\ref{app:proxy-derivation} distinguishes
them from the reference codec's theoretical guarantee.

\subsection{NIAH protocol and context map}
\label{app:niah-protocol}

Our six RULER-inspired NIAH variants~\citep{hsieh2024ruler} insert synthetic
key--value needles into filler text.
\texttt{single} retrieves one value; \texttt{distract} adds same-format
distractors; \texttt{multi} queries one target among several needles; and
\texttt{m-key}, \texttt{m-value}, and \texttt{m-query} require multiple
responses. The first three tasks use binary scoring. The last three score the
fraction of expected responses recovered.

Each $(\text{task},\text{context length},\text{depth})$ cell averages three
haystacks over context lengths 4K, 8K, 16K, 32K, 64K, and 128K and depths
$\{0\%,10\%,\ldots,100\%\}$. Methods and budgets share the same needles, so
comparisons are paired. Block-GTQ uses the first 2048 WikiText-2 test tokens as
one calibration stream. Table~\ref{tab:niah-summary-main} uses this
fractional-scoring protocol for every reported method and rate.

\subsection{Qwen2.5 context extension}
\label{app:qwen-yarn}

Qwen2.5-7B-Instruct has a native 32K context. We therefore use native RoPE at
4K--32K and static YaRN~\citep{peng2024yarn} with factor 4 and original
maximum position 32,768 at 64K/128K. The same length-dependent rule is
applied to fp16 and every quantized method. This keeps the positional
configuration identical across methods at each length and avoids applying
YaRN within the native context range.
Table~\ref{tab:niah-qwen-length} reports the resulting per-length record.
The supplemental Qwen2.5-32B-Instruct sweep follows the same native/YaRN
length rule; Table~\ref{tab:niah-qwen32b} reports its Stress results.
Calibration scores use raw pre-RoPE Q/K activations, before rotary or amplitude
scaling. Static YaRN applies a shared fixed amplitude to Q and K; including it
in the score would multiply all blocks within a head by the same positive
factor and leave the allocation unchanged
(Appendix~\ref{app:coordinate-change}).

\subsection{LongBench protocol}
\label{app:longbench-protocol}

The controlled panel covers all 16 English tasks: Qasper, MultiFieldQA-en,
HotpotQA, 2WikiMultihopQA, GovReport, TREC, PassageRetrieval-en, LCC,
QMSum, MultiNews, TriviaQA, SAMSum, RepoBench-P, NarrativeQA, MuSiQue,
and PassageCount. Inputs are middle-truncated to at most 31,500 tokens and
decoded greedily with official task-specific output caps. Each task uses its
official metric; \texttt{Avg16} is their unweighted macro-average. The
Qwen2.5-7B comparison in Table~\ref{tab:longbench-qwen7b} uses the same
16-task metric.
For the controlled LongBench panel, Block-GTQ uses 2048 WikiText-2
calibration tokens with Q statistics pooled over the query heads associated
with each KV head.
Table~\ref{tab:long-context-summary} reports the main aggregate comparison;
Tables~\ref{tab:longbench-8b-full}, \ref{tab:longbench-70b-full},
and~\ref{tab:longbench-qwen32b-full} give
the controlled 8B record and the 70B/32B scaling records;
Table~\ref{tab:longbench-qwen7b} reports the Qwen2.5-7B no-buffer sweep.

\subsection{Packed deployment protocol}
\label{app:deployment-protocol}

The deployment experiments use Qwen2.5-3B-Instruct on one H800 80GB GPU
at batch size 1. Block-GTQ uses the first 64 WikiText-2
train tokens for allocation, with pooled Q statistics and \texttt{rt=1}.
The final K3V3 comparison uses whole-step CUDA graphs for both packed
methods and the FP16 baseline, which uses native-GQA FlashAttention-2. Each sweep cell
reports the median of 20 timed steps after five warmups; the separate
optimization ladder uses three interleaved repetitions.
Perplexity scores $[T-1000,T)$ through 128K. Resident KV footprint,
post-prefill peak memory, and prefill peak memory are reported separately;
Appendix~\ref{app:deployment-measurement} specifies the peak-reset boundaries.
Appendix~\ref{app:deployment-details} gives complete results and protocol details.

\FloatBarrier

\section{Extended Allocation and Attention Diagnostics}
\label{app:attention-details}

\subsection{Diagnostic scope and protocol}
\label{app:diagnostic-protocol}

The panel comprises eight full-RoPE configurations, one partial-RoPE
configuration (GLM-4-9B), and one MLA rotary-branch configuration
(DeepSeek-V2-Lite). For GLM, every method quantizes the full 128-dimensional
key; only its first 64 dimensions are rotated. For DeepSeek-V2-Lite, the
diagnostic quantizes only the shared 64-dimensional RoPE key, excluding the
MLA latent and its reconstructed nonrotary (NoPE) key and value state. The joint
MLA latent interface is described in Appendix~\ref{app:architecture-adapters}.
That interface uses DeepSeek-V2-Lite-Chat, a different checkpoint from this panel's
DeepSeek-V2-Lite (DS-V2-Lite).
The label Llama-3.1-8B in this panel denotes the
\texttt{meta-llama/Llama-3.1-8B-Instruct} checkpoint.
DS-R1-Llama-8B and DS-R1-Qwen-7B abbreviate DeepSeek-R1-Distill-Llama-8B and
DeepSeek-R1-Distill-Qwen-7B, respectively.

\begin{table}[H]
  \centering
  \caption{Architectures used in the attention-interface diagnostic. Geometry
  lists layers, query/KV heads, and the RoPE dimension. GLM has a 128-dimensional
  key, whereas the MLA row covers only its shared RoPE key. Qwen3 activations
  are collected after QK normalization.}
  \label{tab:diagnostic-model-panel}
  \scriptsize
  \setlength{\tabcolsep}{4pt}
  \begin{tabular}{lll}
    \toprule
    Model & Geometry & Diagnostic path \\
    \midrule
    Qwen2.5-3B & $36$L, $16/2$, $d_{\rm rope}=128$ & full-RoPE GQA \\
    Qwen2.5-14B & $48$L, $40/8$, $d_{\rm rope}=128$ & full-RoPE GQA \\
    Qwen3-8B & $36$L, $32/8$, $d_{\rm rope}=128$ & full-RoPE GQA, post-QK norm \\
    Qwen3-30B-A3B & $48$L, $32/4$, $d_{\rm rope}=128$ & MoE GQA, post-QK norm \\
    Llama-3.1-8B & $32$L, $32/8$, $d_{\rm rope}=128$ & full-RoPE GQA \\
    DS-R1-Llama-8B & $32$L, $32/8$, $d_{\rm rope}=128$ & full-RoPE GQA \\
    DS-R1-Qwen-7B & $28$L, $28/4$, $d_{\rm rope}=128$ & full-RoPE GQA \\
    Mistral-Nemo-12B & $40$L, $32/8$, $d_{\rm rope}=128$ & full-RoPE GQA \\
    GLM-4-9B & $40$L, $32/2$, $d_{\rm rope}=64$ & fused QKV, 128-d partial-RoPE key \\
    DS-V2-Lite & $27$L, $16/1$, $d_{\rm rope}=64$ & shared MLA RoPE key \\
    \bottomrule
  \end{tabular}
\end{table}

Block-GTQ is calibrated on the first 2048 WikiText-2 train tokens and evaluated
on the first 2048 test tokens. For each layer, KV head, served query head, and a
grid of 50 relative offsets in $[-1024,1024]$, we compute the absolute
difference between clean and reconstructed query--key logits. Only keys are
quantized; values remain fp16. For GLM, the logit includes both the rotated
and static key terms. Its fused QKV projection is separated into Q and K
before head-wise evaluation.

For attention distributions, each held-out query position 1025--2048 attends
to its causal prefix. With clean and quantized logit rows $s,\hat s$ and
$p=\softmax(s/\sqrt{d_h})$, $\hat p=\softmax(\hat s/\sqrt{d_h})$, we report
$\mathrm{KL}(p\|\hat p)$ and top-10 overlap
$100\,|\operatorname{top}_{10}(s)\cap\operatorname{top}_{10}(\hat s)|/10$
(\%), where $\operatorname{top}_{10}$ returns the set of ten highest-logit
token indices. Each metric is averaged equally over the 1024 query positions and all
query heads served by each KV head, then equally over all layer--KV-head
pairs. Panel means weight the ten configurations equally. Here $d_h$ is the
dimension of the diagnostic key path. For GLM-4-9B, let $t$ and $j$ denote
the query and key positions, and let subscripts $\mathrm R$ and $\mathrm N$
denote their pre-RoPE rotary and nonrotary subvectors. Then
$s_{t,j}=q_{{\rm R},t}^{\top}R_{j-t}k_{{\rm R},j}
+q_{{\rm N},t}^{\top}k_{{\rm N},j}$, with both key terms reconstructed in
$\hat s$, $d_h=128$, and $R_{j-t}$ acting on the 64-dimensional rotary
subspace. The offset follows the key-minus-query convention of
Equation~\ref{eq:rope-decomposition}. For DeepSeek-V2-Lite, $s_{t,j}$ contains only the
64-dimensional RoPE term; the MLA NoPE contribution is omitted from this
branch-local diagnostic, and $d_h=64$.

\paragraph{KIVI-ScaleOnly.}\phantomsection\label{para:kivi-scaleonly-note}
The diagnostic removes KIVI's fp16 residual buffer while retaining its
per-channel scale quantizer, so every attended key is quantized. We call this
setting \emph{KIVI-ScaleOnly} and
reserve \emph{KIVI} for its protected native path.

\subsection{RoPE-logit fidelity}

\begin{table}[H]
  \centering
  \caption{\textbf{RoPE-logit fidelity at 2 logical K bits.} Layer-averaged
  MAE under K-only quantization; lower is better. $\Delta$ is the relative
  reduction over uniform TQ-MSE, and Wins counts layers improved by Block-GTQ.
  GLM quantizes its full partial-RoPE key and includes the static NoPE logit
  term; DS-V2-Lite measures only its RoPE branch.}
  \label{tab:logit-mae-2b}
  \scriptsize
  \setlength{\tabcolsep}{4pt}
  \begin{tabular}{@{}lrrrc@{\qquad}lrrrc@{}}
    \toprule
    Model & TQ & Block & $\Delta$ & Wins & Model & TQ & Block & $\Delta$ & Wins \\
    \midrule
    Qwen2.5-3B & 13.25 & \textbf{6.30} & 52.5\% & 36/36 &
    DS-R1-Llama-8B & 6.77 & \textbf{4.54} & 33.0\% & 32/32 \\
    Qwen2.5-14B & 8.20 & \textbf{5.15} & 37.2\% & 48/48 &
    DS-R1-Qwen-7B & 25.20 & \textbf{5.01} & 80.1\% & 28/28 \\
    Qwen3-8B & 11.46 & \textbf{5.88} & 48.7\% & 36/36 &
    Mistral-Nemo-12B & 6.75 & \textbf{4.42} & 34.5\% & 40/40 \\
    Qwen3-30B-A3B & 12.97 & \textbf{5.89} & 54.6\% & 48/48 &
    GLM-4-9B & 16.31 & \textbf{9.90} & 39.3\% & 40/40 \\
    Llama-3.1-8B & 7.49 & \textbf{5.02} & 33.0\% & 32/32 &
    DS-V2-Lite & 11.31 & \textbf{7.25} & 35.9\% & 27/27 \\
    \bottomrule
  \end{tabular}
\end{table}

\subsection{Per-configuration attention fidelity}

Tables~\ref{tab:attention-kl-details} and~\ref{tab:attention-top10-details}
give the per-configuration values summarized in
Figure~\ref{fig:attention-diagnostic}, including its KIVI-ScaleOnly baseline.

\begin{table}[H]
  \centering
  \caption{Per-configuration softmax KL to fp16 (lower is better) under K-only,
  no-buffer quantization. Columns are grouped by logical K rate.
  Bold marks the best method in each configuration--rate cell.}
  \label{tab:attention-kl-details}
  \scriptsize
  \setlength{\tabcolsep}{3pt}
  \begin{tabular}{lrrrcrrrcrrr}
    \toprule
    & \multicolumn{3}{c}{2 b/d} & & \multicolumn{3}{c}{3 b/d} &
    & \multicolumn{3}{c}{4 b/d} \\
    \cmidrule(lr){2-4}\cmidrule(lr){6-8}\cmidrule(lr){10-12}
    Model & Block-GTQ & \shortstack{KIVI-\\ScaleOnly} & TQ-MSE & & Block-GTQ & \shortstack{KIVI-\\ScaleOnly} & TQ-MSE &
    & Block-GTQ & \shortstack{KIVI-\\ScaleOnly} & TQ-MSE \\
    \midrule
    Qwen2.5-3B       & \textbf{0.1636} & 0.3114 & 0.6359 & & \textbf{0.0444} & 0.0989 & 0.2544 & & \textbf{0.0121} & 0.0680 & 0.0859 \\
    Qwen2.5-14B      & \textbf{0.0773} & 0.1596 & 0.2118 & & \textbf{0.0214} & 0.0558 & 0.0567 & & \textbf{0.0056} & 0.0405 & 0.0153 \\
    Qwen3-8B         & \textbf{0.1210} & 0.5126 & 0.6349 & & \textbf{0.0327} & 0.3594 & 0.1960 & & \textbf{0.0091} & 0.3338 & 0.0594 \\
    Qwen3-30B-A3B    & \textbf{0.1234} & 0.2020 & 0.8097 & & \textbf{0.0335} & 0.0696 & 0.2288 & & \textbf{0.0088} & 0.0490 & 0.0571 \\
    Llama-3.1-8B     & \textbf{0.0569} & 0.1298 & 0.1444 & & \textbf{0.0151} & 0.0614 & 0.0367 & & \textbf{0.0041} & 0.0510 & 0.0098 \\
    DS-R1-Llama-8B   & \textbf{0.0378} & 0.1256 & 0.1024 & & \textbf{0.0099} & 0.0663 & 0.0245 & & \textbf{0.0026} & 0.0584 & 0.0063 \\
    DS-R1-Qwen-7B    & \textbf{0.1229} & 0.7481 & 0.8782 & & \textbf{0.0271} & 0.5569 & 0.4840 & & \textbf{0.0104} & 0.5326 & 0.2915 \\
    Mistral-Nemo-12B & \textbf{0.0515} & 1.9675 & 0.1288 & & \textbf{0.0140} & 1.8197 & 0.0332 & & \textbf{0.0035} & 1.7932 & 0.0088 \\
    GLM-4-9B         & 0.4546 & \textbf{0.3492} & 0.9397 & & 0.1617 & \textbf{0.1188} & 0.3450 & & \textbf{0.0546} & 0.0829 & 0.1072 \\
    DS-V2-Lite       & 0.5709 & \textbf{0.2023} & 1.3615 & & 0.1799 & \textbf{0.0684} & 0.4381 & & 0.0548 & \textbf{0.0435} & 0.1274 \\
    \midrule
    \textit{Mean}    & \textbf{0.1780} & 0.4708 & 0.5847 & & \textbf{0.0540} & 0.3275 & 0.2097 & & \textbf{0.0166} & 0.3053 & 0.0769 \\
    \bottomrule
  \end{tabular}
\end{table}

\begin{table}[H]
  \centering
  \caption{Per-configuration top-10 attended-token overlap with fp16 (\%; higher
  is better) under the same protocol as Table~\ref{tab:attention-kl-details}.
  Bold marks the best method in each configuration--rate cell.}
  \label{tab:attention-top10-details}
  \scriptsize
  \setlength{\tabcolsep}{3pt}
  \begin{tabular}{lrrrcrrrcrrr}
    \toprule
    & \multicolumn{3}{c}{2 b/d} & & \multicolumn{3}{c}{3 b/d} &
    & \multicolumn{3}{c}{4 b/d} \\
    \cmidrule(lr){2-4}\cmidrule(lr){6-8}\cmidrule(lr){10-12}
    Model & Block-GTQ & \shortstack{KIVI-\\ScaleOnly} & TQ-MSE & & Block-GTQ & \shortstack{KIVI-\\ScaleOnly} & TQ-MSE &
    & Block-GTQ & \shortstack{KIVI-\\ScaleOnly} & TQ-MSE \\
    \midrule
    Qwen2.5-3B       & \textbf{75.2} & 71.6 & 62.1 & & \textbf{86.1} & 82.1 & 76.8 & & \textbf{92.5} & 85.7 & 86.4 \\
    Qwen2.5-14B      & \textbf{75.7} & 72.3 & 64.8 & & \textbf{86.4} & 82.4 & 79.3 & & \textbf{92.6} & 85.6 & 88.3 \\
    Qwen3-8B         & \textbf{76.3} & 69.4 & 61.6 & & \textbf{86.7} & 78.5 & 77.2 & & \textbf{92.6} & 81.1 & 86.8 \\
    Qwen3-30B-A3B    & \textbf{74.4} & 72.6 & 55.2 & & \textbf{85.5} & 82.5 & 72.5 & & \textbf{92.0} & 85.5 & 84.1 \\
    Llama-3.1-8B     & 77.1 & \textbf{77.2} & 68.3 & & \textbf{87.2} & 85.0 & 81.4 & & \textbf{92.9} & 87.3 & 89.5 \\
    DS-R1-Llama-8B   & \textbf{77.1} & 76.3 & 68.2 & & \textbf{87.2} & 84.6 & 81.5 & & \textbf{93.0} & 87.0 & 89.6 \\
    DS-R1-Qwen-7B    & \textbf{81.4} & 72.9 & 66.5 & & \textbf{89.7} & 82.3 & 79.4 & & \textbf{94.2} & 85.1 & 87.2 \\
    Mistral-Nemo-12B & \textbf{80.5} & 69.5 & 72.5 & & \textbf{89.1} & 76.9 & 84.2 & & \textbf{94.2} & 78.8 & 91.2 \\
    GLM-4-9B         & 60.2 & \textbf{61.7} & 47.0 & & \textbf{76.1} & 75.8 & 66.0 & & \textbf{86.0} & 80.8 & 79.6 \\
    DS-V2-Lite       & 49.9 & \textbf{68.0} & 33.9 & & 68.9 & \textbf{81.2} & 55.4 & & 81.6 & \textbf{85.9} & 73.3 \\
    \midrule
    \textit{Mean}    & \textbf{72.8} & 71.2 & 60.0 & & \textbf{84.3} & 81.1 & 75.4 & & \textbf{91.2} & 84.3 & 85.6 \\
    \bottomrule
  \end{tabular}
\end{table}

Averaged across these heterogeneous diagnostic paths, uniform TQ-MSE incurs
$3.28\times$, $3.88\times$, and $4.63\times$ the softmax KL of Block-GTQ at
2, 3, and 4 bits, respectively. Block-GTQ also raises mean top-10 overlap by
12.8, 8.9, and 5.6 percentage points. It is best on both metrics in 7/10,
8/10, and 9/10 configurations at the three rates. The remaining exceptions
occur on the two architecture-specific paths and, at 2 bits, the 0.1-point top-10
difference on Llama-3.1-8B. The panel mean is descriptive: the MLA row measures
rotary-branch rather than full-model attention fidelity.

\subsection{A 4K NIAH failure case study on Qwen}
\label{app:qwen-tqmse-collapse}

We revisit uniform TQ-MSE's zero-score result on Qwen2.5-7B-Instruct
by separating K/V quantization and testing a fixed-group, uniform-rate
control. All six paths evaluate the same
198 NIAH samples (six tasks, 11 needle depths, and three trials per cell)
at 4K with native RoPE and no fp16 buffer. Prompt hashes match in every
task--depth--trial cell. All paths share the scoring rule, greedy decoding,
and a 30-token generation cap; calibrated paths use the same 2048-token
WikiText-2 cache.

\begin{table}[H]
  \centering
  \caption{Paired 4K NIAH failure-case study on Qwen2.5-7B-Instruct (\%).
  Quantized branches use an average logical rate of 3 bits per coordinate;
  unquantized branches remain fp16. Avg is the unweighted six-task mean,
  with fractional credit within multi-answer tasks.}
  \label{tab:qwen-collapse-interventions}
  \scriptsize
  \setlength{\tabcolsep}{3pt}
  \begin{tabular}{lrrrrrrr}
    \toprule
    Method & single & distr. & multi & m-key & m-val & m-qry & Avg \\
    \midrule
    fp16 & 100.00 & 100.00 & 100.00 & 33.33 & 66.67 & 33.33 & 72.22 \\
    \midrule
    TQ-MSE, K only & 0 & 0 & 0 & 0 & 0 & 0 & 0 \\
    TQ-MSE, V only & 100.00 & 100.00 & 100.00 & 33.33 & 66.67 & 33.33 & 72.22 \\
    TQ-MSE, K+V & 0 & 0 & 0 & 0 & 0 & 0 & 0 \\
    \midrule
    Fixed-group uniform K3 + TQ-MSE V3 & 0 & 0 & 0 & 1.01 & 1.01 & 0 & 0.34 \\
    Block-GTQ K3 + TQ-MSE V3 & 100.00 & 100.00 & 100.00 & 33.33 & 66.67 & 33.33 & 72.22 \\
    \bottomrule
  \end{tabular}
\end{table}

\paragraph{Uniform K3 quantization is sufficient for failure.}
V-only TQ-MSE matches all six fp16 task scores, averaging 72.22\%, whereas
K-only and K+V TQ-MSE score zero on every task
(Table~\ref{tab:qwen-collapse-interventions}). The two zero-scoring paths
produce nonempty outputs with repetitive, malformed fragments, rather than
merely correct answers rejected by extraction. Thus uniform K3 quantization
alone is sufficient to produce the observed failure under this protocol;
V quantization is not required.

\paragraph{Preserving the groups does not recover performance.}
The fixed-group control retains Block-GTQ's calibrated K3 partition,
coordinate permutation, rotations, and separate group normalization, but
uses a 3-bit codebook for every group without merging groups. With the same
uniform V3 codec, it scores only 0.34\%, whereas the original mixed-width
K schedule recovers the fp16 task scores at 72.22\%. Hence preserving these
groups alone is insufficient: the assigned widths matter \emph{given this
partition}. The comparison fixes the average logical K rate; the multiple
group norms and metadata need not match whole-key TQ-MSE's physical storage.

\paragraph{Interpretation and scope.}
An independent K-only diagnostic using the protocol of
Appendix~\ref{app:diagnostic-protocol} gives Qwen's 3-bit softmax KL to fp16
as 0.7189 for uniform TQ-MSE and 0.067 for Block-GTQ. For the block scores in
Equation~\ref{eq:block-score}, let $p_i=s_i/\sum_j s_j$. Computing
$N_{\rm eff}=1/\sum_i p_i^2$ within each layer--KV-head pair and then averaging
gives 7.2 effective blocks for Qwen2.5-7B-Instruct versus 14.8 for
Llama-3.1-8B-Instruct. These observations motivate an explanation based on
uneven block sensitivity, but neither diagnostic nor intervention compares
score-guided placement with an alternative mixed-width schedule. The
intervention conclusions are limited to this model, budget, and 4K generation
protocol, including its output cap; they do not establish the cause of
the zero AIME scores on DS-R1-Qwen-7B.

\FloatBarrier

\section{Calibration and Score-Function Ablations}
\label{app:calibration}

The NIAH ablations in this section use partial-credit scoring:
\texttt{single}, \texttt{distract}, and
\texttt{multi} are binary, whereas \texttt{m-key}, \texttt{m-value}, and
\texttt{m-query} score the fraction of expected answers recovered. Overall is
the unweighted mean of the six task scores. For calibration and score-function
ablations, Block-GTQ uses the first 2048 tokens of the WikiText-2 test split
unless calibration data are varied. The LoRA comparison in
Section~\ref{app:lora-recalibration} uses the controlled NIAH harness with
separate calibration caches for the base and adapted models. Each experiment
covers six context lengths, eleven needle depths, and three trials per cell.

\subsection{Calibration length and corpus}
\label{app:calibration-data}

The length sweep takes one contiguous WikiText-2 test prefix of the specified
length. The corpus sweep takes 2048-token streams from WikiText-2 test, PG19
train, C4 validation, or CPython source code. PG19 and C4 use the first records
in dataset order with at least 200 and 50 characters, respectively, joined by
blank lines; the code stream starts with CPython v3.11.0
\texttt{Lib/argparse.py}. All evaluation instances and
quantizer settings are fixed. Table~\ref{tab:calibration-main} gives the
aggregate comparison; Tables~\ref{tab:calibration-length-detailed}
and~\ref{tab:calibration-corpus-detailed} give the per-task scores.

\begin{table}[H]
  \centering
  \caption{Calibration-length sensitivity on Llama-3.1-8B-Instruct. Values
  are per-task NIAH scores (\%) under partial-credit scoring;
  $N_{\rm cal}$ is the number of calibration tokens, and $\Delta_{2048}$ is
  Overall minus the 2048-token baseline, in percentage points.}
  \label{tab:calibration-length-detailed}
  \scriptsize
  \setlength{\tabcolsep}{3.5pt}
  \begin{tabular*}{\textwidth}{@{\extracolsep{\fill}}llrrrrrrrr@{}}
    \toprule
    Rate & $N_{\rm cal}$ & single & distract & multi & m-key & m-value & m-query & Overall & $\Delta_{2048}$ \\
    \midrule
    K2V2 & 64   & 100.00 & 96.46 & 97.98 & 95.29 & 98.48 & 85.86 & 95.68 & $-1.68$ \\
    K2V2 & 128  & 100.00 & 98.48 & 97.47 & 95.79 & 96.97 & 76.94 & 94.28 & $-3.08$ \\
    K2V2 & 256  & 100.00 & 96.46 & 91.92 & 95.96 & 99.16 & 80.98 & 94.08 & $-3.28$ \\
    K2V2 & 512  & 100.00 & 97.98 & 96.46 & 97.47 & 99.49 & 83.33 & 95.79 & $-1.57$ \\
    K2V2 & 1024 & 100.00 & 98.48 & 94.95 & 92.59 & 99.33 & 74.24 & 93.27 & $-4.09$ \\
    K2V2 & \textbf{2048} & 100.00 & 98.99 & 100.00 & 98.48 & 100.00 & 86.70 & \textbf{97.36} & (base) \\
    \midrule
    K3V3 & 64   & 100.00 & 100.00 & 100.00 & 97.98 & 100.00 & 92.76 & 98.46 & $+0.06$ \\
    K3V3 & 128  & 100.00 & 100.00 & 100.00 & 100.00 & 100.00 & 91.75 & 98.63 & $+0.23$ \\
    K3V3 & 256  & 100.00 & 100.00 & 100.00 & 96.46 & 100.00 & 87.54 & 97.33 & $-1.07$ \\
    K3V3 & 512  & 100.00 & 100.00 & 99.49 & 96.97 & 100.00 & 91.58 & 98.01 & $-0.39$ \\
    K3V3 & 1024 & 100.00 & 100.00 & 100.00 & 99.49 & 100.00 & 90.24 & 98.29 & $-0.11$ \\
    K3V3 & \textbf{2048} & 100.00 & 100.00 & 100.00 & 100.00 & 99.66 & 90.74 & \textbf{98.40} & (base) \\
    \bottomrule
  \end{tabular*}
\end{table}

The K2V2 length response is non-monotone, with most variation concentrated in
\texttt{m-query}; the three binary tasks remain comparatively stable. K3V3 is
less sensitive: every length lies within 1.07 Overall points of the baseline.

\begin{table}[H]
  \centering
  \caption{Calibration-corpus sensitivity on Llama-3.1-8B-Instruct. Each
  corpus contributes 2048 tokens. Values are per-task NIAH scores (\%) under
  partial-credit scoring; $\Delta$ is Overall minus the WikiText-2 baseline,
  in percentage points.}
  \label{tab:calibration-corpus-detailed}
  \scriptsize
  \setlength{\tabcolsep}{3.5pt}
  \begin{tabular*}{\textwidth}{@{\extracolsep{\fill}}llrrrrrrrr@{}}
    \toprule
    Rate & Corpus & single & distract & multi & m-key & m-value & m-query & Overall & $\Delta$ \\
    \midrule
    K2V2 & WikiText-2 & 100.00 & 98.99 & 100.00 & 98.48 & 100.00 & 86.70 & \textbf{97.36} & (base) \\
    K2V2 & PG19       & 99.49 & 97.47 & 100.00 & 98.65 & 99.83 & 87.04 & 97.08 & $-0.28$ \\
    K2V2 & C4         & 100.00 & 98.99 & 100.00 & 90.40 & 98.82 & 79.29 & 94.58 & $-2.78$ \\
    K2V2 & code       & 98.99 & 97.47 & 96.46 & 94.28 & 95.96 & 78.79 & 93.66 & $-3.70$ \\
    \midrule
    K3V3 & WikiText-2 & 100.00 & 100.00 & 100.00 & 100.00 & 99.66 & 90.74 & \textbf{98.40} & (base) \\
    K3V3 & PG19       & 100.00 & 98.99 & 100.00 & 98.32 & 100.00 & 91.41 & 98.12 & $-0.28$ \\
    K3V3 & C4         & 100.00 & 100.00 & 100.00 & 98.99 & 99.66 & 90.74 & 98.23 & $-0.17$ \\
    K3V3 & code       & 100.00 & 100.00 & 100.00 & 99.49 & 100.00 & 88.89 & 98.06 & $-0.34$ \\
    \bottomrule
  \end{tabular*}
\end{table}

PG19 remains close to WikiText-2 at both rates. C4 and code have larger effects
at K2V2, while all four sources remain within 0.34 Overall points at K3V3.

We additionally assess sensitivity to calibration-window selection using
WikiText-2 test perplexity. Three WikiText-2 train windows start at token
offsets 0, 10K, and 20K. Evaluation uses the first 100,000 test tokens with a
4096-token sliding context and 512-token stride. Table~\ref{tab:calib-ppl}
reports the mean and sample standard deviation across the three windows.

\begin{table}[H]
  \centering
  \caption{Calibration-length sensitivity under window variation. Cells are
  WikiText-2 test perplexity mean $\pm$ standard deviation over three
  calibration windows.}
  \label{tab:calib-ppl}
  \scriptsize
  \setlength{\tabcolsep}{4pt}
  \begin{tabular}{llrrrr}
    \toprule
    Model & Rate & fp16 & $N_{\rm cal}=128$ & $N_{\rm cal}=512$ & $N_{\rm cal}=2048$ \\
    \midrule
    Llama-3.1-8B & K3V3 & 6.4095 & $6.6414\pm0.0192$ & $6.6462\pm0.0137$ & $6.6256\pm0.0130$ \\
    Llama-3.1-8B & K2V2 & 6.4095 & $7.9949\pm0.0235$ & $7.8474\pm0.0669$ & $7.8459\pm0.0784$ \\
    DS-R1-Qwen-7B & K3V3 & 18.4800 & $19.0611\pm0.0229$ & $19.0303\pm0.0294$ & $19.0576\pm0.0328$ \\
    DS-R1-Qwen-7B & K2V2 & 18.4800 & $23.2010\pm0.2150$ & $23.1319\pm0.1781$ & $23.0211\pm0.1132$ \\
    \bottomrule
  \end{tabular}
\end{table}

\subsection{Block-score ablation}
\label{app:energy-score-ablation}

We compare eight block scores in three groups: five marginal-moment scores,
two paired-activation scores, and one weight-only score.

\paragraph{Marginal moments (five).}
For readability, we suppress the fixed layer, head, and calibration-position
indices. Expectations involving $\bq^{(i)}$ average over
$t\in\mathcal C$ and $g\in G(h)$; those involving $\bk^{(i)}$ average over
$t\in\mathcal C$. Using the names from Table~\ref{tab:score-ablation},
\begin{align*}
  \text{Mean Q/K energy (default)}\quad
    &=\tfrac12\!\left(\E\|\bq^{(i)}\|_2^2+\E\|\bk^{(i)}\|_2^2\right),\\
  \text{Maximum}\quad
    &=\max\!\left\{\E\|\bq^{(i)}\|_2^2,\E\|\bk^{(i)}\|_2^2\right\},\\
  \text{Geometric mean}\quad
    &=\sqrt{\E\|\bq^{(i)}\|_2^2\,\E\|\bk^{(i)}\|_2^2},\\
  \text{Key only}\quad &=\E\|\bk^{(i)}\|_2^2,\\
  \text{Query only}\quad &=\E\|\bq^{(i)}\|_2^2.
\end{align*}
Their implementation identifiers are, respectively, \texttt{qk\_avg},
\texttt{qk\_max}, \texttt{qk\_product}, \texttt{k\_only}, and
\texttt{q\_only}.
Among the five marginal-moment variants, the first three symmetrically combine
the two moments using their arithmetic
mean, maximum, or geometric mean, respectively. The two one-sided variants
test whether either marginal alone provides sufficient allocation signal.

\paragraph{Paired activations (two).}
The bound-motivated \emph{Paired norm product} and \emph{Direct RoPE logit}
scores are
\begin{align*}
  s_{\ell,h,i}^\star
  &=\E_{t\in\mathcal C}\!\left[
    \frac{1}{|G(h)|}\sum_{g\in G(h)}
    \|\bq_{\ell,g,t}^{(i)}\|_2\|\bk_{\ell,h,t}^{(i)}\|_2
  \right],\\
  s_{\ell,h,i}^{\mathrm{direct}}
  &=\E_{\substack{t\in\mathcal C,\,g\in G(h)\\\Delta\in\mathcal D}}
  \left|\bq_{\ell,g,t}^{(i)\top}R(\Delta\theta_i)
  \bk_{\ell,h,t}^{(i)}\right|.
\end{align*}
Here $\mathcal D$ contains 50 evenly spaced offsets in $[-1024,1024]$.
Both scores preserve token-aligned query--key pairs; the first averages their
norm product, whereas the second directly measures RoPE-logit magnitude. Their
implementation identifiers are \texttt{s\_star} and \texttt{s\_direct}.

\paragraph{Weight only (one).}
The data-independent \emph{Weight only} score is
\begin{equation}
  s_{\ell,h,i}^{\mathrm{weight}}
  =\frac{1}{|G(h)|}\sum_{g\in G(h)}
  \left\|W_{q,\ell,g}^{(i)\top}W_{k,\ell,h}^{(i)}\right\|_{\mathrm F},
  \label{eq:weight-score-ablation}
\end{equation}
Here $W_{q,\ell,g}^{(i)}$ and $W_{k,\ell,h}^{(i)}$ are the two-row projection-weight
blocks producing RoPE frequency pair $i$. The score corresponds to relative
offset $\Delta_0=0$. It uses model weights rather
than calibration activations; its implementation identifier is
\texttt{s\_weight}.

All variants use the same K2V2 budget, allocator, codec, and
Llama-3.1-8B-Instruct NIAH instances; activation-based variants use the same
2048-token WikiText-2 stream. Table~\ref{tab:score-ablation-detailed} expands
the Overall results in Table~\ref{tab:score-ablation}.

\begin{table}[H]
  \centering
  \caption{Block-score ablation on Llama-3.1-8B-Instruct at K2V2. Values
  are per-task NIAH scores (\%) under partial-credit scoring; per-column best
  results are bold.}
  \label{tab:score-ablation-detailed}
  \small
  \setlength{\tabcolsep}{4.5pt}
  \begin{tabular}{lrrrrrrr}
    \toprule
    Score & single & distract & multi & m-key & m-value & m-query & Overall \\
    \midrule
    Mean Q/K energy (default)
      & \textbf{100.00} & 98.99 & \textbf{100.00} & \textbf{98.48}
      & \textbf{100.00} & 86.70 & \textbf{97.36} \\
    Maximum
      & 98.99 & 98.48 & 98.99 & 98.15
      & 98.32 & 87.71 & 96.77 \\
    Key only
      & 99.49 & 97.47 & 99.49 & 95.79
      & 99.49 & 79.97 & 95.29 \\
    Query only
      & \textbf{100.00} & \textbf{99.49} & 98.99 & 94.95
      & 97.31 & 73.57 & 94.05 \\
    Geometric mean
      & \textbf{100.00} & 97.98 & 99.49 & 93.10
      & 98.48 & 81.65 & 95.12 \\
    Paired norm product
      & 98.99 & 95.45 & 98.99 & 85.52
      & 99.33 & 60.77 & 89.84 \\
    Direct RoPE logit
      & 99.49 & 98.99 & 98.99 & 97.98
      & 98.82 & \textbf{89.06} & 97.22 \\
    Weight only
      & 96.97 & 85.35 & 96.46 & 76.43
      & 87.21 & 68.18 & 85.10 \\
    \bottomrule
  \end{tabular}
\end{table}

The default mean Q/K energy score has the highest observed Overall score, with Direct RoPE
logit 0.14 points behind and leading on \texttt{m-query}. Paired norm product
and Weight only are weaker in this experiment. The Weight only
score also produces a near-uniform K2V2 schedule: 92.2\% of block assignments
remain at 2 bits, and 156 of 256 layer--KV-head schedules assign 2 bits to all
64 blocks. This comparison supports the mean Q/K energy default for the tested
model and rate without attributing the gap to estimator variance.

\subsection{LoRA adaptation}
\label{app:lora-recalibration}

We test whether Block-GTQ retains its KV-quantization advantage over uniform
TQ-MSE on a LoRA-fine-tuned model, comparing both methods and fp16 before
and after adaptation. We merge a public GSM8K-GRPO LoRA adapter into
Llama-3.1-8B-Instruct. The adapter has rank 64, $\alpha=128$, dropout 0.05,
\texttt{use\_rslora=true}, and targets the Q/K/V/O and gate/up/down
projections.

Both model states use the controlled NIAH prompts, 30-token
generation cap, partial-credit scoring, and K2V2 Stress setting
(Appendix~\ref{app:niah-protocol}). We compute a separate Block-GTQ schedule
for each state from its own 2048-token WikiText-2 Q/K activation cache.

\begin{table}[H]
  \centering
  \caption{Six-task NIAH before and after LoRA adaptation on
  Llama-3.1-8B-Instruct (K2V2, Stress; \%). \texttt{Overall} averages
  4K--128K; each length averages eleven depths and three trials per cell.}
  \label{tab:lora-gap-preservation}
  \small
  \renewcommand{\arraystretch}{0.94}
  \setlength{\tabcolsep}{3pt}
  \begin{tabular*}{\linewidth}{@{\extracolsep{\fill}}llrrrrrrr@{}}
    \toprule
    Model state & Method & Overall & 4K & 8K & 16K & 32K & 64K & 128K \\
    \midrule
    \multirow{3}{*}{Base}
      & fp16      & 99.55 & 100.00 & 99.83 & 99.66 & 99.83 & 100.00 & 97.98 \\
      & TQ-MSE    & 70.62 & 86.36 & 84.18 & 80.47 & 77.10 & 65.66 & 29.97 \\
      & Block-GTQ & 97.36 & 98.15 & 98.99 & 98.65 & 98.48 & 96.46 & 93.43 \\
    \midrule
    \multirow{3}{*}{LoRA-adapted}
      & fp16      & 97.45 & 97.47 & 96.30 & 96.63 & 98.32 & 99.33 & 96.63 \\
      & TQ-MSE    & 67.28 & 85.19 & 79.46 & 77.27 & 72.56 & 60.61 & 28.62 \\
      & Block-GTQ & 93.57 & 94.61 & 95.45 & 94.44 & 95.12 & 93.77 & 88.05 \\
    \bottomrule
  \end{tabular*}
\end{table}

The base rows reproduce the controlled Llama K2V2 scores in
Table~\ref{tab:niah-summary-main}. After adaptation, Block-GTQ scores
93.57\% versus 67.28\% for uniform TQ-MSE, retaining a 26.29-point advantage
(26.74 points on the base model) and outperforming TQ-MSE at every evaluated
context length. Thus, Block-GTQ's relative advantage persists on the tested
LoRA-adapted model when its allocation is calibrated on the adapted model's
activations.

\FloatBarrier

\section{Extended Long-Context Results}
\label{app:long-context-details}

Throughout this appendix, $T$ denotes the number of cached tokens.

\subsection{NIAH task breakdown}

Table~\ref{tab:niah-llama-full} gives the full six-task record behind the
Llama-3.1-8B averages in Table~\ref{tab:niah-summary-main}. Every method and rate
uses the paired protocol in Appendix~\ref{app:niah-protocol}.

\begin{table}[H]
  \centering
  \caption{Multi-task NIAH scores (\%) on Llama-3.1-8B-Instruct. Each entry
  averages context lengths 4K--128K, needle depths 0--100\%, and three paired
  trials per cell. KIVI-ScaleOnly is the no-buffer diagnostic defined in
  Appendix~\ref{para:kivi-scaleonly-note}.}
  \label{tab:niah-llama-full}
  \scriptsize
  \setlength{\tabcolsep}{4.0pt}
  \begin{tabular}{llrrrrrrr}
    \toprule
    Rate & Method & single & distract & multi & m-key & m-value & m-query & \textbf{Avg} \\
    \midrule
    fp16 & fp16 & 100.00 & 100.00 & 100.00 & 99.49 & 100.00 & 97.81 & 99.55 \\
    \midrule
    K3V3 & KIVI-ScaleOnly & 38.38 & 49.49 & 36.36 & 33.33 & 26.60 & 28.28 & 35.41 \\
     & TQ-MSE & \textbf{100.00} & 99.49 & \textbf{100.00} & \textbf{100.00} & 99.49 & 87.37 & 97.73 \\
     & Block-GTQ & \textbf{100.00} & \textbf{100.00} & \textbf{100.00} & \textbf{100.00} & \textbf{99.66} & \textbf{90.74} & \textbf{98.40} \\
    \midrule
    K3V2 & KIVI-ScaleOnly & 39.39 & 50.00 & 34.34 & 30.98 & 24.07 & 27.61 & 34.40 \\
     & TQ-MSE & \textbf{100.00} & 98.99 & \textbf{100.00} & 96.13 & 96.97 & \textbf{81.99} & 95.68 \\
     & Block-GTQ & \textbf{100.00} & \textbf{100.00} & \textbf{100.00} & \textbf{100.00} & \textbf{99.66} & 81.14 & \textbf{96.80} \\
    \midrule
    K2V2 & KIVI-ScaleOnly & 30.81 & 32.83 & 23.23 & 20.88 & 16.16 & 21.55 & 24.24 \\
     & TQ-MSE & 93.94 & 70.71 & 80.81 & 47.64 & 71.89 & 58.75 & 70.62 \\
     & Block-GTQ & \textbf{100.00} & \textbf{98.99} & \textbf{100.00} & \textbf{98.48} & \textbf{100.00} & \textbf{86.70} & \textbf{97.36} \\
    \bottomrule
  \end{tabular}
\end{table}

\begin{figure}[H]
  \centering
  \includegraphics[width=\textwidth]{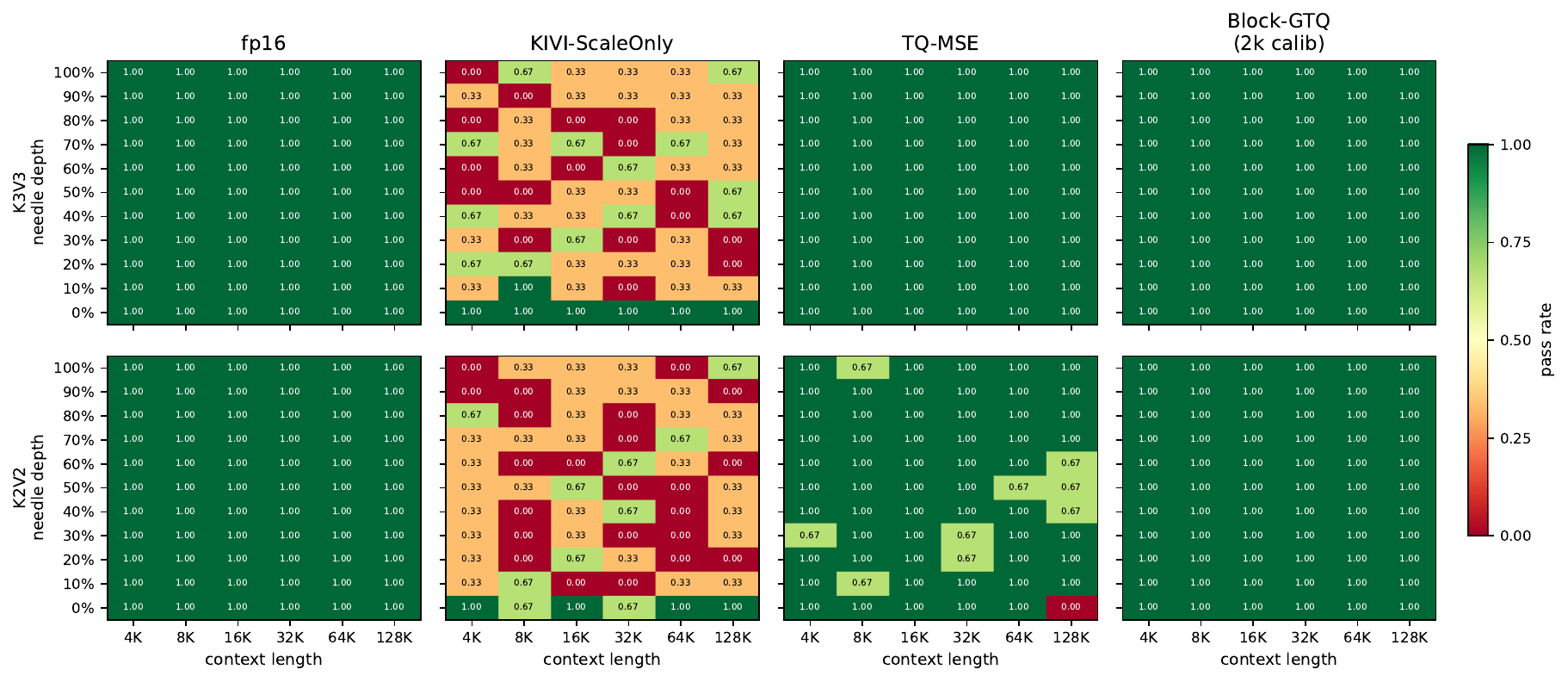}
  \caption{\textbf{Single-needle NIAH on Llama-3.1-8B-Instruct.} Pass rate over
  context length and needle depth, with three paired trials per cell. Rows are
  K3V3 and K2V2; columns are fp16, KIVI-ScaleOnly, TQ-MSE, and Block-GTQ.}
  \label{fig:niah-heatmap-llama}
\end{figure}

\FloatBarrier
\subsection{Qwen2.5-7B NIAH with context extension}

Table~\ref{tab:niah-qwen-full} gives the task-level record behind the Qwen
rows in Table~\ref{tab:niah-summary-main}. Native RoPE is used through 32K and
static YaRN (factor 4) at 64K/128K for every method.

\begin{table}[H]
  \centering
  \caption{Multi-task NIAH scores (\%) on Qwen2.5-7B-Instruct. Each entry
  averages context lengths 4K--128K, needle depths 0--100\%, and three paired
  trials per cell. KIVI-ScaleOnly is retained as a no-buffer diagnostic.}
  \label{tab:niah-qwen-full}
  \scriptsize
  \setlength{\tabcolsep}{4.0pt}
  \begin{tabular}{llrrrrrrr}
    \toprule
    Rate & Method & single & distract & multi & m-key & m-value & m-query & \textbf{Avg} \\
    \midrule
    fp16 & fp16 & 100.00 & 100.00 & 100.00 & 33.33 & 66.67 & 33.33 & 72.22 \\
    \midrule
    K3V3 & KIVI-ScaleOnly & 68.69 & 55.56 & 51.01 & 22.39 & 29.12 & 22.05 & 41.47 \\
     & TQ-MSE & 0.00 & 0.00 & 0.00 & 0.00 & 0.00 & 0.00 & 0.00 \\
     & Block-GTQ & \textbf{100.00} & \textbf{100.00} & \textbf{100.00} & \textbf{36.70} & \textbf{63.47} & \textbf{33.33} & \textbf{72.25} \\
    \midrule
    K3V2 & KIVI-ScaleOnly & 68.18 & 45.45 & 41.41 & 26.09 & 26.60 & 21.21 & 38.16 \\
     & TQ-MSE & 0.00 & 0.00 & 0.00 & 0.00 & 0.00 & 0.00 & 0.00 \\
     & Block-GTQ & \textbf{100.00} & \textbf{100.00} & \textbf{100.00} & \textbf{42.59} & \textbf{65.32} & \textbf{33.33} & \textbf{73.54} \\
    \midrule
    K2V2 & KIVI-ScaleOnly & 24.75 & 16.16 & 15.66 & 12.63 & 7.41 & 8.59 & 14.20 \\
     & TQ-MSE & 0.00 & 0.00 & 0.00 & 0.00 & 0.00 & 0.00 & 0.00 \\
     & Block-GTQ & \textbf{82.32} & \textbf{76.77} & \textbf{85.86} & \textbf{37.88} & \textbf{58.08} & \textbf{29.29} & \textbf{61.70} \\
    \bottomrule
  \end{tabular}
\end{table}

\begin{table}[H]
  \centering
  \caption{Qwen2.5-7B-Instruct NIAH six-task mean (\%) by context length.
  Cells through 32K use native RoPE; 64K/128K use static YaRN (factor 4).
  Higher is better; bold marks the best quantized method within each budget.}
  \label{tab:niah-qwen-length}
  \scriptsize
  \setlength{\tabcolsep}{5pt}
  \begin{tabular}{llrrrrrr}
    \toprule
    Rate & Method & 4K & 8K & 16K & 32K & 64K & 128K \\
    \midrule
    fp16 & fp16 & 72.22 & 72.22 & 72.22 & 72.22 & 72.22 & 72.22 \\
    \midrule
    K3V3 & KIVI-ScaleOnly & 53.20 & 46.13 & 46.63 & 40.74 & 32.66 & 29.46 \\
     & TQ-MSE & 0.00 & 0.00 & 0.00 & 0.00 & 0.00 & 0.00 \\
     & Block-GTQ & \textbf{72.22} & \textbf{72.22} & \textbf{72.22} & \textbf{72.22} & \textbf{73.74} & \textbf{70.88} \\
    \midrule
    K3V2 & KIVI-ScaleOnly & 51.35 & 46.80 & 39.90 & 39.73 & 24.58 & 26.60 \\
     & TQ-MSE & 0.00 & 0.00 & 0.00 & 0.00 & 0.00 & 0.00 \\
     & Block-GTQ & \textbf{72.22} & \textbf{72.56} & \textbf{73.57} & \textbf{73.23} & \textbf{77.44} & \textbf{72.22} \\
    \midrule
    K2V2 & KIVI-ScaleOnly & 23.06 & 17.85 & 16.33 & 10.27 & 11.28 & 6.40 \\
     & TQ-MSE & 0.00 & 0.00 & 0.00 & 0.00 & 0.00 & 0.00 \\
     & Block-GTQ & \textbf{73.57} & \textbf{73.23} & \textbf{71.55} & \textbf{75.93} & \textbf{39.39} & \textbf{36.53} \\
    \bottomrule
  \end{tabular}
\end{table}

\FloatBarrier
\subsection{Qwen2.5-32B NIAH scaling}
\label{app:niah-qwen32b}

Table~\ref{tab:niah-qwen32b} extends the K2V2 Stress comparison to
Qwen2.5-32B-Instruct. All methods use native RoPE through 32K and static
YaRN at 64K/128K. Block-GTQ averages 88.08, close to fp16's 88.61,
while TQ-MSE falls to 63.78.

\begin{table}[H]
  \centering
  \caption{Qwen2.5-32B-Instruct K2V2 NIAH under Stress, with fp16 reference.
  Scores are six-task means (\%) over eleven depths and three trials; Avg
  averages the six lengths.}
  \label{tab:niah-qwen32b}
  \small
  \begin{tabular*}{0.92\linewidth}{@{\extracolsep{\fill}}lrrrrrrr@{}}
    \toprule
    Method & 4K & 8K & 16K & 32K & 64K & 128K & Avg \\
    \midrule
    fp16 & 88.89 & 88.89 & 88.89 & 88.89 & 87.88 & 88.22 & 88.61 \\
    \midrule
    TQ-MSE & 90.57 & 86.03 & 63.64 & 56.23 & 46.30 & 39.90 & 63.78 \\
    Block-GTQ & 88.55 & 88.55 & 88.89 & 88.22 & 86.53 & 87.71 & \textbf{88.08} \\
    \bottomrule
  \end{tabular*}
\end{table}

\FloatBarrier
\subsection{Cross-family NIAH campaign}
\label{app:niah-cross-family}

We additionally compare Block-GTQ with Kitty, KVQuant, CommVQ, and
AQUA-KV~\citep{xia2025kitty,hooper2024kvquant,li2025commvq,
malinovskii2025aqua} at K2V2 under no-buffer and with-buffer
regimes. This campaign uses the same six tasks, lengths, depths, three trials,
fractional scoring, and Qwen context-extension rule as Appendix~\ref{app:niah-protocol},
but a different harness and prompt builders from
Tables~\ref{tab:niah-llama-full}--\ref{tab:niah-qwen-length}. Its results must
therefore be compared only within Tables~\ref{tab:niah-cross-llama}
and~\ref{tab:niah-cross-qwen}. No buffer and with buffer identify protection
regimes, not identical cache layouts or total memory: the nominal K2V2 payload
excludes metadata, learned model-side state, and retained high-precision state.

\begin{table}[H]
  \centering
  \caption{Cross-family NIAH six-task mean (\%) by context length on
  Llama-3.1-8B-Instruct at K2V2. \texttt{Mean} averages the six lengths.
  Parentheses give the high-precision allowance in tokens (token-equivalents
  for sparse KVQuant outliers); bold marks the highest quantized mean in each
  regime.}
  \label{tab:niah-cross-llama}
  \scriptsize
  \setlength{\tabcolsep}{4.3pt}
  \begin{tabular}{llrrrrrrr}
    \toprule
    Regime & Method & 4K & 8K & 16K & 32K & 64K & 128K & \textbf{Mean} \\
    \midrule
    fp16 & fp16 & 100.00 & 99.83 & 99.66 & 99.83 & 100.00 & 97.98 & 99.55 \\
    \midrule
    \multirow{4}{*}{No buffer}
      & Block-GTQ (0) & 92.59 & 96.30 & 93.43 & 97.47 & 97.14 & 89.73 & 94.44 \\
      & KVQuant-0.01\% ($0.0001T$) & 20.88 & 12.46 & 12.96 & 4.88 & 10.27 & 1.85 & 10.55 \\
      & CommVQ (1) & 98.82 & 96.30 & 97.81 & 97.31 & 99.49 & 93.10 & 97.14 \\
      & AQUA-KV (0) & 98.32 & 98.48 & 97.31 & 97.64 & 98.65 & 96.97 & \textbf{97.90} \\
    \midrule
    \multirow{5}{*}{With buffer}
      & Block-GTQ (132) & 98.82 & 98.82 & 98.15 & 99.66 & 99.83 & 95.79 & 98.51 \\
      & Kitty (132) & 93.94 & 94.95 & 94.61 & 96.97 & 98.15 & 92.59 & 95.20 \\
      & KVQuant ($\approx132$) & 99.16 & 83.33 & 86.20 & 90.24 & 86.20 & 40.74 & 80.98 \\
      & CommVQ (128) & 98.82 & 98.65 & 97.64 & 99.33 & 99.66 & 95.79 & 98.32 \\
      & AQUA-KV (132) & 99.49 & 99.49 & 98.99 & 99.49 & 99.83 & 96.30 & \textbf{98.93} \\
    \bottomrule
  \end{tabular}
\end{table}

\begin{table}[H]
  \centering
  \caption{Cross-family NIAH six-task mean (\%) by context length on
  Qwen2.5-7B-Instruct at K2V2. Native RoPE is used through 32K and static YaRN
  (factor 4) at 64K/128K for every evaluated row. \texttt{Mean} averages the six lengths;
  bold marks the highest quantized mean in each regime. AQUA-KV dashes denote
  unavailable validated NIAH results, not zeros; the separate no-buffer
  generation diagnostic is reported in Table~\ref{tab:aquakv-qwen-diagnostic}.}
  \label{tab:niah-cross-qwen}
  \scriptsize
  \setlength{\tabcolsep}{4.3pt}
  \begin{tabular}{llrrrrrrr}
    \toprule
    Regime & Method & 4K & 8K & 16K & 32K & 64K & 128K & \textbf{Mean} \\
    \midrule
    fp16 & fp16 & 77.78 & 77.78 & 77.78 & 77.78 & 77.61 & 74.41 & 77.19 \\
    \midrule
    \multirow{4}{*}{No buffer}
      & Block-GTQ (0) & 78.96 & 76.94 & 73.74 & 70.54 & 46.80 & 30.98 & \textbf{62.99} \\
      & KVQuant-0.01\% ($0.0001T$) & 25.59 & 17.34 & 9.43 & 5.22 & 0.17 & 0.17 & 9.65 \\
      & CommVQ (1) & 90.07 & 79.46 & 85.86 & 91.25 & 15.15 & 0.84 & 60.44 \\
      & AQUA-KV (0) & -- & -- & -- & -- & -- & -- & -- \\
    \midrule
    \multirow{5}{*}{With buffer}
      & Block-GTQ (132) & 78.28 & 77.78 & 77.95 & 77.27 & 78.28 & 75.76 & \textbf{77.55} \\
      & Kitty (132) & 76.43 & 76.94 & 74.07 & 72.39 & 69.02 & 65.82 & 72.45 \\
      & KVQuant ($\approx132$) & 72.39 & 66.16 & 56.23 & 56.06 & 38.72 & 26.43 & 52.67 \\
      & CommVQ (128) & 90.40 & 88.22 & 87.37 & 94.61 & 34.51 & 11.11 & 67.70 \\
      & AQUA-KV (132) & -- & -- & -- & -- & -- & -- & -- \\
    \bottomrule
  \end{tabular}
\end{table}

\FloatBarrier
\paragraph{Protection and rate accounting.}
No buffer removes the fixed high-precision window, but CommVQ still requires
one staging token and KVQuant retains 0.01\% sparse fp16 outliers. With buffer,
Block-GTQ, Kitty, and AQUA-KV retain four sink plus 128 recent tokens; CommVQ
retains 128 recent tokens without a sink. KVQuant instead adjusts its outlier
fraction with length to target approximately 132 token-equivalents. These are
scattered values plus sparse indices, not a contiguous 132-token window or an
exact match in total bytes. Learned codebooks and predictors are additional
method-specific state, so these tables provide cross-family context, not a
controlled allocation ablation.

\paragraph{Kitty: no-buffer limitation.}
The evaluated Kitty implementation couples quantization-group size to its fp16
flush buffer and cannot use a zero-length buffer. Its 16-token low-buffer
proxy still retains fp16 tokens; the smaller groups increase scale/minimum
metadata to about 2 additional b/d, giving approximately
$2.00+32/16=4.00$ b/d for compressed codes and metadata at K2V2,
before retained fp16 state. It is therefore neither buffer-free nor
rate-matched and is omitted from the no-buffer comparison.

\paragraph{CommVQ: Qwen context extension.}
The Qwen codebooks were trained for native RoPE frequencies, whereas 64K/128K
evaluation uses static YaRN. The drop at 64K/128K coincides with this switch.
A mismatch between native-RoPE codebook training and YaRN evaluation is a
possible explanation, but this sweep does not separate it from context-length
effects. These results do not establish a generic context-length limit for
CommVQ.

\paragraph{AQUA-KV: Qwen generation failure.}
\phantomsection\label{para:aquakv-qwen-diagnostic}
To characterize generation failure in our Qwen2.5-7B-Instruct
AQUA-KV port, we examine three Qasper items (0--2), with
Llama-3.1-8B-Instruct as a control. Both models run in bfloat16 with
model-specific fitted predictors, greedy decoding capped at 128 new tokens,
and no sink or recent-token buffer. We vary the Hadamard group size and
bypass K or V residual quantization, separately or together.

Table~\ref{tab:aquakv-qwen-diagnostic} reports word repetition
$r_{\mathrm{rep}}=1-U/W$, where $W$ counts whitespace-delimited words and
$U$ counts distinct words, retaining case and punctuation. This
length-dependent statistic measures repetition, not answer accuracy.

\begin{table}[H]
  \centering
  \caption{AQUA-KV word-repetition ratios on three Qasper items.
  EDEN denotes 2-bit EDEN/HIGGS residual coding; Identity bypasses the
  indicated residual quantizer. First-layer K/V quantization remains 4-bit,
  and the predictor and RoPE undo/reapply operations remain active in every
  configuration.}
  \label{tab:aquakv-qwen-diagnostic}
  \small
  \setlength{\tabcolsep}{5pt}
  \begin{tabular}{lrllrrr}
    \toprule
    Model & Group size & K residual & V residual & Item 0 & Item 1 & Item 2 \\
    \midrule
    Qwen & 1024 & EDEN & EDEN & 0.953 & 0.894 & 0.969 \\
    Qwen & 512 & EDEN & EDEN & 0.953 & 0.894 & 0.969 \\
    Qwen & 256 & EDEN & EDEN & 0.969 & 0.977 & 0.977 \\
    \midrule
    Qwen & 1024 & Identity & EDEN & 0.143 & 0.133 & 0.411 \\
    Qwen & 1024 & EDEN & Identity & 0.961 & 0.953 & 0.961 \\
    Qwen & 1024 & Identity & Identity & 0.091 & 0.222 & 0.143 \\
    \midrule
    Llama & 1024 & EDEN & EDEN & 0.042 & 0.040 & 0.000 \\
    \bottomrule
  \end{tabular}
\end{table}

All three group sizes produce highly repetitive, malformed Qwen outputs.
Bypassing V alone does not resolve this behavior; bypassing K reduces
repetition but still produces malformed answers. Item 0 illustrates why low
repetition is insufficient (line breaks normalized; ellipsis marks truncation):
\begin{quote}\small
  EDEN/EDEN: ``un pesticCNT p\'erdida Rencontre Rencontre pestic \ldots''\\
  Identity/Identity: ``un user How is the-called so determined in the the article?''
\end{quote}

Our AQUA-KV port also exhibits generation collapse with zero AIME
accuracy on DS-R1-Qwen-7B (Appendix~\ref{para:aime-aquakv}).

\FloatBarrier
\subsection{Controlled 16-task LongBench-EN panel}
\label{app:longbench-full}

Tables~\ref{tab:longbench-8b-full}, \ref{tab:longbench-70b-full},
and~\ref{tab:longbench-qwen32b-full} give the task-level records behind
Table~\ref{tab:longbench-summary-main} and the Qwen-32B extension. All use
the no-buffer protocol. The 8B panel also includes K3V2 for the rate
sensitivity analysis at the end of this section; the 70B and 32B tables use
K3V3 and K2V2.
Appendix~\ref{app:longbench-protocol} defines the tasks and metrics.

\begin{table}[H]
  \centering
  \caption{LongBench-EN scores on Llama-3.1-8B-Instruct under Stress. All 16 English tasks are shown; Avg16 is their unweighted mean. TQ and Block denote uniform TQ-MSE and Block-GTQ.}
  \label{tab:longbench-8b-full}
  \scriptsize
  \renewcommand{\arraystretch}{1.0}
  \begin{tabular*}{\linewidth}{@{\extracolsep{\fill}}lrrrrrrr@{}}
  \toprule
  & & \multicolumn{2}{c}{K3V3} & \multicolumn{2}{c}{K3V2} & \multicolumn{2}{c}{K2V2} \\
  \cmidrule(lr){3-4}\cmidrule(lr){5-6}\cmidrule(l){7-8}
  Task & fp16 & TQ & Block & TQ & Block & TQ & Block \\
  \midrule
  Qasper & 44.70 & 43.96 & 46.17 & 44.02 & 47.09 & 32.66 & 45.18 \\
  MultiFieldQA-en & 55.77 & 54.33 & 53.98 & 53.75 & 52.10 & 43.61 & 50.35 \\
  HotpotQA & 57.62 & 53.37 & 54.71 & 52.18 & 54.09 & 33.65 & 50.68 \\
  2WikiMultihopQA & 48.85 & 44.12 & 46.48 & 41.20 & 40.49 & 27.25 & 36.49 \\
  GovReport & 34.54 & 33.31 & 33.81 & 33.08 & 32.73 & 28.13 & 32.51 \\
  TREC & 72.50 & 71.00 & 72.50 & 71.00 & 72.00 & 55.50 & 70.50 \\
  PassageRetrieval-en & 99.50 & 95.50 & 99.50 & 97.00 & 98.00 & 31.92 & 88.50 \\
  LCC & 65.13 & 63.59 & 65.17 & 60.96 & 66.11 & 42.24 & 56.57 \\
  QMSum & 25.35 & 25.32 & 25.23 & 25.30 & 24.80 & 22.93 & 24.34 \\
  MultiNews & 26.96 & 26.65 & 26.57 & 26.85 & 26.75 & 26.35 & 26.01 \\
  TriviaQA & 91.65 & 91.45 & 89.98 & 90.69 & 91.20 & 78.78 & 89.12 \\
  SAMSum & 43.63 & 42.11 & 43.64 & 42.48 & 43.42 & 38.76 & 41.73 \\
  RepoBench-P & 58.65 & 56.13 & 54.62 & 51.56 & 54.53 & 43.59 & 50.24 \\
  NarrativeQA & 28.08 & 29.29 & 27.97 & 27.02 & 30.35 & 17.46 & 23.99 \\
  MuSiQue & 32.72 & 25.51 & 30.72 & 24.76 & 29.60 & 13.85 & 27.33 \\
  PassageCount & 10.08 & 2.26 & 5.83 & 2.72 & 3.00 & 2.07 & 4.69 \\
  \midrule
  \textbf{Avg16} & 49.73 & 47.37 & \textbf{48.56} & 46.54 & \textbf{47.89} & 33.67 & \textbf{44.89} \\
  \bottomrule
  \end{tabular*}
\end{table}

\begin{table}[H]
  \centering
  \caption{LongBench-EN scores on Llama-3.1-70B-Instruct under Stress. All 16 English tasks are shown; Avg16 is their unweighted mean.}
  \label{tab:longbench-70b-full}
  \footnotesize
  \renewcommand{\arraystretch}{1.0}
  \begin{tabular*}{\linewidth}{@{\extracolsep{\fill}}lrrrrr@{}}
  \toprule
  & & \multicolumn{2}{c}{K3V3} & \multicolumn{2}{c}{K2V2} \\
  \cmidrule(lr){3-4}\cmidrule(l){5-6}
  Task & fp16 & TQ-MSE & Block-GTQ & TQ-MSE & Block-GTQ \\
  \midrule
  Qasper & 49.32 & 49.53 & 48.57 & 44.83 & 49.42 \\
  MultiFieldQA-en & 54.55 & 54.17 & 56.34 & 49.16 & 52.93 \\
  HotpotQA & 65.51 & 62.26 & 64.45 & 55.29 & 62.32 \\
  2WikiMultihopQA & 67.96 & 64.70 & 67.05 & 45.86 & 60.71 \\
  GovReport & 34.99 & 34.44 & 34.32 & 33.07 & 33.43 \\
  TREC & 76.50 & 77.00 & 75.00 & 75.00 & 75.50 \\
  PassageRetrieval-en & 98.50 & 97.50 & 99.00 & 92.50 & 97.00 \\
  LCC & 39.61 & 38.85 & 41.04 & 47.31 & 37.51 \\
  QMSum & 24.18 & 24.97 & 24.46 & 23.97 & 24.31 \\
  MultiNews & 26.82 & 26.66 & 26.40 & 26.69 & 26.44 \\
  TriviaQA & 94.04 & 94.38 & 94.71 & 91.24 & 93.53 \\
  SAMSum & 47.49 & 46.27 & 47.73 & 45.34 & 46.59 \\
  RepoBench-P & 63.25 & 60.58 & 63.03 & 52.91 & 60.07 \\
  NarrativeQA & 36.18 & 32.81 & 34.23 & 31.77 & 35.57 \\
  MuSiQue & 46.17 & 44.67 & 44.40 & 32.23 & 37.93 \\
  PassageCount & 17.50 & 19.50 & 17.50 & 8.03 & 9.00 \\
  \midrule
  \textbf{Avg16} & 52.66 & 51.77 & \textbf{52.39} & 47.20 & \textbf{50.14} \\
  \bottomrule
  \end{tabular*}
\end{table}

\begin{table}[H]
  \centering
  \caption{LongBench-EN scores on Qwen2.5-32B-Instruct under Stress. All 16 English tasks are shown; Avg16 is their unweighted mean.}
  \label{tab:longbench-qwen32b-full}
  \footnotesize
  \renewcommand{\arraystretch}{1.0}
  \begin{tabular*}{\linewidth}{@{\extracolsep{\fill}}lrrrrr@{}}
  \toprule
  & & \multicolumn{2}{c}{K3V3} & \multicolumn{2}{c}{K2V2} \\
  \cmidrule(lr){3-4}\cmidrule(l){5-6}
  Task & fp16 & TQ-MSE & Block-GTQ & TQ-MSE & Block-GTQ \\
  \midrule
  Qasper & 44.31 & 43.43 & 43.73 & 31.90 & 44.00 \\
  MultiFieldQA-en & 52.37 & 51.22 & 53.15 & 43.76 & 52.16 \\
  HotpotQA & 63.21 & 61.36 & 62.46 & 47.40 & 59.07 \\
  2WikiMultihopQA & 61.00 & 57.49 & 56.72 & 39.89 & 52.83 \\
  GovReport & 30.30 & 29.25 & 30.02 & 22.84 & 27.99 \\
  TREC & 73.50 & 73.00 & 72.00 & 69.00 & 73.00 \\
  PassageRetrieval-en & 100.00 & 98.50 & 100.00 & 61.33 & 96.50 \\
  LCC & 54.36 & 54.83 & 46.14 & 38.83 & 41.40 \\
  QMSum & 23.42 & 22.88 & 22.50 & 20.02 & 22.40 \\
  MultiNews & 22.74 & 22.67 & 22.71 & 21.37 & 22.43 \\
  TriviaQA & 87.61 & 88.19 & 86.88 & 86.38 & 88.74 \\
  SAMSum & 45.70 & 46.37 & 46.82 & 40.17 & 46.73 \\
  RepoBench-P & 39.97 & 35.82 & 36.47 & 42.46 & 30.49 \\
  NarrativeQA & 29.00 & 31.43 & 29.45 & 20.84 & 28.84 \\
  MuSiQue & 38.73 & 37.99 & 39.43 & 19.22 & 32.42 \\
  PassageCount & 10.50 & 6.79 & 5.17 & 5.33 & 6.81 \\
  \midrule
  \textbf{Avg16} & 48.55 & \textbf{47.58} & 47.10 & 38.17 & \textbf{45.36} \\
  \bottomrule
  \end{tabular*}
\end{table}

At K2V2, Block-GTQ improves the 16-task mean over TQ-MSE by 2.94 points on
Llama-70B and 7.19 on Qwen-32B. At K3V3, it is 0.62 points higher on
Llama-70B but 0.48 points lower on Qwen-32B; the advantage is larger at the
lower-rate K2V2 operating point.

\FloatBarrier
\subsection{Qwen2.5-7B-Instruct LongBench-EN}
\label{app:longbench-qwen7b}

Table~\ref{tab:longbench-qwen7b} reports the complete Qwen2.5-7B-Instruct
comparison. All quantized
methods use the Stress setting (no fp16 sink or recent-token window). These
inputs stay within the model's native context, so YaRN is not used.

\begin{table}[H]
  \centering
  \caption{Qwen2.5-7B-Instruct LongBench-EN under Stress. All 16 task scores
  are shown; \texttt{Avg16} is their unweighted mean. Bold marks the better
  quantized \texttt{Avg16} within each budget.}
  \label{tab:longbench-qwen7b}
  \footnotesize
  \setlength{\tabcolsep}{3pt}
  \renewcommand{\arraystretch}{1.03}
  \begin{tabular*}{\linewidth}{@{\extracolsep{\fill}}lrrrrrrr@{}}
    \toprule
    & & \multicolumn{2}{c}{K3V3} & \multicolumn{2}{c}{K3V2} & \multicolumn{2}{c}{K2V2} \\
    \cmidrule(lr){3-4}\cmidrule(lr){5-6}\cmidrule(l){7-8}
    Task & fp16 & TQ-MSE & Block-GTQ & TQ-MSE & Block-GTQ & TQ-MSE & Block-GTQ \\
    \midrule
    Qasper & 43.77 & 1.02 & 44.91 & 0.70 & 41.28 & 0.68 & 30.10 \\
    MultiFieldQA-en & 52.62 & 1.13 & 53.08 & 1.17 & 52.96 & 1.28 & 43.05 \\
    HotpotQA & 57.68 & 0.42 & 59.03 & 0.35 & 51.88 & 0.63 & 41.67 \\
    2WikiMultihopQA & 47.19 & 0.60 & 46.39 & 0.24 & 40.20 & 0.08 & 30.20 \\
    GovReport & 31.87 & 2.89 & 31.21 & 2.69 & 30.86 & 2.61 & 27.98 \\
    TREC & 71.50 & 0.00 & 70.50 & 0.00 & 67.00 & 0.00 & 60.00 \\
    PassageRetrieval-en & 100.00 & 0.38 & 99.50 & 0.81 & 97.50 & 0.07 & 68.88 \\
    LCC & 61.88 & 10.55 & 61.95 & 10.78 & 58.14 & 8.07 & 59.59 \\
    QMSum & 23.67 & 8.11 & 23.57 & 7.89 & 23.03 & 6.50 & 21.45 \\
    MultiNews & 23.84 & 3.70 & 23.44 & 4.12 & 23.62 & 2.84 & 23.20 \\
    TriviaQA & 89.97 & 0.94 & 88.85 & 1.45 & 87.22 & 2.26 & 74.29 \\
    SAMSum & 45.95 & 2.23 & 44.27 & 1.85 & 43.25 & 3.48 & 38.42 \\
    RepoBench-P & 68.26 & 6.93 & 62.58 & 6.92 & 63.87 & 8.38 & 54.01 \\
    NarrativeQA & 28.53 & 0.42 & 27.58 & 0.35 & 26.46 & 0.54 & 10.08 \\
    MuSiQue & 29.85 & 0.64 & 29.76 & 0.33 & 25.63 & 0.27 & 18.95 \\
    PassageCount & 8.50 & 1.33 & 8.50 & 0.72 & 5.12 & 0.07 & 2.25 \\
    \midrule
    \textbf{Avg16} & 49.07 & 2.58 & \textbf{48.44} & 2.52 & \textbf{46.13} & 2.36 & \textbf{37.76} \\
    \bottomrule
  \end{tabular*}
\end{table}

Block-GTQ nearly matches fp16 at K3V3 (48.44 versus 49.07) and exceeds
uniform TQ-MSE at every budget, although its Avg16 falls to 37.76 at K2V2.

\FloatBarrier
\subsection{Key-versus-value rate sensitivity}
\label{app:kv-rate-sensitivity}

Prior work reports benefits from assigning higher precision to keys than to
values~\citep{hariri2026quantize}. We examine the related question of K/V rate
sensitivity on Llama-3.1-8B-Instruct under our no-buffer protocol.
Tables~\ref{tab:niah-llama-full} and~\ref{tab:longbench-8b-full} provide two
adjacent one-factor contrasts: K3V3 to K3V2 reduces only the value rate, with
K fixed at 3 bits; K3V2 to K2V2 reduces only the key rate, with V fixed at
2 bits. Table~\ref{tab:kv-rate-sensitivity} summarizes these score drops.

\begin{table}[H]
  \centering
  \caption{Key- and value-rate sensitivity on Llama-3.1-8B-Instruct without
  fp16 sink or recent-token protection. Rates are average logical bits per
  coordinate. Drop V is K3V3 minus K3V2; Drop K is K3V2 minus K2V2, both in
  score points. Positive values denote a lower score after reducing that
  rate. NIAH uses partial-credit Overall and LongBench uses its 16-task
  average.}
  \label{tab:kv-rate-sensitivity}
  \scriptsize
  \setlength{\tabcolsep}{5pt}
  \begin{tabular}{llrrrrr}
    \toprule
    Benchmark & Method & K3V3 & K3V2 & K2V2 & Drop V & Drop K \\
    \midrule
    NIAH & TQ-MSE & 97.73 & 95.68 & 70.62 & 2.05 & 25.06 \\
         & Block-GTQ & 98.40 & 96.80 & 97.36 & 1.60 & $-0.56$ \\
    \midrule
    LongBench-EN & TQ-MSE & 47.37 & 46.54 & 33.67 & 0.83 & 12.87 \\
                 & Block-GTQ & 48.56 & 47.89 & 44.89 & 0.67 & 3.00 \\
    \bottomrule
  \end{tabular}
\end{table}

For uniform TQ-MSE, the K-rate reduction produces a much larger score drop
than the V-rate reduction on both benchmarks: 25.06 versus 2.05 points on
NIAH and 12.87 versus 0.83 on LongBench-EN. Block-GTQ reduces the LongBench
K-rate drop to 3.00 points and remains near the NIAH ceiling across all three
budgets. Thus, lowering the key rate is particularly damaging for uniform
allocation in these contrasts, while adaptive allocation mitigates the
degradation at matched K/V budgets.

These results provide empirical motivation for focusing allocation on K,
not a universal ordering of K/V sensitivity. The two contrasts start from
different K/V budgets; without K2V3, they do not compare key-favored and
value-favored allocations at the same total logical budget. Block-GTQ
addresses the complementary problem of distributing a fixed key budget
across RoPE blocks.

\FloatBarrier

\section{AIME Calibration, Protocol, and Extended Analysis}
\label{app:aime-extended}

\subsection{Calibration-corpus swap}
\label{app:aime-calibration-corpus}

For Table~\ref{tab:aime-calibration}, the only changed variable is the
calibration corpus. The WikiText-2 arm uses the first 2048 test tokens. The
MATH arm uses the canonical MATH training split~\citep{hendrycks2021math},
interleaves its seven subjects in round-robin order, renders a contiguous
\texttt{Problem: ... Solution: ...} stream, and takes the first 2048 tokens
under each model's tokenizer. Within each subject, examples retain their
training-split order; no shuffle is applied. Both arms use Block-GTQ K3V2 under Stress,
with the same score, GQA aggregation, allocator, and generation protocol.

\begin{table}[H]
  \centering
  \caption{Complete AIME calibration-source results for Block-GTQ K3V2 under
  Stress (avg@8, \%). Both calibration streams contain 2048 tokens;
  $\Delta=\text{MATH}-\text{WikiText-2}$ is in percentage points and is
  computed before rounding.}
  \label{tab:aime-calibration-detailed}
  \scriptsize
  \setlength{\tabcolsep}{4pt}
  \begin{tabular}{llrrrr}
    \toprule
    Model & Contest & fp16 & WikiText-2 & MATH & $\Delta$ \\
    \midrule
    \multirow{3}{*}{DS-R1-Llama-8B}
      & AIME 2024 & 43.33 & 27.50 & 27.50 & $0.00$ \\
      & AIME 2025 & 28.75 & 25.42 & 25.83 & $+0.42$ \\
      & Combined  & 36.04 & 26.46 & 26.67 & $+0.21$ \\
    \midrule
    \multirow{3}{*}{DS-R1-Qwen-7B}
      & AIME 2024 & 54.17 & 49.58 & 51.67 & $+2.08$ \\
      & AIME 2025 & 37.92 & 38.75 & 37.50 & $-1.25$ \\
      & Combined  & 46.04 & 44.17 & 44.58 & $+0.42$ \\
    \bottomrule
  \end{tabular}
\end{table}

\FloatBarrier
\subsection{Protocol}
\label{app:aime-protocol}

The main comparison uses K3V2: keys average 3 logical bits per dimension and
values use 2 bits per dimension. We evaluate all 30 problems from
\texttt{HuggingFaceH4/aime\_2024} and both 15-problem splits of
\texttt{opencompass/AIME2025}, reporting the two contests separately. For each
problem, we sample eight responses with temperature $0.6$, top-$p=0.95$, and
\texttt{max\_new\_tokens}$=32768$. Avg@8 is the mean binary correctness over
the resulting $30\times8$ responses within each contest, not pass@8.

We apply each model's chat template, enter its native \texttt{<think>} mode,
and request step-by-step reasoning followed by a boxed final answer. The scorer
extracts the integer from the last balanced \texttt{\textbackslash boxed\{\}}
expression, then falls back to an explicit final-answer pattern and finally the
last generated integer. Predictions and references are compared numerically,
so leading zeros and annotations such as degree symbols do not affect scoring.

Block-GTQ is calibrated on the first 2048 WikiText-2 test tokens without
positional interpolation. In the controlled comparison, Stress uses
sink/window $0/0$, while Protected uses four fp16 sink tokens and 128 recent
fp16 tokens. For cross-family methods, the same panel names denote reduced
and protected high-precision state, not identical cache layouts. The ledger
below records each method's protection, rate convention, and auxiliary state.

\subsection{Rate and protection ledger}

The nominal K3V2 payload is $(3+2)/2=2.50$ b/d, averaged over key and value
coordinates. Table~\ref{tab:aime-rate-ledger} distinguishes this payload from
rates including selected overheads, analytic or nominal targets, and reported
realized rates. These entries do not form an equal-memory ranking; only
Block-GTQ versus uniform TQ-MSE is the controlled allocation comparison.

\begin{table}[H]
  \centering
  \caption{Rate conventions and protection for the nominal AIME K3V2
  comparison. Logical and nominal residual payloads exclude codec metadata,
  protected tokens, and model-side weights. Included metadata or sparse state
  is marked explicitly; these entries are not measurements of total physical
  cache storage. The last column records protection, precision exceptions,
  and other state.}
  \label{tab:aime-rate-ledger}
  \scriptsize
  \setlength{\tabcolsep}{3pt}
  \renewcommand{\arraystretch}{1.12}
  \begin{tabular*}{\textwidth}{@{\extracolsep{\fill}}l
    >{\raggedright\arraybackslash}p{0.29\textwidth}
    >{\raggedright\arraybackslash}p{0.23\textwidth}
    >{\raggedright\arraybackslash}p{0.33\textwidth}@{}}
    \toprule
    Setting & Method & Rate (b/d) and convention & Protection and other state \\
    \midrule
    Stress
      & Block-GTQ / uniform TQ-MSE / KIVI-ScaleOnly & 2.50 (logical payload) & no protected tokens \\
      & AATC / AATC (variance-only) & 2.50 (global analytic rate) & first three layers K4/V2 (included); transforms extra; no protected tokens \\
      & AQUA-KV & 2.50 (nominal residual payload) & no protected tokens; first-layer K/V at 4 bits; predictors \\
      & Kitty & no matched Stress setting & 16-token proxy: $\approx4.50$ b/d (codes + metadata), plus retained fp16 state \\
      & CommVQ & 2.50 (code payload) & one-token staging; prescale; learned codebooks \\
      & KVQuant-0.01\% & $\approx2.50$ (incl. sparse values/indices) & 0.01\% fp16 outliers (included); no recent-token window \\
      & PM-KVQ (length-matched diagnostic) & 2.14--2.58 (range of four model--year mean realized rates) & offline reference-length budget; no protected tokens \\
      & fp16 & 16.00 (uncompressed) & full cache \\
    \midrule
    Protected
      & Block-GTQ / uniform TQ-MSE / KIVI & 2.50 (logical payload) & 4 sink + 128 recent fp16 tokens \\
      & AATC / AATC (variance-only) & 2.50 (global analytic rate) & first three layers K4/V2 (included); transforms and 4 sink + 128 recent fp16 tokens extra \\
      & AQUA-KV & 2.50 (nominal residual payload) & 4 sink + 128 recent; first-layer K/V at 4 bits; predictors \\
      & Kitty & $\approx2.75$ (codes + metadata) & 4 sink + 128 recent fp16 tokens \\
      & CommVQ & 2.50 (code payload) & 128 recent bf16 tokens; prescale; learned codebooks \\
      & KVQuant-0.5\% & $\approx2.65$ (incl. sparse values/indices) & 0.5\% fp16 outliers (included); no recent-token window \\
      & PM-KVQ (length-matched diagnostic) & 2.50 (nominal main-cache target) & 4 sink + 128 recent fp16 tokens; protection cost added separately \\
      & fp16 & 16.00 (uncompressed) & full cache \\
    \bottomrule
  \end{tabular*}
\end{table}

\subsection{Baseline implementations and result interpretation}
\label{app:aime-methods}

Table~\ref{tab:aime-k3v2-main} combines a controlled key-rate allocation
comparison with cross-family methods that differ in cache representation,
calibration or training, and high-precision state. The controlled comparison
asks whether changing the key allocation improves accuracy within a shared
quantization pipeline; the cross-family rows provide context under the
method-specific rates and protection policies in
Table~\ref{tab:aime-rate-ledger}. Below, we follow the table's method order,
explain the evaluated implementations, and connect their settings to the
observed results. Additional budget and implementation diagnostics are
identified separately from the main K3V2 comparison. Unless noted otherwise,
accuracy pairs are AIME 2024/2025 avg@8 scores in percent; Llama and Qwen refer
to DS-R1-Llama-8B and DS-R1-Qwen-7B.

\paragraph{Block-GTQ and uniform TQ-MSE.}
Block-GTQ assigns mixed integer widths to RoPE key blocks, averaging 3 bits
per key coordinate, and uses uniform 2-bit TQ-MSE for values. The uniform
baseline assigns 3 bits to every key coordinate while sharing the local
encoder, value path, logical payload, and protection policy. This is the
controlled key-rate allocation comparison~\citep{zandieh2025turboquant}.

Block-GTQ scores higher in every model--contest cell under both regimes.
The largest contrast occurs on Qwen: Block-GTQ scores $49.58/38.75$ in Stress while uniform TQ-MSE
scores $0.00/0.00$; with 4/128 protection the corresponding results are
$54.17/37.50$ and $7.50/8.75$. Because the local codec and protection are
matched within each regime, this comparison tests the allocation choice
rather than a change in value quantization or retained high-precision tokens.

\paragraph{KIVI.}
KIVI uses per-channel key and per-token value quantization~\citep{liu2024kivi}.
The Stress row is KIVI-ScaleOnly, a buffer-free diagnostic retaining the scale
quantizer, not KIVI's native residual-cache configuration. The Protected row
uses the official flush-time path with four sink and 128 recent fp16 tokens.

At the same nominal K3V2 payload, Llama scores rise from
$7.50/10.42$ to $43.33/27.08$ and Qwen from $28.75/17.50$ to
$52.92/35.83$ on switching to the protected path.
This compares the buffer-free diagnostic with KIVI's protected residual path;
it does not show that the unmodified method fails at this nominal rate.

\paragraph{AATC and AATC (variance-only).}
AATC builds per-head-group whitening transforms, scalar-quantizes the
transformed coefficients, and reverse-water-fills separate global K/V budgets
using query-energy and output-projection weights~\citep{laus2026aatc}. We use
a training-free reimplementation, as official code was unavailable for these
experiments, with pre-RoPE keys, per-channel P99 symmetric quantization, and
OpenR1-Math calibration. We check reconstruction consistency with a 16-bit
transform--RoPE round-trip test.

For K3V2, the first three layers use K4/V2. Their cost is
deducted from the global K/V budgets before allocating the remaining layers,
preserving global averages of K3/V2 (2.50 b/d). The experiment uses fake
quantization through the standard Hugging Face cache: its analytic rate
excludes transform metadata and the Protected fp16 region and is not a
physical-memory measurement.

AATC (variance-only) sets the attention-dependent weights to one
while keeping the transform, quantizer, allocator, and budget fixed. The AATC
paper describes this as a per-layer analogue of KVTC; it is not the separate
global, cross-layer KVTC pipeline~\citep{staniszewski2026kvtc}.

The two variants and two protection regimes address different questions.
Full versus variance-only AATC tests the attention-dependent allocation
weights within the same implementation. Protected versus Stress tests the
joint retention of four fp16 sink tokens and 128 recent fp16 tokens;
Stress removes both while preserving the other AATC settings.
Table~\ref{tab:aime-aatc-ablation-means} separates these two comparisons.

\begin{table}[H]
  \centering
  \caption{AATC two-contest mean avg@8 (\%) at K3V2, averaged over AIME
  2024 and 2025. Rows within each model compare allocation weights;
  columns compare the joint sink/recent-token protection policy.}
  \label{tab:aime-aatc-ablation-means}
  \small
  \setlength{\tabcolsep}{9pt}
  \begin{tabular}{llrr}
    \toprule
    Model & Variant & Stress & Protected \\
    \midrule
    \multirow{2}{*}{DS-R1-Llama-8B}
      & Full & 4.17 & 32.29 \\
      & Variance-only & 1.04 & 33.96 \\
    \midrule
    \multirow{2}{*}{DS-R1-Qwen-7B}
      & Full & 24.38 & 32.29 \\
      & Variance-only & 24.58 & 34.79 \\
    \bottomrule
  \end{tabular}
\end{table}

Both variants score lower without protection on both models, with the larger
drop on Llama. This comparison changes sink and recent-token protection
together and does not isolate their individual effects. In contrast, the
attention-dependent weights provide no consistent accuracy gain over the
variance-only variant in these runs. Thus, the effect of the protection
policy should be distinguished from the effect of the allocation weights.

\paragraph{AQUA-KV.}\phantomsection\label{para:aime-aquakv}
AQUA-KV predicts each layer's K/V and vector-quantizes the residual through its
EDEN/HIGGS path~\citep{malinovskii2025aqua}. Our asymmetric port uses a 3-bit
key grid, a 2-bit value grid, a 4-bit first layer, and predictors calibrated on
256 WikiText-2 samples because the reported RedPajama data were unavailable.
Protected retains four sink and 128 recent fp16 tokens, whereas Stress
removes this protection.

On Llama at K3V2, Protected scores $43.33/30.00$, close to fp16's
$43.33/28.75$, while Stress scores $37.92/29.17$. Their two-contest averages
are $36.67$ and $33.54$, respectively. An additional K2V2 comparison examines
the same protection contrast at a tighter residual budget: Stress scores
$21.25/20.00$ (combined $20.62$), versus a combined $35.00$ under protection.
The larger protection gap at K2V2 shows that the mild K3V2 gap on Llama does
not extend to this lower-budget setting.

On DS-R1-Qwen-7B, both K3V2 regimes score $0.00/0.00$; Protected also scores
$0.00/0.00$ at K2V2. In these Protected runs, all sampled trajectories reach
the 32,768-token cap and enter repeated-zero output. These are completed
evaluations with generation collapse after predictor calibration, not missing
or crashed runs. The same harness and cache path produce coherent Llama
generations; the cause of the Qwen collapse in this port remains unresolved.
These completed zero-accuracy runs are distinct from the unavailable
Qwen2.5-7B-Instruct NIAH entries in
Appendix~\ref{app:niah-cross-family}. The Qasper diagnostics in
Table~\ref{tab:aquakv-qwen-diagnostic} concern that different checkpoint and
task; they document malformed generation but do not establish the cause of
the DS-R1-Qwen-7B AIME failures.

AQUA-KV additionally stores fitted cross-layer predictors for each backbone.
Their storage depends on the number of predicted layers, K/V dimensions,
predictor architecture, and parameter precision. For a fixed model and
predictor configuration, this overhead does not grow with sequence length;
it is not a model-independent constant. Predictor storage is excluded from
the logical K3V2 payload and must be included in end-to-end memory accounting.

The additional Llama K2V2 Stress campaign also provides an implementation-cost
observation, separate from the K3V2 accuracy comparison. Without a recent-token
buffer, each new token enters the tested port's prediction, RoPE-undo/reapply,
residual-coding, and chunk-reconstruction path. Observed GPU utilization is
0--25\%, and the complete campaign (60 problems, avg@8) took approximately
5.8 days on four GPUs. These port-specific observations are consistent with
reconstruction and data-movement overhead, but do not isolate the bottleneck.
They are not a controlled throughput comparison or a prediction of the runtime
of an optimized AQUA-KV implementation.

\paragraph{Kitty.}
Kitty promotes 50\% of key channels from a 2-bit base to 4 bits and uses 2-bit
values~\citep{xia2025kitty}, giving a 2.50 b/d logical K3V2 payload. Our
Protected setting uses four sink tokens, a 128-token fp16 window, and group
size $g=128$. Scale/minimum metadata adds approximately $32/g$ b/d, giving
about 2.75 b/d for compressed codes and metadata, before the fp16 region.
This setting scores $40.83/28.33$ on Llama and $50.42/37.50$ on Qwen.

The missing Stress row is a configuration and budget constraint, not a failed
accuracy run. In the evaluated implementation, quantization groups must fit
within the fp16 flush buffer, which cannot be reduced to zero. Reducing both
group and buffer to 16 still retains fp16 tokens and increases compressed-cache
cost to $2.50+32/16=4.50$ b/d, before that retained state. Thus this low-buffer
proxy satisfies neither the zero-buffer condition nor the target rate; we do
not substitute it for Stress.

\paragraph{CommVQ.}
CommVQ uses learned RoPE-commutative key residuals and a trained 2-bit value
compressor~\citep{li2025commvq}. No codebooks were available for the DS-R1
distills, so we fit each model's key codebooks on about one million vectors per
layer and train the value compressor for 5,000 FineWeb-Edu
steps~\citep{penedo2024fineweb} with the backbone frozen.

Protected uses its native 128-token bf16 residual window and
adds a small prescale overhead to the 2.50 b/d payload; learned codebooks are
additional model-side storage. Stress reduces the window to one-token staging,
so it is a reduced-window configuration rather than a strictly buffer-free one.

The Protected configuration scores $33.33/25.42$ on Llama and $37.92/29.17$
on Qwen, compared with $7.50/6.67$ and $14.58/17.92$, respectively, under
Stress. These results indicate sensitivity to the retained high-precision
window in the evaluated configuration. The representation, trained components,
and retained state differ from Block-GTQ; this is cross-family context rather
than the controlled allocation comparison above.

\paragraph{KVQuant.}
KVQuant applies pre-RoPE per-channel NUQ3 to keys and per-token NUQ2 to values,
with sensitivity-weighted codebooks and sparse fp16
outliers~\citep{hooper2024kvquant}.
The released implementation supports Llama/DBRX; for DS-R1-Qwen-7B we port its
simulated-quantization path and use bf16 computation to avoid the fp16 overflow
observed in the port.

Here, protection consists of sparse outliers rather than a recent-token window.
Its outlier rate $s$ is a fraction of all cached entries, so the fp16 values and
sparse indices grow linearly with sequence length rather than forming a fixed
recent-token buffer. At cache length $T$ tokens, this corresponds to approximately
$sT$ token-equivalents of fp16 values per K/V side; counting 16-bit values and
16-bit indices gives an approximate K3V2 rate of
\[
  2.5(1-s)+32s \quad \text{b/d}.
\]
At 32K, the default 1\% rate retains about 320
token-equivalents of fp16 values per K/V side and realizes approximately
2.80 b/d at K3V2. To reduce this excess state, we use 0.5\%
($\approx160$ token-equivalents and $2.65$ b/d) for Protected and 0.01\%
($\approx3$ token-equivalents and $2.50$ b/d) for Stress. These percentages
give higher and lower sparse-protection allowances, not exact matches to a
fixed 132-token region and a zero-protection cache. In particular, the sparse
allowance grows with context length.

With the larger allowance, KVQuant scores $16.67/20.42$ on Llama and
$39.58/27.92$ on Qwen; at 0.01\%, it falls to $1.25/0.00$ and
$25.42/20.83$. As a separate diagnostic, we evaluate mean next-token
cross-entropy on the first 4096 tokens of the concatenated WikiText-2 test
split. Here $\Delta\mathrm{loss}$ is KVQuant minus the unquantized bf16 model
on the same tokens (nats/token); the mean is over token positions, not runs
or AIME answers. It rises from $+0.13/+0.23$ at 0.5\% to $+0.79/+0.75$ at 0.01\%
for Llama/Qwen. The two AIME rows measure sensitivity to KVQuant's
dense-and-sparse protection, not the presence or absence of a fixed buffer
or accuracy at a matched total memory footprint.

\paragraph{PM-KVQ (length-matched diagnostic).}
PM-KVQ begins at 16 bits and progressively halves each layer's precision when
its packed cache crosses a memory ceiling~\citep{liu2025pmkvq}. Unlike a
fixed-rate codec, its realized precision depends on both that ceiling and
the cache length. In our initial configuration, ceilings sized for 32,768
tokens leave shorter AIME traces at roughly 5--7 realized bits per dimension.
The nominal K3V2 label therefore does not make those runs rate-matched to
fixed K3V2 codecs. The length-matched diagnostic addresses this mismatch
by scaling the memory ceiling to an offline reference length.

We retain the halving policy and clamp the main cache at a 2-bit floor.
For each problem, the reference length is the mean input-plus-output token
length of the 16 responses (eight per regime) in an earlier, completed K3V2
PM-KVQ run on the same model. Let $T_{\rm ref}$ denote this reference token
count and $M_{\ell}$ the original physical-memory ceiling for layer $\ell$,
sized for 32,768 tokens. The diagnostic sets the main-cache ceiling to
\[
  M_{\ell}\min(1,T_{\rm ref}/32768).
\]
Input length includes the tokenized
prompt; responses reaching the generation cap contribute their observed
32,768-token output length. All eight new responses to a problem share this
ceiling, but generation itself is not capped at $T_{\rm ref}$. Protected adds
the 4-sink/128-recent fp16 cost on top of the main-cache ceiling, as in the
fixed-rate Protected rows, rather than letting protection consume that budget.
The reference lengths come from earlier responses to the same problems, so
this is an offline diagnostic rather than a test of online budget selection
for unseen prompts.

Under Stress, the four model--year mean realized rates are $2.39/2.14$ b/d
for Llama and $2.58/2.31$ b/d for Qwen (AIME 2024/2025). Each mean is computed
by replaying the precision-adjustment rule using each response's final cache
length, then averaging over layers and responses. The ledger's 2.14--2.58 b/d
range spans these four means, not individual responses, and is not a
measurement of physical cache storage. This rate accounting describes the
precision reached under the diagnostic's budgets, not a fixed per-response
K3V2 assignment.

Under these settings, Stress scores $15.42/3.33$ on Llama and $18.33/4.17$
on Qwen; Protected scores $19.58/6.25$ and $18.33/6.67$, respectively.
The Protected per-contest scores remain below Block-GTQ's $43.75/30.00$
on Llama and $54.17/37.50$ on Qwen.

An additional budget-accounting comparison explains why the protected region
is charged separately. An earlier setting made that region consume the
main-cache ceiling and obtained two-contest averages of $11.25$ on Llama and
$10.42$ on Qwen. With its cost added on top, the averages are $12.92$ and
$12.50$, respectively. The latter convention preserves the main-cache budget
when adding protection, as in the fixed-rate comparisons. These two settings
have different total ceilings; their accuracy difference is not a gain at
fixed total memory.

\FloatBarrier

\section{Extensions to Partial-RoPE and MLA}
\label{app:architecture-adapters}

These experiments test whether Block-GTQ's accuracy advantage over uniform
TQ-MSE extends beyond full-RoPE attention to partial-RoPE and MLA interfaces.
Matched MATH-500 comparisons show higher accuracy on both tested interfaces,
covering direct key-branch quantization and a shared K/V latent cache.
We encode each interface's nonrotary and rotary state separately, fixing
their average logical rates at 2 and 4 bits per coordinate, respectively.
The primary comparison applies Block-GTQ to both branches versus TQ-MSE to
both; two mixed-codec configurations provide supporting branch controls.
An unquantized reference measures the remaining accuracy gap.
The RoPE-block interpretation applies only to the rotary state; on the
nonrotary branch, Block-GTQ is a generic blockwise transform codec with
score-based rate allocation.

\subsection{GLM-4-9B: partial-RoPE adaptation}
\label{sec:glm-partial-adapter}

The GLM-4-9B base checkpoint applies RoPE to the first 64 of
128 key dimensions~\citep{glmteam2024chatglm}. We encode the 64-dimensional
NoPE key branch at an average of 2 bits per coordinate and the
64-dimensional RoPE key branch at an average of 4 bits per coordinate.
In all four quantized configurations, values use the same uniform 2-bit
TQ-MSE codec. Thus the dimension-weighted key rate is
\begin{equation}
  \bar b_K=\frac{64\cdot2+64\cdot4}{128}=3,
  \qquad
  \bar b_{KV}=\frac{\bar b_K+2}{2}=2.50\ \mathrm{b/d}.
  \label{eq:glm-branch-logical-rate}
\end{equation}
Here $\bar b_{KV}$ averages over equally sized key and value states.
The four configurations differ only in the codecs assigned to the two
key branches, not their budgets or the value codec.

\subsection{DeepSeek-V2-Lite-Chat: MLA latent interface}
\label{sec:mla}

DeepSeek-V2-Lite-Chat caches a 512-dimensional nonrotary latent $c$ and a
64-dimensional shared RoPE key $k_{\mathrm{pe}}$~\citep{deepseekai2024deepseekv2}.
The latent reconstructs the per-head nonrotary keys and values, whereas
$k_{\mathrm{pe}}$ contributes only to keys. We encode $c$ at an average of
2 bits per coordinate and $k_{\mathrm{pe}}$ at 4, giving the logical rate
\begin{equation}
  \bar b_{\mathrm{MLA}}
  =\frac{512\cdot2+64\cdot4}{512+64}
  =\frac{20}{9}\simeq2.222\ \mathrm{b/d}.
  \label{eq:mla-logical-rate}
\end{equation}
This rate averages over the latent and shared rotary key, not over dense
K/V coordinates. If $\widehat c$ is the reconstructed quantized latent,
each head's values are reconstructed as $\widehat v=W_{\mathrm{UV}}\widehat c$,
where $W_{\mathrm{UV}}$ is that head's value up-projection.
Thus quantizing $c$ affects both reconstructed keys and values; there is
no additional independent V2 codec in this interface.

\paragraph{Choice of branch budgets.}
SnapMLA quantizes MLA latent content to FP8 while retaining the rotary
component in BF16, supported by numerical diagnostics on
LongCat-Flash-Thinking~\citep{zhang2026snapmla}.
Consistent with this precision ordering, we use $(2,4)$ as a fixed
asymmetric low-rate setting for matched-codec comparisons, without claiming
that this budget split is optimal.
For GLM, the same branch rates define a controlled test setting; the MLA
evidence does not establish a sensitivity ordering for partial-RoPE keys.

\paragraph{Query-aware calibration in latent coordinates.}
For a fixed layer and query head, let $q_{\mathrm{pe}}\in\mathbb R^{64}$ be
the pre-RoPE rotary query and $q_{\mathrm{nope}}$ its nonrotary counterpart.
If $W_{\mathrm{UK}}$ is that head's key up-projection, then
$k_{\mathrm{nope}}=W_{\mathrm{UK}}c$. The effective query for the latent is
\begin{equation}
  q_c=W_{\mathrm{UK}}^\top q_{\mathrm{nope}}\in\mathbb R^{512},
  \qquad
  q_{\mathrm{nope}}^\top k_{\mathrm{nope}}=q_c^\top c.
  \label{eq:mla-effective-query}
\end{equation}
We therefore pair $q_c$ with $c$, and $q_{\mathrm{pe}}$ with
$k_{\mathrm{pe}}$, when applying the default mean Q/K energy score in
Equation~\ref{eq:block-score}. Query energies are averaged over the heads
sharing each cached branch, without first averaging the query vectors.
The latent score describes its key contribution; it does not explicitly
model the value-reconstruction error. Only the rotary branch retains a
RoPE-frequency interpretation.

\subsection{Shared protocol and matched-codec comparison}
\label{sec:architecture-codec-comparison}

\paragraph{Calibration and quantization.}
For each checkpoint, calibration uses 2048 WikiText-2
tokens~\citep{merity2017pointer}. All Block-GTQ components use the default
mean Q/K energy score. We use the same calibration token positions for the query
and cached-state statistics, pool the per-query-head energies for shared
state, and freeze each branch's allocation before evaluation.
All four configurations use Stress: prefill and decode operate on
quantized--reconstructed state from its first attention use, with no
unquantized sink, recent-token window, or delayed-flush tail.
The branch rates above are average logical payload rates, excluding
norms, padding, and other metadata.

\paragraph{Task and statistical comparison.}
We evaluate MATH-500~\citep{lightman2024lets} using eight responses per
problem, temperature $0.6$, top-$p=0.95$, and a 3072-token generation cap.
Within each checkpoint, the configurations share all 500 problems,
tokenized prompts, and generation and scoring settings.
Avg@8 is the mean of the eight binary answer-correctness scores for each
problem, averaged over problems. SEM is the sample standard deviation of
the 500 per-problem avg@8 scores divided by $\sqrt{500}$, expressed in
percentage points. The problem, not an individual generated response, is
the statistical unit.
The MLA fp16 reference reuses an earlier run with the same checkpoint,
prompts, sampling, and scoring protocol; it has no calibration dependency.

\begin{table}[!ht]
  \centering
  \caption{MATH-500 avg@8 (\%, mean $\pm$ per-problem SEM) for GLM-4-9B
  partial-RoPE keys and DeepSeek-V2-Lite-Chat MLA. Quantized rows use
  $(2,4)$-bit average rates for the nonrotary and rotary branches, calibrated
  on 2048 WikiText-2 tokens. A/D use a single codec on both branches; B/C change one branch
  at a time. Ref. is unquantized fp16; bold marks the highest observed
  quantized score in each column.}
  \label{tab:architecture-adapter-accuracy}
  \small
  \setlength{\tabcolsep}{3pt}
  \begin{tabular*}{\linewidth}{@{\extracolsep{\fill}}lllcc@{}}
    \toprule
    & Nonrotary codec & Rotary codec & GLM-4-9B & DS-V2-Lite-Chat \\
    \midrule
    Ref. & fp16 & fp16 & $40.95\pm1.75$ & $27.28\pm1.46$ \\
    A & TQ-MSE & TQ-MSE & $16.88\pm1.15$ & $14.53\pm1.00$ \\
    D & Block-GTQ & Block-GTQ & $\mathbf{30.05}\pm1.51$ & $\mathbf{20.50}\pm1.27$ \\
    \midrule
    B & TQ-MSE & Block-GTQ & $17.45\pm1.13$ & $19.38\pm1.24$ \\
    C & Block-GTQ & TQ-MSE & $29.00\pm1.47$ & $15.08\pm1.08$ \\
    \bottomrule
  \end{tabular*}
\end{table}

\FloatBarrier
\paragraph{Accuracy advantage on both interfaces.}
The primary D-versus-A comparison in
Table~\ref{tab:architecture-adapter-accuracy} holds branch budgets fixed and
changes the codec on both branches. On GLM-4-9B, Block-GTQ raises MATH-500
avg@8 from 16.88\% to 30.05\%, a gain of 13.2 percentage points. On
DeepSeek-V2-Lite-Chat, it raises avg@8 from 14.53\% to 20.50\%, a gain of
6.0 points. Thus the relative accuracy benefit is not confined to the
full-RoPE models in the main evaluation: it also appears in the tested
partial-RoPE key and MLA latent interfaces. The fp16 references remain
higher at 40.95\% and 27.28\%, respectively, so the gains do not imply
lossless low-bit quantization.

\paragraph{What the branch controls show.}
Rows B/C replace one branch codec at a time relative to A. On GLM, changing
the nonrotary codec yields the larger improvement (C: 29.00\%; B: 17.45\%);
on MLA, changing the rotary codec does (B: 19.38\%; C: 15.08\%). Using
Block-GTQ on both branches gives the highest observed quantized score on
each model. These controls locate the larger codec-replacement benefit
within this fixed-rate comparison; they do not measure sensitivity to
changing a branch's bit budget. In particular, they establish neither a
universal branch-sensitivity ordering nor optimality of the $(2,4)$ split.
The conclusions concern the two tested checkpoints and this MATH-500
protocol. Comparisons are within each architecture, whose cached
representations and weighted logical rates differ.

\FloatBarrier

\section{Packed-Cache Implementation Details}
\label{app:packed-details1}

\subsection{Physical layout}

In our implementation, logical widths of one to four bits are stored in
four-bit nibble containers, whereas wider key segments use byte containers.
For each layer and KV head, blocks assigned the same rate are permuted into a
contiguous segment. A token stores one mixed-rate K code row with one fp16 norm
per nonempty segment, together with one uniform-rate V code row and one fp16 V
norm. Allocated norm tensors are padded to the cache layout's maximum number
of key segments. Rate-specific codebooks, group rotations, permutations, and
segment offsets are static and amortized across tokens. Only the code rows and
their norms grow with context length in the persistent cache.
If $S_{\ell,h}$ is the number of nonempty K-rate segments, the static
descriptor count is $O(\sum_{\ell,h} S_{\ell,h})$ and is independent of
sequence length; decode still traverses every cached position.

\subsection{Attention in codec coordinates}

Each same-rate key segment is encoded after a TQ-MSE orthogonal rotation.
Rather than inverse-rotating every reconstructed key tile, the serving path
applies the matching permutation and groupwise rotations to each post-RoPE
query. Orthogonality preserves the query--key dot product, allowing
$QK^\top$ to be accumulated directly in the encoded coordinates. The current
implementation applies this transform after RoPE; it cannot in general be
folded into the query projection because a dense group rotation may mix
distinct RoPE frequencies.

Values likewise remain in their rotated coordinates during the $PV$
accumulation. By linearity, the inverse value rotation is applied once to the
accumulated output rather than separately to every reconstructed value tile.
End-to-end timings and the packed attention calls in
Appendix~\ref{app:deployment-kernel-microbench} include both the query
transform and this output transform.

\subsection{CUDA decode kernel and split-axis merge}

The measured CUDA path uses separate template specializations for uniform
TQ-MSE and mixed-rate Block-GTQ, reading the same layouts used by prefill.
One core launch covers all KV heads of a layer, with additional blocks
partitioning the key axis. Warps pipeline asynchronous global-to-shared
copies of packed K/V tiles. Threads reconstruct only the operands needed
by their tensor-core fragments, using codebook lookup and norm rescaling.
Uniform low-bit codes, the V codes, and selected low-bit Block-GTQ segments
use register-resident byte-permute lookup tables. The remaining Block-GTQ
segments, including byte-stored segments, use a per-head codebook staged
in shared memory; not every mixed-rate segment uses a register lookup.

The kernel computes QK and PV with fp16-input, fp32-accumulate tensor-core
products and fp32 online-softmax statistics. Each key-axis split writes a
partial output and log-sum-exp state. A second kernel combines them by a
numerically stable weighted merge and applies the inverse V rotation.
Only these partial outputs and their statistics are written to HBM;
decoded K/V tiles and the attention matrix are not materialized there.
The fast path changes floating-point operation order, so bitwise equivalence
to earlier kernels is not assumed. We report deployed perplexity in
Appendix~\ref{app:deployment-qwen} and attention-kernel costs in
Appendix~\ref{app:deployment-kernel-microbench}.

\subsection{Integration and CUDA graphs}

The attention kernel reads sequence length from a device tensor and uses
a fixed launch grid with masked inactive partitions. New-token K/V encoding,
packing, and cache-row updates are fused into one launch per layer.
The post-RoPE query transform is computed in the attention prologue;
it is not baked into the query projection. These changes allow the complete
decode step, including projections, cache updates, attention, and the
remaining model layers, to be captured and replayed as a single CUDA graph.
The main FP16 baseline receives the same whole-step graph treatment.

\subsection{Prefill}

The measured prefill processes the prompt one layer at a time in
8192-token chunks. Within each chunk, batched K and V quantization kernels
write codes and norms directly into the persistent packed buffers.
Packed prefill attention streams the relevant prefix tiles under a causal
mask and online softmax, without materializing a dense attention matrix or
a decoded fp16 cache. Prefill and autoregressive decoding therefore share
one cache representation instead of constructing the prompt cache through
$T$ autoregressive update steps. The FP16 baseline also uses 8192-token
prefill chunks, with native-GQA FlashAttention-2 writing directly to its dense cache.
Prefill and graph-setup time are excluded from the reported decode latency.

\subsection{Rate and footprint accounting}

Logical rate and physical footprint differ because the deployed representation
also includes container padding, norms, metadata, and allocation alignment.
Appendix~\ref{app:rate-accounting} defines the logical payload rate, total
effective rate, and allocated physical footprint. The concrete
Qwen2.5-3B-Instruct K3V3 ledger and layout-derived resident footprint are
reported in Appendix~\ref{app:deployment-details}; they should not be
conflated with peak GPU memory, which also includes model weights and transient
workspaces.

\section{Packed Deployment: Protocol and Extended Results}
\label{app:deployment-details}

These experiments ask whether the accuracy benefit of mixed-rate keys can be
retained when the cache is physically packed, and what memory and decode-time
costs remain after optimizing the serving path. We first give the matched
end-to-end results, then describe the packed-serving optimization steps and
compare the final attention kernel across context lengths. The optimization
ladder and kernel microbenchmark are separate experiments, not additional
repetitions of the main end-to-end sweep.

\subsection{Measurement protocol}
\label{app:deployment-measurement}

We evaluate Qwen2.5-3B-Instruct at batch size 1 on one H800 80GB GPU.
Qwen uses YaRN with factor 4, extending its configured 32K window to 128K.
We measure perplexity only through 128K. Measurements at
256K--1M test execution capacity, latency, and memory, not usable
long-context quality. Here K and M denote $2^{10}$ and $2^{20}$ tokens;
memory is reported in decimal MB and GB.

\paragraph{Cache and calibration.}
Within each checkpoint, all methods use identical fp16 model weights and
the same consecutive WikiText-2 train token stream, without repeating the
corpus to construct the longest context. Block-GTQ uses the first 64 tokens
for allocation. This deployment-specific
calibration is distinct from the 2048-token calibration in the main accuracy
studies. Both packed methods use K3V3: Block-GTQ assigns key-block widths
from 1 to 8 at an average of 3, while TQ-MSE fixes every key block at 3;
both use the same uniform V3 codec. Keys are encoded after RoPE in the
packed serving path. No persistent fp16 sink or recent-token buffer is
retained.

\paragraph{Matched decode baseline.}
Throughout the end-to-end results, \emph{FP16} denotes the uncompressed-cache
baseline using FlashAttention-2
(\texttt{flash-attn} 2.8.3, \texttt{flash\_attn\_with\_kvcache}) with a native
GQA cache, without expanding KV heads to query heads or retaining a second
dense cache copy. The two packed methods use the CUDA decode kernel
described in Appendix~\ref{app:packed-details1}; the reported sweep uses this
path on all 36 Qwen layers.
FlashAttention-2 computes attention in tiles without materializing the full
attention matrix. Separately, all three main methods use a \emph{whole-step
CUDA graph}: the GPU operations for one decode step---model layers, cache
updates, and attention---are recorded once and replayed with new inputs,
reducing repeated CPU launch overhead. Each replay computes the current
token's outputs; it does not merge all operations into one kernel.
Prefill uses 8192-token chunks
for every method and is excluded from decode latency, as are calibration
and graph setup.

For each context length, we report the median of 20 timed steps after five
untimed warmup steps. Steps consume consecutive corpus tokens rather than
sampled model outputs; synchronized wall-clock timing surrounds each decode
call, with token and position preparation outside the timed region.
These are within-run step statistics, not medians
over three independent runs. The context-length sweep's graphed cells have a
5th--95th percentile spread below 1.3\% of their median; this measures
within-run variability, not uncertainty across jobs.

\paragraph{Perplexity and memory.}
Perplexity is teacher-forced over the last 1000 tokens of each $T$-token
context, positions $[T-1000,T)$, with the same scored positions for all
methods. The cache is rewound before scoring, so timing steps do not
change the evaluation window. The scored positions do not overlap the
64-token calibration prefix.
Resident KV footprint counts code and norm storage for $T$ valid positions;
it is distinct from the allocation high-water mark, which also includes
model weights, reserved cache capacity, and workspaces.
The reported post-prefill peak excludes prefill and covers warmup, timed
decoding, and teacher-forced scoring where enabled. In the recorded harness,
the packed paths reset peak statistics before fast-cache and graph setup,
whereas FP16 resets after graph setup; the packed peak can therefore also
include setup transients. We report these peaks with this measurement
boundary rather than interpret them as an isolated attention workspace.
Prefill peaks are listed separately.

\subsection{Qwen: quality, decode latency, and memory}
\label{app:deployment-qwen}

\paragraph{Quality at the supported lengths.}
Table~\ref{tab:app-deployment-ppl-n1000} reports the deployed paths'
perplexities through 128K. Block-GTQ is 3.7--7.1\% above FP16 across the six
lengths; at 128K the values are 9.80 and 9.34. Uniform TQ-MSE has much
higher perplexity (615.50--9319.53) in this Qwen K3V3 setting.
Thus its lower latency must be read alongside its loss of quality.
This is a deployment-path observation, not an additional isolation of the
failure mechanism or evidence about other Qwen checkpoints.

\begin{table}[!htbp]
\centering
\caption{\textbf{Qwen2.5-3B-Instruct perplexity through the deployed decode path.}
YaRN factor 4; 1000 teacher-forced WikiText-2 tokens at positions $[T-1000,T)$.
FP16 uses FlashAttention-2; all three paths use whole-step CUDA graphs.
Bold identifies the lower perplexity of the two K3V3 methods.
$\Delta$ is Block-GTQ's relative change from FP16; TQ/BGT is the PPL ratio.}
\label{tab:app-deployment-ppl-n1000}
\small
\setlength{\tabcolsep}{5pt}
\begin{tabular}{r rrr r r}
\toprule
$T$ & FP16 & Block-GTQ & TQ-MSE & $\Delta$ vs.\ FP16 & TQ/BGT \\
\midrule
4K & 11.99 & \textbf{12.84} & 615.50 & $+7.1\%$ & 48$\times$ \\
8K & 9.63 & \textbf{9.99} & 679.01 & $+3.7\%$ & 68$\times$ \\
16K & 6.54 & \textbf{6.99} & 2335.39 & $+7.0\%$ & 334$\times$ \\
32K & 7.58 & \textbf{7.91} & 2512.15 & $+4.3\%$ & 318$\times$ \\
64K & 8.84 & \textbf{9.40} & 2733.30 & $+6.3\%$ & 291$\times$ \\
128K & 9.34 & \textbf{9.80} & 9319.53 & $+4.9\%$ & 951$\times$ \\
\bottomrule
\end{tabular}
\end{table}

\paragraph{Latency under matched graph execution.}
Table~\ref{tab:app-deployment-latency} gives the context-length sweep in
Figure~\ref{fig:deployment-main}(b).
At 128K, Block-GTQ takes 9.50 ms/step versus 8.99 for FP16, a 5.7\%
latency overhead; it does not outperform this baseline in the evaluated
quality range. Uniform TQ-MSE takes 7.73 ms/step at the same length.
Block-GTQ's latency relative to FP16 decreases from $1.26\times$ at 4K
to $1.06\times$ at 128K, and falls below FP16 at the measured 512K and 1M
points ($0.98\times$ and $0.92\times$). These latter gains concern execution
beyond the supported quality range.

\begin{table}[!htbp]
\centering
\caption{\textbf{Qwen2.5-3B-Instruct decode latency.} Median ms/step over 20 steps
after five warmups, one H800, batch size 1. The main methods all use whole-step
CUDA graphs. FP16 uses FlashAttention-2. Ratios divide packed-path latency
by FP16 latency ($<1$: faster). Beyond 128K, results concern execution
capacity and cost only.}
\label{tab:app-deployment-latency}
\footnotesize
\setlength{\tabcolsep}{4pt}
\begin{tabular}{r rrr rr}
\toprule
 & \multicolumn{3}{c}{Decode latency (ms/step)} & \multicolumn{2}{c}{Ratio vs.\ FP16} \\
\cmidrule(lr){2-4}\cmidrule(lr){5-6}
$T$ & FP16 & TQ-MSE & Block-GTQ & TQ-MSE & Block-GTQ \\
\midrule
4K & 6.45 & 7.02 & 8.12 & 1.09$\times$ & 1.26$\times$ \\
8K & 6.74 & 7.03 & 8.12 & 1.04$\times$ & 1.20$\times$ \\
16K & 7.10 & 7.08 & 8.12 & 1.00$\times$ & 1.14$\times$ \\
32K & 7.46 & 7.16 & 8.31 & 0.96$\times$ & 1.11$\times$ \\
64K & 8.00 & 7.34 & 8.69 & 0.92$\times$ & 1.09$\times$ \\
128K & 8.99 & 7.73 & 9.50 & 0.86$\times$ & 1.06$\times$ \\
256K & 10.56 & 8.20 & 10.72 & 0.78$\times$ & 1.02$\times$ \\
512K & 13.46 & 9.28 & 13.18 & 0.69$\times$ & 0.98$\times$ \\
1M & 19.50 & 11.31 & 17.94 & 0.58$\times$ & 0.92$\times$ \\
\bottomrule
\end{tabular}
\end{table}

\FloatBarrier
\paragraph{Physical footprint versus logical rate.}
For head dimension 128, ideal three-bit K/V codes require 96 bytes per
token and KV head, compared with 512 for FP16.
The deployed Qwen layouts require 132 bytes for uniform TQ-MSE and
158 bytes for Block-GTQ, giving $3.88\times$ and $3.24\times$ resident-KV
compression (Table~\ref{tab:app-deployment-footprint}).
The difference reflects nibble/byte containers and per-segment norms,
including norm padding; it is not a change in the average logical K3V3
budget. Static codebooks and transforms do not grow with context length.
The 158-byte figure is specific to this model's calibrated layout, not a
universal cost of Block-GTQ.

\begin{table}[!htbp]
\centering
\caption{\textbf{Allocated KV layout per token and KV head.} Qwen2.5-3B-Instruct,
K3V3. Bytes are derived from the code and norm tensors for valid cache positions;
static transforms and codebooks are excluded. Compression is relative to FP16.}
\label{tab:app-deployment-footprint}
\small
\setlength{\tabcolsep}{8pt}
\begin{tabular}{l r r}
\toprule
Cache representation & Allocated bytes & Compression \\
\midrule
FP16 KV cache & 512 & $1.00\times$ \\
Ideal 3-bit code-only & 96 & $5.33\times$ \\
Deployed TQ-MSE K3V3 & 132 & $3.88\times$ \\
Deployed Block-GTQ K3V3 & 158 & $3.24\times$ \\
\bottomrule
\end{tabular}
\end{table}

At 128K, resident KV occupies 1.49 GB for Block-GTQ versus 4.83 GB for
FP16; the reported post-prefill peaks are 7.93 and 11.24 GB.
All three Qwen paths run through 1M in this harness. At that length,
Block-GTQ uses 11.93 GB of resident KV versus 38.65 GB for FP16, with
post-prefill peaks of 18.46 and 45.13 GB
(Table~\ref{tab:app-deployment-memory}).
These memory measurements do not establish language-model quality beyond
128K.

\begin{table}[!htbp]
\centering
\caption{\textbf{Qwen2.5-3B-Instruct memory (decimal GB).} Same runs as
Table~\ref{tab:app-deployment-latency}. KV counts resident code and norm storage
for $T$ valid positions. Post-prefill and prefill peaks include model weights
and workspaces. The post-prefill reset boundaries differ in graph-setup
inclusion; see Appendix~\ref{app:deployment-measurement}.
Reported model allocation is 6.17 GB.}
\label{tab:app-deployment-memory}
\footnotesize
\setlength{\tabcolsep}{3.5pt}
\begin{tabular}{r rrr rrr rrr}
\toprule
 & \multicolumn{3}{c}{KV (GB)} & \multicolumn{3}{c}{Post-prefill peak (GB)} & \multicolumn{3}{c}{Prefill peak (GB)} \\
\cmidrule(lr){2-4}\cmidrule(lr){5-7}\cmidrule(lr){8-10}
$T$ & FP16 & TQ-MSE & Block-GTQ & FP16 & TQ-MSE & Block-GTQ & FP16 & TQ-MSE & Block-GTQ \\
\midrule
4K & 0.15 & 0.04 & 0.05 & 6.46 & 6.37 & 6.38 & 6.74 & 6.66 & 6.68 \\
8K & 0.30 & 0.08 & 0.09 & 6.61 & 6.41 & 6.43 & 7.23 & 7.08 & 7.10 \\
16K & 0.60 & 0.16 & 0.19 & 6.95 & 6.52 & 6.56 & 7.60 & 7.37 & 7.41 \\
32K & 1.21 & 0.31 & 0.37 & 7.55 & 6.67 & 6.74 & 8.21 & 7.66 & 7.73 \\
64K & 2.42 & 0.62 & 0.75 & 8.79 & 7.02 & 7.15 & 9.45 & 8.28 & 8.41 \\
128K & 4.83 & 1.25 & 1.49 & 11.24 & 7.68 & 7.93 & 11.90 & 9.47 & 9.73 \\
256K & 9.66 & 2.49 & 2.98 & 16.07 & 8.95 & 9.44 & 16.73 & 12.32 & 12.81 \\
512K & 19.33 & 4.98 & 5.96 & 25.77 & 11.47 & 12.46 & 26.43 & 18.07 & 19.06 \\
1M & 38.65 & 9.97 & 11.93 & 45.13 & 16.49 & 18.46 & 45.79 & 29.54 & 31.50 \\
\bottomrule
\end{tabular}
\end{table}

\FloatBarrier

\subsection{Packed-serving optimization steps}
\label{app:deployment-optimization}

Figure~\ref{fig:deployment-main}(a) summarizes the optimization step of packed cache serving :
tile and key-axis split tuning, graph-based encoding, fused V decoding,
and tensor-core QK with nibble/byte-stored keys. These historical timings
come from the original plotting script, which labels an 8-KV-head setup;
raw job timings are not retained. They document the earlier kernel work
and are not combined with the later end-to-end measurements on the
2-KV-head Qwen model.

\paragraph{End-to-end integration and CUDA kernel.}
Figure~\ref{fig:deployment-e2e-ladder-128k} measures the subsequent
optimizations  on Qwen2.5-3B-Instruct:

\begin{enumerate}[label=(\arabic*),leftmargin=*,itemsep=2pt,topsep=3pt]
  \item \textbf{All-head attention and fused updates.} Combine all KV heads
  into one attention-core launch, read sequence length from device memory,
  fuse new-token encoding and cache updates into one launch per layer, and
  move the query transform into the attention prologue.
  \item \textbf{Per-layer graphs.} Capture each layer's encoding and attention
  together; CPU code still launches the rest of the model.
  \item \textbf{Whole-step graph.} Capture the complete decode step, including
  projections, cache updates, attention, and the remaining model operations.
  \item \textbf{CUDA attention kernel.} Replace the Triton attention core
  with the final CUDA implementation, using asynchronous tile copies,
  fragment-level decoding, and register lookups for selected low-bit groups.
  The query prologue and split-axis merge remain part of the attention path
  (Appendix~\ref{app:packed-details1}).
\end{enumerate}

We measured the optimization ladder in a separate experiment from the
context-length sweep in Figure~\ref{fig:deployment-main}(b).
Each stage reports the median of three interleaved run-level medians,
each with five warmup and 20 timed steps; the context-length sweep
reports one run's median of 20 steps.
Block-GTQ decreases from 52.39 to 9.50 ms/step ($5.5\times$ relative to
its own starting implementation, not FP16). The CUDA kernel substitution
alone reduces 12.09 to 9.50 ms/step ($1.27\times$); preceding changes
also reduce encoding and launch overhead. The same-job FP16 baseline
uses FlashAttention-2 with whole-step graph replay and takes 8.85 ms/step,
so the optimized Block-GTQ path remains slower. We retain this experiment's
own FP16 reference rather than substitute the context-length sweep's
8.99 ms/step.

\begin{figure}[!htbp]
  \centering
  \includegraphics[width=\linewidth]{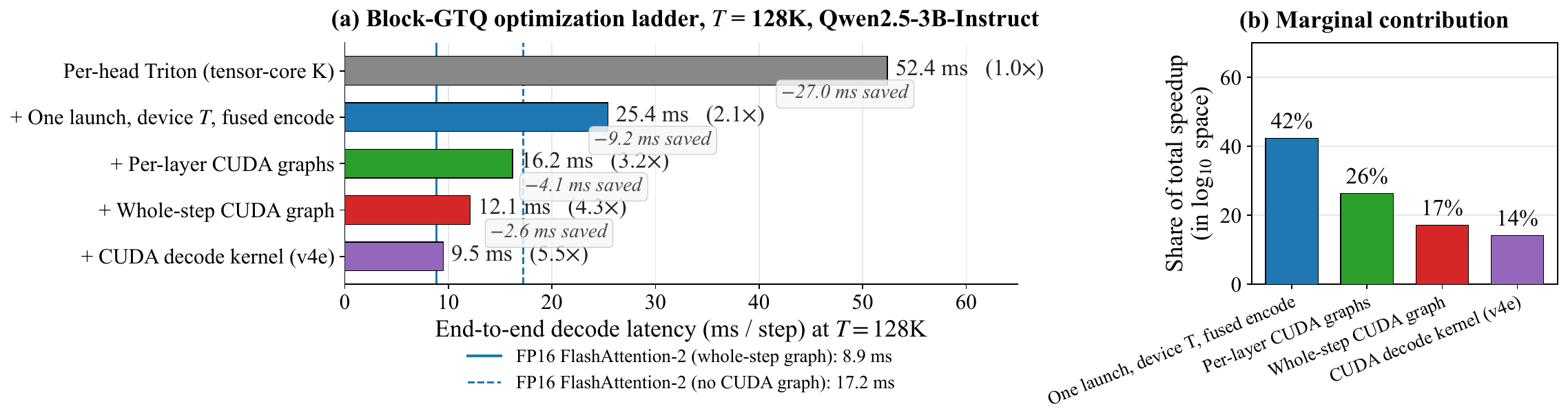}
  \caption{\textbf{Block-GTQ end-to-end optimization at 128K.}
  Qwen2.5-3B-Instruct K3V3 on one H800: the initial per-head Triton path
  followed by four cumulative optimizations.
  (a) Median of three run-level median latencies, with per-stage savings
  and gains relative to the initial per-head Triton path, measured
  independently of Figure~\ref{fig:deployment-main}(b).
  Vertical lines are same-job FP16 FlashAttention-2 references: whole-step
  CUDA graph (solid) and no CUDA graph (dashed), both using FlashAttention-2.
  (b) Each transition's share of the total logarithmic gain:
  $\log(t_{j-1}/t_j)/\log(t_1/t_5)$ for stages $j=2,\ldots,5$,
  where $t_j$ is the latency of stage $j$.
  These order-dependent shares describe this cumulative sequence, not
  independent component ablations or fractions of total milliseconds saved.}
  \label{fig:deployment-e2e-ladder-128k}
\end{figure}
\FloatBarrier

\subsection{Final attention-kernel microbenchmark}
\label{app:deployment-kernel-microbench}

\paragraph{Protocol and timing scope.}
Table~\ref{tab:app-deployment-kernel-sweep} compares the all-head Triton
attention core with the final CUDA kernel (v4e), both serving Block-GTQ
at K3V3. The benchmark uses one H800, batch size 1, head dimension 128,
and the attention geometries of Qwen2.5-3B (16 Q / 2 KV heads) and
Llama-3.1-70B (64 Q / 8 KV heads). Allocation uses each model's layer-0
Q/K calibration; random fp16 K/V states populate the cache. No full-model
weights are loaded. Each call attends over $T+1$ positions after a new
token has been written at position $T$.

All columns use per-layer CUDA-graph replay: 20 calls per replay and
10 timed replays, with synchronized wall-clock time divided by 200.
These are amortized per-call times, not medians over independent runs.
The two packed columns include the query transform, attention core,
split merge, and inverse V transform, but exclude new-token encoding
and cache writes. The FP16 reference uses FlashAttention through PyTorch's
scaled dot-product attention (SDPA) interface and includes the new-token
cache write. It is not the direct \texttt{flash\_attn\_with\_kvcache}
interface used in the end-to-end comparison. Thus CUDA/FP16 ratios have
asymmetric timing boundaries and are not matched full-step speedups.

\begin{table}[!htbp]
\centering
\captionsetup{font=small,skip=4pt}
\caption{\textbf{Final attention-kernel context-length sweep.}
Milliseconds per layer on one H800, K3V3, batch size 1.
Triton and CUDA both serve Block-GTQ; CUDA uses the final v4e kernel.
FP16 uses FlashAttention through PyTorch SDPA, not the direct FA-2
interface of the end-to-end baseline. All columns use per-layer
CUDA-graph replay. Packed columns exclude encoding and cache writes;
FP16 includes the new-token cache write. CUDA/FP16 is a latency ratio
computed before rounding, not an end-to-end speedup.}
\label{tab:app-deployment-kernel-sweep}
\footnotesize
\setlength{\tabcolsep}{3pt}
\begin{tabular}{r rrrr rrrr}
\toprule
 & \multicolumn{4}{c}{Qwen2.5-3B geometry} & \multicolumn{4}{c}{Llama-3.1-70B geometry} \\
\cmidrule(lr){2-5}\cmidrule(lr){6-9}
$T$ & FP16 & Triton & CUDA & CUDA/FP16 & FP16 & Triton & CUDA & CUDA/FP16 \\
\midrule
4K & 0.017 & 0.026 & 0.023 & 1.33$\times$ & 0.022 & 0.038 & 0.024 & 1.10$\times$ \\
8K & 0.023 & 0.026 & 0.025 & 1.07$\times$ & 0.024 & 0.052 & 0.035 & 1.43$\times$ \\
16K & 0.026 & 0.038 & 0.025 & 0.96$\times$ & 0.044 & 0.084 & 0.050 & 1.15$\times$ \\
32K & 0.036 & 0.051 & 0.035 & 0.97$\times$ & 0.065 & 0.138 & 0.068 & 1.05$\times$ \\
64K & 0.054 & 0.075 & 0.044 & 0.81$\times$ & 0.107 & 0.233 & 0.104 & 0.98$\times$ \\
128K & 0.078 & 0.125 & 0.067 & 0.86$\times$ & 0.191 & 0.417 & 0.171 & 0.89$\times$ \\
256K & 0.124 & 0.219 & 0.112 & 0.90$\times$ & 0.362 & 0.788 & 0.306 & 0.85$\times$ \\
512K & 0.204 & 0.407 & 0.191 & 0.93$\times$ & 0.697 & 1.538 & 0.564 & 0.81$\times$ \\
1M & 0.374 & 0.780 & 0.364 & 0.97$\times$ & 1.376 & 3.031 & 1.085 & 0.79$\times$ \\
\bottomrule
\end{tabular}
\end{table}

At 128K, the CUDA attention path takes 0.067 ms for the Qwen geometry and
0.171 ms for the 70B geometry, versus 0.125 and 0.417 ms for Triton.
The CUDA/FP16 ratios are 0.86 and 0.89 under the timing boundaries above;
CUDA can still be slower than the FP16 reference at short contexts.
\FloatBarrier

\end{document}